\documentclass[twoside, 11pt]{article}
\usepackage[utf8]{inputenc}
\usepackage[left=1in,top=1in,right=1in,bottom=1in]{geometry}
\usepackage{savesym}
\usepackage{amsmath,amsthm}
\usepackage[preprint]{jmlr2e}

\usepackage{color}
\usepackage[english]{babel}
\usepackage{mathrsfs,dsfont}
\usepackage{array}
\usepackage{tikz}
\usepackage{adjustbox}
\usepackage{multirow}
\usepackage{float}
\usepackage{booktabs}
\usepackage{lineno}
\usepackage{rotating}
\usepackage{pdflscape}
\usepackage{dsfont}
\usepackage{diagbox}
\usepackage{tikz}
\usepackage[mathscr]{euscript}
\usepackage{bbm}
\usepackage{mathtools}
\usepackage[linesnumbered,vlined,ruled]{algorithm2e}
\usepackage{accents}
\usepackage{tikz-cd}
\savesymbol{ring}
\usepackage{mathabx}
\restoresymbol{abx}{ring}
\usepackage[T1]{fontenc}

\theoremstyle{plain}
\newtheorem{prop}[theorem]{\bf Proposition}
\theoremstyle{definition}
\newtheorem{assumption}{Assumption}

\newcommand{\nn}{\nonumber}

\def\QED{~\rule[-1pt]{5pt}{5pt}\par\medskip}
\renewenvironment{proof}{{\bf Proof: \ }}{\hfill \QED}

\let\ALP  \mathcal

\renewcommand{\sp}{\texttt{span}}

\DeclareMathOperator*{\minimize}{minimize}

\DeclareUnicodeCharacter{03BB}{$\lambda$}

\renewcommand{\cal}{\mathcal}
\newcommand{\bb}{\mathbb}
\newcommand{\bbm}{\mathbbm}
\newcommand{\mbf}{\mathbf}

\newcommand{\lip}{\textnormal{\texttt{Lip}}} 
\newcommand{\mmd}{\textnormal{\texttt{MMD}}}
\newcommand{\ipm}{\textnormal{\texttt{IPM}}}

\newcommand{\tv}{\textnormal{\texttt{TV}}}

\newcommand{\wass}{\textnormal{\texttt{W}}}

\newcommand{\diam}{\textnormal{\texttt{diam}}}

\newcommand{\msf}{\mathsf}

\newcommand{\absgap}{\textnormal{\texttt{Abs.Gap}}}
\newcommand{\spec}{\textnormal{\texttt{Spec}}}

\newcommand{\bl}{\textnormal{\texttt{BL}}}
\newcommand{\proj}{\textnormal{\texttt{Proj}}}
\newcommand{\ltpi}{{\ensuremath{L^2(\pi)}}}
\newcommand{\ltzpi}{{\ensuremath{L^2_0(\pi)}}}
\newcommand{\lppi}{{\ensuremath{L^p(\pi)}}}
\newcommand{\unif}{{\ensuremath{\|\cdot\|_\infty}}}

\usepackage{enumitem}
\usepackage{array}
\newcolumntype{C}[1]{>{\centering\arraybackslash}p{#1}}

\usepackage{lastpage}
\jmlrheading{23}{2023}{1-\pageref{LastPage}}{7/23; Revised 7/23}{ 7/23}{21-0000}{Hao Chen, Abhishek Gupta, Yin Sun, and Ness Shroff}
\ShortHeadings{Hoeffding's Inequality for Markov Chains}{Chen, Gupta, Sun, and Shroff}
\firstpageno{1}

\begin{document}

\title{Hoeffding's Inequality for Markov Chains under Generalized Concentrability Condition}
\author{%
    \name Hao Chen \email chen.6945@osu.edu\\
    \name Abhishek Gupta \email gupta.706@osu.edu\\
    \addr Department of ECE\\
    The Ohio State University\\
    Columbus, OH 43210, USA
    \AND
    \name Yin Sun \email yzs0078@auburn.edu\\
    \addr Department of ECE\\
    Auburn University\\
    Auburn, AL 36849, USA
    \AND
    \name Ness Shroff \email shroff.11@osu.edu\\
    \addr Department of ECE and CSE\\
    The Ohio State University\\
    Columbus, OH 43210, USA
}
\editor{}

\maketitle

\begin{abstract}%
    This paper studies Hoeffding's inequality for Markov chains under the generalized concentrability condition defined via integral probability metric (IPM). The generalized concentrability condition establishes a framework that interpolates and extends the existing hypotheses of Markov chain Hoeffding-type inequalities. The flexibility of our framework allows Hoeffding's inequality to be applied beyond the ergodic Markov chains in the traditional sense. We demonstrate the utility by applying our framework to several non-asymptotic analyses arising from the field of machine learning, including (i) a generalization bound for empirical risk minimization with Markovian samples, (ii) a finite sample guarantee for Ployak-Ruppert averaging of SGD, and (iii) a new regret bound for rested Markovian bandits with general state space. 
\end{abstract}


\section{Introduction}
\label{sec: introduction}
In this paper, we revisit Hoeffding's inequality for general state space Markov Chains through the lens of integral probability metric (IPM). The general state space considered here is \textit{Polish space}, i.e., a separable completely metrizable space. Concentration inequalities provide upper bounds on the tail probability of large deviation, that is, the deviation of the sum of random variables from its mean. One of the most important and celebrated concentration inequalities was first proposed by \citet{hoeffding1994probability}. Hoeffding's inequality states that for $n$ independent bounded random variables $Y_i\in[a_i, b_i]$, $i=1, \cdots, n$ and some $\epsilon > 0$, the following holds
\begin{align}
\label{eqn: og hoeffding inequality}
    \bb{P}\bigg(\bigg|\sum_{i=1}^n Y_i - \sum_{i=1}^n\bb{E}(Y_i)\bigg| \geq n\epsilon \bigg) \leq 2\exp\bigg(\frac{-2(n\epsilon)^2}{\sum_{i=1}^n(b_i-a_i)^2}\bigg).
\end{align}

Inequalities of this type are also called concentration of measure. Hoeffding's inequality gives a sub-Gaussian bound on the tail probability, i.e., the right-hand side of Equation \ref{eqn: og hoeffding inequality} decays as $\exp[-\Theta(n\epsilon^2)]$. There are numerous other concentration inequalities for \textit{independent} random variables that provide tail bounds with varying decay rates depending on the probability distributions of the random variables. Some related results are Bernstein's inequality and Bennett's inequality with decay rates $\exp[-\Theta(n\epsilon)]$ and $\exp[-\Theta(n\epsilon\log\epsilon + n\epsilon)]$, respectively. Interested readers can refer to \citet[Section 2]{boucheron2013concentration} for a complete overview of this topic.  These results are widely used in statistics, econometrics, causal inference, machine learning, reinforcement learning, empirical process theory, and other fields as a tool for non-asymptotic analysis. 

The independence assumption in the original Hoeffding's inequality limits its usage to applications where dependence among random variables naturally arises. Many research efforts have been dedicated to the extension of this result to random variables with various dependence structures, such as martingales \citep[see Azuma-Hoeffding inequality]{azuma1967weighted}. In particular, Hoeffding-type inequalities for Markov chains have attracted much attention during the past few decades, due to their prevalence in statistics, machine learning, and stochastic dynamic systems. For example, Markovian dependence occurs broadly in Markov chain Monte Carlo (MCMC), time series prediction, and reinforcement learning. 

Hoeffding-type inequalities for finite state space Markov chains are well-studied in the past \citep{davisson1981error, watanabe2017finite}. See \citet{moulos2019optimal, moulos2020hoeffding} for a recent survey on this topic. For \textit{general} state space Markov chains, Hoeffding's inequality was originally studied \citep{glynn2002hoeffding} under Doeblin's condition \citep{douc2018markov} which implies uniform ergodicity of Markov chains. More relaxed conditions were not considered until very recently. 

Unlike finite or countably infinite state space Markov chains, uniform ergodicity for general state space Markov chains is not met easily. As implied by Doeblin's condition, the Markov chain is required to jump, within a few steps, to anywhere in the state space with nontrivial probability given an arbitrary initial condition. Physical processes, such as the motion of vehicles and the articulation of robotic arms, often naturally follow continuity, which means the system state is not likely to evolve far from its starting point within a short period of time. Therefore, uniform ergodicity assumption can lead to two major inefficiencies in the analysis: (i) excluding many ``well-behaving'' systems, such as a stable linear time-invariant system; and (ii) yielding loose bounds for the tail probability, which are made precise in Section \ref{sec: comparison}.

For example, a common assumption in reinforcement learning is that the system is uniformly ergodic under any control policy. However, many reinforcement learning algorithms are tested in OpenAI Gym \citep{brockman2016openai} environments that are deterministic robotic systems. They do not satisfy the uniform ergodicity under a deterministic control policy, as the singular nature of Dirac measures violates the minorization condition. However, under a stabilizing control policy, these systems satisfy weaker versions of ergodicity, such as ergodicity in the sense of the Wasserstein metric. Similarly, in multi-armed bandit algorithms, the usual assumption is that the reward distributions are i.i.d. However, bandits that are driven by physical systems feature reward samples that are Markovian in nature \citep{tekin2012online, anantharam1987asymptotically} with hard-to-determine convergence properties. The existing results on this topic assume finite state space and Markov chains are uniformly ergodic, which is a strong assumption to impose on physical systems and leads to limited utility in practice. Online risk minimization with temporally correlated samples is another example. A generalization error bound was proposed for $\beta$-mixing processes \citep{agarwal2012generalization}. In the context of Markovian processes, this corresponds to uniform ergodicity which, again, is very rigid to use in practice. 

The aforementioned bottlenecks prohibit the application of Hoeffding-type inequality in the analysis of learning algorithms in dynamic systems which are commonly found in reinforcement learning. Various research efforts have been dedicated to the relaxation of uniform ergodicity used in \citet{glynn2002hoeffding}. In \citet{miasojedow2014hoeffding} and \citet{fan2021hoeffding}, the Markov transition kernel is viewed as a linear operator on $L^2(\pi)$ where $\pi$ is the invariant distribution. By assuming a positive spectral gap of the transition kernel, the authors are able to show Hoeffding's inequalities for geometrically ergodic Markov chains. A recent study \citep{sandric2021hoeffding} shows a version of Hoeffding's inequality for Markov chains by using the Wasserstein metric in place of the total variation (TV) distance. By considering a weaker metric for the Markov chain convergence condition, Hoeffding's inequality now can be extended to non-irreducible Markov chains which are not covered by the classical ergodicity conditions (Example \ref{exa: non-irreducible mc}). 

Despite the various research efforts made to relax the uniform ergodicity, they still leave something to be desired---\textit{does there exist a flexible notion of ergodicity that generalizes the existing ones and adapts to different use cases in practice?} We answer this question positively by proposing a novel notion of Markov chains convergence via integral probability metric (IPM) \citep{muller1997integral} and studying Hoeffding's inequality for Markov chains converging in general integral probability metric. IPM is a statistical distance function that has been successfully applied to generative adversarial networks (GAN), distributionally robust optimization (DRO), and other machine learning problems (see Section \ref{sec: ipm applications}). It can be intuitively thought of as the maximal discrepancy between to probability distribution seen by a class of functions $\mbf F$, or generators:
\begin{align*}
    \ipm_\mbf F(\mu, \nu) = \sup_{f\in\mbf F} |\bb E_\mu f - \bb E_\nu f|,
\end{align*}
where $f$ is referred to as the witness function. IPM provides a framework that recovers many probability metrics by choosing the appropriate function class (or generator) $\mbf F$ (Table \ref{tab: probability metrics in IPM form}). The strength of the probability metric can be easily controlled by the size of $\mbf F$. This allows us to extend the convergence of Markov chains in terms of total variation metric to a wide array of distance functions. Our main results, formally stated in Theorem \ref{thm: main result time dep} and \ref{thm: main result time indep} show that, depending on the function applied to the Markov chain, Hoeffding-type inequality holds when the Markov chain converges in a weaker sense (than TV metric). In most cases, it is unnecessary to enforce the ergodicity on the Markov chain if the functions of interest lie in a smaller set than bounded measurable functions. Thus, our results consider richer classes of Markov chains and the corresponding function classes for computing the concentration of measure results, which have not been considered before. Additionally, we refine the existing results by providing a sharper tail bound than those assuming the ergodicity. Our assumption on the Markov is adaptive to the functions applied to the Markov chain, which makes our theory a flexible framework that generalizes many existing studies.

The remainder of the paper is organized as follows. The rest of Section \ref{sec: introduction} contains the summary of the main contributions of this paper and an introduction to the notations. Section \ref{sec: preliminary} is dedicated to the review of the integral probability metric and its application. Section \ref{sec: main results} presents the main results of this paper and the proofs are deferred to the Appendix. Section \ref{sec: comparison} reviews the existing Hoeffding-type inequalities in the literature and establishes their connections with our framework. Applications of Hoeffding's inequality to machine learning are presented in Section \ref{sec: application}. 

\begin{table}[t]
    \centering
    \begin{tabular}{|c|c|c|}
        \hline
        \textbf{Assumptions on Markov Chains} &  \textbf{Reference}\\
        \hline
        Doeblin's condition/Uniformly ergodic & \citet{glynn2002hoeffding}\\
        \hline
        V-Uniformly ergodic & \citet{douc2011consistency}\\
        \hline
        Geometric ergodic & \citet{miasojedow2014hoeffding}\\
        \hline
        Non-irreducible converging in Wasserstein metric & \citet{sandric2021hoeffding}\\
        \hline
        Periodic & \citet{liu2021hoeffding}\\
        \hline
    \end{tabular}
    \caption{Existing Hoeffding's inequalities for Markov chains}
    \label{tab:my_label}
\end{table}

\subsection{Contributions}
The contributions of this paper are summarized as follows.

\begin{enumerate}
    \item\textit{Extending the Notion for Markov Chain Convergence.}
    We propose, via IPM, the \textbf{generalized concentrability condition} and\textbf{ IPM Dobrushin coefficient} for general state space Markov chain. We argue that these two new concepts along with IPM itself should be treated as a framework to study the convergence of Markov chains. We demonstrate their unifying power by recovering and extending the traditional conditions on Markov chains assumed for concentration inequalities, such as the uniform ergodicity \citep{douc2018markov, meyn2012markov}, V-uniformly ergodicity \citep{douc2011consistency}, Wasserstein ergodicity \citep{sandric2021hoeffding, rudolf2018perturbation}, and weak convergence \citep{diaconis1999iterated, breiman1960strong}. 
    
    The weak convergence of Markov chains to an invariant distribution is studied in \citet{diaconis1999iterated} and \citet{breiman1960strong} under the iterated random operator theory. The convergence of Markov chains in the Wasserstein metric is studied in \citet{hairer2011asymptotic} and \citet{butkovsky2014subgeometric}. Our motivation for further expanding the notion of Markov chain convergence is to increase the flexibility of the theory to accommodate the ever-changing real-world challenges. Characterizing the Markov chain convergence with IPM is suitable for this task as the strength of the distance function can be controlled via the choice of generators. 

    \item\textit{Hoeffding's Inequality under Weaker Assumptions on the Markov Chain.}
    We relax the assumptions on the Markov chains required by Hoeffding-type inequalities. In the main results, we emphasize the link between the functions applied on the Markov chains and the generator of the IPM. Armed with this insight, one can identify the unnecessarily strong conditions used in some existing works. For example, the empirical mean of a bounded Lipschitz function concentrates around its mean (under invariant distribution) with Hoeffding-type tail bound as soon as the Markov chain converges in the bounded Lipschitz metric under a certain speed. However, according to the existing state-of-the-art result \citep[see][Theorem 1]{sandric2021hoeffding}, the convergence needs to happen under the Wasserstein-1 metric which is known to be slightly stronger. 
    
    \item\textit{Unifying and Strengthening Different Versions of Concentration Inequalities.}
    Our work unifies, and in some cases strengthens, multiple previous results on Hoeffding's inequality for general state space Markov chains. Many existing results on Hoeffding-type inequalities for Markov chains emerge as special cases under our framework, including \citet{glynn2002hoeffding, douc2011consistency, fan2021hoeffding, sandric2021hoeffding} to name a few. For example, we strengthen the Hoeffding inequality for uniformly ergodic Markov chain derived in \citet{glynn2002hoeffding} by a tighter upper bound on the martingale difference term and exploit the Hoeffding lemma for martingales with bounded martingale differences. This also allowed us to derive the constants in the tail bound. In Section \ref{sec: comparison}, we establish the link with various Hoeffding-type inequalities in the literature. 

    Our tail bounds also feature explicit constants, whereas some results in the literature only provide the existence of constants \citep[see][Chapter 23]{douc2018markov} rendering them less convenient to use in practice.
    
    \item\textit{Applications to Learning Problems.}
    We provide three examples of machine learning and reinforcement learning in Section \ref{sec: application}, where our concentration result is used to weaken the hypotheses under which the results have been proved previously. The first example covers the well-studied generalization error bound for empirical risk minimization with non-i.i.d. samples \citep{agarwal2012generalization}. We derive the generalization bound of empirical risk minimization with Markovian samples, where the sample-generating Markov chain converges in IPM. This result extends the result in \citet{agarwal2012generalization}. The second example tackles the finite sample analysis of averaging of iterates of the stochastic gradient descent (SGD), extending the results in \citet{gupta2020some}. The last example establishes the logarithmic regret bound for general state space Markovian bandit under a UCB-type algorithm, which extends the existing finite state results reported in \citet{anantharam1987asymptotically} and \citet{tekin2012online}. 

\end{enumerate}
\subsection{Notations}
\label{sec: notations}
Consider a measure space $(\bb{X}, \cal X, \zeta)$ equipped with Borel $\sigma$-algebra $\cal X$ and a non-negative $\sigma$-finite measure $\zeta$. Assume $\bb{X}$ is a Polish space, that is, a separable completely metrizable topological space with metric $\mathsf d$. Consider a measurable weight function $V: \bb{X} \to [1, \infty)$ and define the $V$-norm for measurable functions $f: \bb{X}\to \bb{R}$ as
\begin{align}
    \|f\|_V = \sup_{x\in\bb{X}}\frac{|f(x)|}{V(x)}.
\end{align}
The integral of a function is written as $\mu(f) = \int_{\bb{X}} f d\mu$. Let $\bb{M}(\bb{X})$ denote the set of measurable real-valued function on $\bb{X}$ and $\bb{M}_V(\bb{X}) = \{f\in\bb{M}(\bb{X}) : \|f\|_V<\infty\}$. Let $\cal P(\cal X)$ denote the set of probability measures on $\cal X$ and $\cal P_V(\cal X) = \{\mu \in\cal P(\cal X) : \mu(V) < \infty\}$. When $V\equiv 1$ for $\forall x\in\bb{X}$, $\bb{M}_1(\bb{X})$ is denoted as $L^\infty(\bb{X}, \zeta)$ or $L^\infty(\zeta)$ when there is not ambiguity, and $V$-norm on $\bb{M}_1(\bb{X})$ is denoted as $\|\cdot\|_\infty$, i.e. the uniform norm. Note that $\cal P_1(\cal X)$ is the set of probability measures on $\cal X$ with finite first moment. Let $\bb{L}(\bb{X})$ denote the set of real-valued Lipschitz continuous functions on $\bb{X}$ and $\lip(f)$ denote the Lipschitz constant of $f\in\bb{L}(\bb{X})$. Let $\bb{C}_b(\bb{X})$ denote the set of real-valued continuous bounded functions on $\bb{X}$. Let $E$ be a set and scaling of $E$ by $a\in[0, \infty)$ is written as $a\cdot E = \{ax : x \in E\}$.

Let $\{X_n\}_{n\geq 0}$ be a Markov chain on $\bb{X}$ with Markov transition kernel $P:\bb{X}\times\cal X \to [0, 1]$. We assume the Markov chain to be time-homogeneous throughout the paper. Denote the $n$ step transition probability as $P^n(x, dy) = \bb{P}(X_n\in dy | X_0 = x)$ for $n \geq 1$. We write $\bb{P}_{\mu}(X_n)$ as the probability distribution of $X_n$ given initial state distribution $\mu$. If $\mu = \delta_x$ for some $x\in\bb{X}$, where $\delta_x$ is the Dirac measure at $x$, we also write it as $\bb{P}_{x}(X_n)$. For a $f\in\bb M(\bb X)$, we denote $P^nf(x) = \int_{\bb X}f(x')P^n(x, dx')$ for $n\geq 1$.

Let $a\vee b \coloneqq \max\{a, b\}$ and $a\wedge b\coloneqq \min\{a, b\}$. Let $\bb{R}_+ = \{x\in\bb{R}: x\geq 0\}$ be the set of non-negative real numbers. Let $\bb{Z}_+ = \{x\in\bb{Z}: x > 0\}$ be the set of positive integers. Let $\bb{N}$ denote the set of natural numbers and we follow the convention that $0\in\bb{N}$.

We summarize the notations for various distance functions used throughout the paper in Table \ref{tab: prob metric notation list}.
\begin{table}[t]
    \centering
    \begin{tabular}{|c|c|}
         \hline
         Notation & Description \\
         \hline
         $\msf d$ & metric for a Polish space\\
         \hline
         $\tv$ & Total variation (TV) metric \\
         \hline
         $\wass_{p, \msf d}$ & Wasserstein-($p$, $\msf d$) metric\\
         \hline
         $\|\cdot\|_\bl$ & Bounded Lipschitz norm of real-valued functions\\
         \hline
         $\bl$ & Bounded Lipschitz metric between probability measures\\
         \hline
         $\|\cdot\|_\ltpi$ & $\ltpi$-norm for functions and probability measures\\
         \hline
    \end{tabular}
    \caption{List of notation for distance functions}
    \label{tab: prob metric notation list}
\end{table}

\section{Overview of IPM}
\label{sec: preliminary}

To facilitate our discussion on IPM, we briefly recall its fundamental properties from \citet{muller1997integral} where the term \textit{integral probability metric} was first coined. Before that, it was known as the \textit{probability metric with a $\zeta$-structure} \citep{zolotarev1984probability}. 

\subsection{Definition and Properties}
An IPM is characterized by its generating function class (or generator) as shown in Definition \ref{def: ipm}. For general choice of $\mbf F$, $\ipm_{\mbf F}$ is a pseudometric with the following properties:
\begin{enumerate}
    \item $\ipm_{\mbf F}(\mu_1, \mu_1) = 0$, \label{weak positive definiteness}
    \item $\ipm_{\mbf F}(\mu_1, \mu_2) = \ipm_{\mbf F}(\mu_2, \mu_1)$,
    \item $\ipm_{\mbf F}(\mu_1, \mu_3) \leq \ipm_{\mbf F}(\mu_1, \mu_2) + \ipm_{\mbf F}(\mu_2, \mu_3)$,
\end{enumerate}
for all $\mu_1, \mu_2, \mu_3 \in\cal P_V(\cal X)$. $\ipm_{\mbf F}$ becomes a metric when $\ipm_{\mbf F}(\mu_1, \mu_2) = 0 \Leftrightarrow \mu_1 = \mu_2$. Examples of $\mbf F$ such that $\ipm_\mbf F$ is a metric are provided in Table \ref{tab: probability metrics in IPM form}. For instance,  when $\mbf F$ is a reproducing kernel Hilbert space (RKHS) \citet{aronszajn1950theory}, $\ipm_{\mbf F}$ is a metric if and only if the kernel function associated with the RKHS is characteristic \citep{sriperumbudur2010hilbert}. To the best of our knowledge, a general sufficient condition on $\mbf F$ for $\ipm_{\mbf F}$ to be a metric is still an open problem in the literature. 
\begin{definition}
\label{def: ipm}
An \text{integral probability (pseudo)metric} $\ipm_{\mbf F}$ on $\cal P_V(\cal X)$ generated by a function class $\mbf F \subseteq \bb{M}_V(\bb{X})$ is defined as
\begin{align}
\label{eqn: ipm def}
    \ipm_{\mbf F}(\mu, \nu) \coloneqq \sup_{h\in\mbf F}|\mu(f) - \nu(f)|, \quad \text{for }\mu,\ \nu\in\cal P_V(\cal X),
\end{align} 
where $f$ is referred to as the witness function, and $\mbf F$ is called the generator. 
\end{definition}

\begin{definition}
\label{def: maximal generator}
    Let $\mbf F\subseteq \bb{M}_V(\bb{X})$. We define the \textbf{maximal generator} of $\mbf F$ as 
    \begin{align}
        \mbf G_{\mbf F} = \{f\in\bb{M}_V(\bb{X}) : |\mu(f) - \nu(f)| \leq \ipm_{\mbf F}(\mu, \nu),\ \text{for all } \mu, \nu\in\cal P_V(\cal X)\}.
    \end{align}
\end{definition}
Given generator $\mbf F$, one can define a unique maximal generator as in Definition \ref{def: maximal generator}. The definition of the maximal generator allows the comparison between IPMs with different generators, which a formally stated in the following propositions.
\begin{prop}
\label{thm: equivalent ipm}
    Let $\mbf F \subseteq \mbf D \subseteq \bb{M}_V(\bb{X})$, and $\mu,\ \nu \in \cal P_V(\cal X)$. The following statements hold
    \begin{enumerate}
        \item $\ipm_{\mbf F}(\mu, \nu) \leq \ipm_{\mbf D}(\mu, \nu)$.
        \item $\mbf G_{\mbf F} \subseteq \mbf G_{\mbf D}$.
        \item If $\mbf D \subseteq \mbf G_{\mbf F}$, then $\ipm_{\mbf F}$ and $\ipm_{\mbf D}$ are identical.
    \end{enumerate}
\end{prop}
\begin{prop}
\label{thm: characterization of maximum generator}
    Let $\mbf F \subseteq \bb{M}_V(\bb{X})$. Then:
    \begin{enumerate}
        \item $\mbf G_{\mbf F}$ contains the convex hull of $\mbf F$;
        \item If $f\in\mbf G_{\mbf F}$, then $\alpha f+ \beta \in\mbf G_{\mbf F}$, for all $\alpha\in[-1, 1]$ and $\beta\in\bb{R}$;
        \item $\mbf G_{\mbf F}$ is closed under the uniform norm.
    \end{enumerate}
\end{prop}
When $\mbf F = \{\bbm{1}_B : B\in\cal X\}$ and $\mbf D = \{\|f\|_\infty \leq 1 : f\in\bb{M}_V(\bb{X})\}$, we can see that $\mbf D \subseteq \mbf G_\mbf F$ as $\mbf G_\mbf F$ contains all absolute convex combination of $\mbf F$ and is closed under the uniform norm by Proposition \ref{thm: characterization of maximum generator}. Using Proposition \ref{thm: equivalent ipm} (3), we can conclude that $\ipm_\mbf F = \ipm_\mbf D$. In fact, they are equivalent representations of the TV metric (differ by a factor of 2). 
\begin{prop}
\label{prop: absolute convexity of max gen}
    If $\mbf F \subseteq \mbf D \subseteq \bb{M}_V(\bb{X})$, and $\mbf D$  is absolutely convex, contains the constant functions, and is closed with respect to pointwise convergence, then $\mbf D = \mbf G_{\mbf F}$.
\end{prop}
\begin{table}[H]
    \centering
    \begin{tabular}{|C{0.34\textwidth}|p{0.5\textwidth}| p{0.1\textwidth}|}
        \hline
        Function class $\mbf F$&  Description of $\ipm_{\mbf F}$ & Notation\\
        \hline
        $\{f\in\bb{C}_b(\bb{X}): \|f\|_\infty \leq 1\}$ & Continuous and bounded function on $\bb X$. Convergence in this IPM implies weak convergence \citep{shorack2000probability}. & \\
        \hline
        $\{\bbm{1}_B : B\in\cal X\}$ or $\{f: \|f\|_\infty \leq 1\}$ & It is 2 times the total variation distance. & $\tv$\\
        \hline
        $\{f\in\bb M(\bb X): \|f\|_\infty + \lip(f) \leq 1\}$ & It is the Dudley metric \citep{dudley2018real} which metrizes the weak topology on Polish spaces. & \\
        \hline
        $\{f\in\bb M(\bb X): \max\{\|f\|_\infty, \lip(f)\}\leq 1\}$ & It is the bounded Lipschitz metric \citep{hille2022explicit,czapla2024central} which is equivalent to Dudley metric. & $\bl$\\
        \hline
        $\{f\in\bb M(\bb X): \lip(f) \leq 1\}$ & It is the Kantorovich–Rubinstein dual representation of the Wasserstein-1 metric on a Polish space 
        \citep{edwards2011kantorovich}. & $\wass_1$\\
        \hline
        $\{f\in\bb M(\bb X): \|f\|_{\bb H_k} \leq 1\}$ & Closed unit ball in a RKHS $\bb H_k$ with kernel $k$. It is the maximum mean discrepancy. & $\mmd_k$.\\
        \hline
        $\{f\in\bb M(\bb X): \|f\|_{\bb W} \leq 1\}$ & Closed unit ball in a Sobolev space $\bb W$. This is the Sobolev IPM used in the Sobolev GAN \citep{mroueh2017sobolev}. & \\
        \hline
    \end{tabular}
    \caption{Probability metrics in the form of IPM}
    \label{tab: probability metrics in IPM form}
\end{table}

\subsection{Connections with Classical Probability Metrics and Weak Convergence}
As a general form of probability distance function, IPM recovers many classical probability metrics with different choices of generators. We collect some of the most well-known probability metrics and their generators in Table \ref{tab: probability metrics in IPM form}. 

One of the most important properties of probability metrics is the metrization of weak topology on the space of probability measures. A sequence of probability measure $\mu_n$ converge weakly to $\mu$ if $\mu_n(f) \to \mu(f)$ for all $f\in\bb{C}_b(\bb{X})$ \citep{billingsley2013convergence}, i.e., the continuous and bounded functions on $\bb{X}$. A probability metric is said to metrize the weak topology if convergence in this probability metric is equivalent to weak convergence. 

\begin{definition}
\label{def: weak convergence properties}
Let $\bb{X}$ be a Polish space and let $V: \bb{X}\to[1,\infty)$ be any weight function. $\ipm_{\mbf F}$ is said to have:
    \begin{enumerate}
        \item Property ($W_1$): if $\lim_{n\to \infty}\ipm_{\mbf F}(\mu_n, \mu) =0$ $\Leftrightarrow$ $\mu_n$ converges to $\mu$ weakly;
        \item Property ($W_2$): if $\lim_{n\to\infty}\ipm_\mbf F(\mu_n, \mu) = 0$ for all $\{\mu_n\}_{n\in\bb N} \subset\cal P_V(\cal X)$ such that $\mu_n \to \mu \in\cal P_V(\cal X)$ weakly as $n\to\infty$;
        \item Property ($W_3$): if $\lim\inf_{n\to\infty}\ipm_{\mbf F}(\mu_n, \nu_n) \geq \ipm_{\mbf F}(\mu, \nu)$, for all $(\mu_n),\ (\nu_n)\subset\cal P_V(\cal X)$ such that $\mu_n \to \mu $ and $\nu_n \to \nu$ weakly as $n\to\infty$ with limits $\mu,\nu\in\cal P_V(\cal X)$.
    \end{enumerate}
\end{definition}
Property ($W_1$) means $\ipm_{\bb F}$ is equivalent to weak convergence or induces a weak topology on $\cal P_V(\cal X)$. Property ($W_2$)

\begin{prop}
\label{thm: conditions for weak convergence properties}
Let $W_1$, $W_2$, and $W_3$ be defined above. The following holds:
\begin{enumerate}
    \item If $\mbf F$ has a uniformly bounded span, is equicontinuous, and contains the function
    \begin{align}
        x\mapsto f_{n, A}(x) \coloneqq \max\{0, 1/n-d(x, A)\},\nn
    \end{align}
    for every closed set $A\subset\bb{X}$ and all $n\in\bb{N}$, then $\ipm_{\mbf F}$ has property $(W_1)$.
    \item An $\ipm_{\mbf F}$ has property $(W_2)$ if and only if $\mbf F$ is equicontinuous and has uniformly bounded span. 
    \item If $\mbf F \subset \bb{C}_b(\bb{X})$, then $(W_3)$ holds.
\end{enumerate}
\end{prop}
\begin{remark}
With $\mbf F = \{f\in\bb{C}_b(\bb{X}): \|f\|_\infty \leq 1\}$, $\ipm_\mbf F$ implies weak convergence. However, the converse ($W_2$ in Definition \ref{def: weak convergence properties}) fails to hold due to Proposition \ref{thm: conditions for weak convergence properties}, as $\mbf F$ is uniformly bounded but not equicontinuous. However, one can verify that the generator of the Dudley metric has these two properties due to the additional constraint on the Lipschitz constant. Thus, the Dudley metric is known to induce weak topology on the space of probability measures \citep{dudley2018real}.
\end{remark}

\subsection{Applications of IPM in Learning and Optimization}
\label{sec: ipm applications}

IPM has received much attention from the statistics and machine learning communities in recent years. Three of the most notable use cases are the formulation of critics in generative adversarial networks (GAN), the formulation of the ambiguity set in distributionally robust optimization, and hypothesis testing. 

The first GAN \citep{goodfellow2020generative} uses a binary classifier trained with cross-entropy loss (KL divergence) as its critic. Since then, numerous critic designs have been proposed based on different probability distance functions other than KL divergence. Some belong to the family of $f$-divergence\footnote{Also known as $\phi$-divergence in some papers.}, such as the $f$-GAN \citep{nowozin2016f} which includes the original GAN as a special case. Others fall into the family of IPM, including MMD-GAN \citep[e.g.,][]{dziugaite2015training, li2017mmd, arbel2018gradient, binkowski2018demystifying}, Wasserstein-GAN \citep[e.g.,][]{arjovsky2017wasserstein, gulrajani2017improved, adler2018banach, xu2021towards}, and Sobolev GAN \citep{mroueh2017sobolev}. These successful developments exploit the flexible framework by carefully picking the IPM generators such that the resulting training objectives have desirable computational properties and regularization effects. Sobolev-GAN \citep{mroueh2017sobolev} is one such example. Based on a novel Sobolev IPM proposed in the same paper, Sobolev-GAN is empirically shown to be better suited for sequence generation compared to Wasserstein-GAN due to its conditional CDF matching structure \citep{mroueh2017sobolev}. Additionally, Sobolev IPM enforces an intrinsic smoothness constraint \citep{mroueh2017sobolev} which corresponds to the gradient penalty used in Wasserstein-GAN \citep{gulrajani2017improved}. Unlike Wasserstein distance, Sobolev IPM enjoys a closed form that allows easy estimation from data. Other IPMs with novel generators used in GAN training include Cramér distance in Cramér-GAN \citep{bellemare2017cramer} and the Fisher IPM in Fisher-GAN \citep{mroueh2017fisher}. 

Distributionally robust optimization (DRO) is another topic that popularizes the use of IPM. For a comprehensive review of DRO, we refer readers to the recent survey papers \citet{rahimian2019distributionally} and \citet{lin2022distributionally}. The connection between DRO and IPM lies in the construction of the ambiguity set\footnote{Also know as the uncertainty set.} which is a set of probability measures within a certain distance from a predefined nominal probability measure. Traditionally, the distance function used for the ambiguity set is the KL-divergence which yields a dual problem that is finite-dimensional and convex \citep{hu2013kullback}. More recent studies have considered distance functions from the IPM family, such as the MMD-DRO \citep[e.g.,][]{zhu2021kernel, staib2019distributionally} and Wasserstein-DRO \citep[e.g.,][]{mohajerin2018data, blanchet2019quantifying, gao2023distributionally, gao2022wasserstein}. DRO under general IPM is studied in \citet{husain2020distributional} where the authors unify the previous studies and reveal an interesting link between IPM-DRO and the previously mentioned GANs with different IPMs. 

The use of MMD is well-celebrated in statistical hypothesis testing, such as two-sample tests \citep[e.g.,][]{gretton2006kernel, gretton2012kernel} and online change detection \citep[e.g.,][]{li2015m, flynn2019change, arlot2019kernel}, due to its superior performance. Other applications of IPM include AI fairness \citep{kim2022learning} with a novel generator optimized for computational performance, quantum machine learning using MMD as the loss function \citep{PhysRevA.98.062324}, etc.

The computation property is one of the main reasons IPM is well-received across many communities in machine learning. For empirical estimation of different instances of IPM, please refer to \citet{sriperumbudur2012empirical} and \cite{sriperumbudur2010non}. For the dimension-free approximation of IPM, see \citet{han2021class}.

\section{Main Results}
\label{sec: main results}
In this section, we present the main result of this paper: Hoeffding's inequality for Markov chains converging in $\ipm_{\mbf F}$ for a nonempty generator $\mbf F \subseteq \bb M(\bb X)$. To avoid unnecessary notation clutter, the generator $\mbf F$ of an IPM is always treated as its maximal generator $\mbf G_{\mbf F}$ in the sequel unless stated otherwise. Thus, we use $\ipm_{\mbf F}$ in place of $\ipm_{\mbf G_{\mbf F}}$ whenever possible. 

In section \ref{sec: gen con}, we state the Hoeffding-type inequality under the generalized concentrability condition. In section \ref{sec: dobrushin}, we establish the connection between the concentrability constant and Dobrushin coefficient. 

\subsection{Hoeffding's Inequality under the Generalized Concentrability}
\label{sec: gen con}

Consider a Markov transition kernel $P: \bb X\times \cal X \to  [0, 1]$ with invariant probability measure $\mu\in\cal P(\cal X)$. The kernel $P$ is said to satisfy the\textbf{ generalized concentrability condition} under $\ipm_\mbf F$ if 
\begin{align}
\label{eqn: gen con}
     \Gamma_{\mbf F} \coloneqq \sup_{x\in\bb X} \sum_{i=1}^\infty\ipm_{\mbf F}(P^i(x, \cdot), \pi) < \infty,
\end{align}
and $\Gamma_{\mbf F}$ is called the \textbf{concentrability constant} of $P$.

In Theorem \ref{thm: main result time dep}, we state Hoeffding's inequality with time-dependent functions of the Markov chain under the generalized concentrability condition. The proof is deferred until Appendix \ref{app: proof of ipm hoeffding time dep}. The functions applied to the Markov chain are assumed to have bounded spans. Particularly, at time $i\in\bb N$, let $f_i: \bb X\to [a_i, b_i]$ be the function applied the Markov chain, where $a_i\leq b_i$ and $a_i, b_i\in\bb R$. Let $\sp(f_i) = b_i - a_i < \infty$ denote the span of $f_i$. We define $M_i = \inf\{m>0: f_i \in m\cdot \mbf F\}$ as the \textbf{minimal stretch} of $\mbf F$ to include $f_i$. A similar definition is also used in \citet[Definition 2]{husain2020distributional}. We subscript the probability and expectation notation $\bb{P}_\mu, \bb{E}_\mu$ with a probability measure $\mu$ to indicate the initial distribution of the Markov chain. 

\begin{theorem}[Time-dependent case]
\label{thm: main result time dep}
    Let $\{X_n\}_{n\in\bb{N}} \subset\bb{X}$ be a Markov chain with transition kernel $P$. Assume that the Markov kernel $P$ admits a unique invariant probability measure $\pi$ such that Equation \ref{eqn: gen con} is satisfied. For $i\in\{0, \cdots, n-1\}$ and $n\in\bb{Z}_+$, suppose $f_i$ has bounded span with minimal stretch $M_i\in(0, \infty)$ such that $f_i\in M_i \cdot \mbf F$. Let $\tilde{S}_{n} \coloneqq \sum_{i=0}^{n-1} f_i(X_i)$. 
    
    Then, $\epsilon >  n^{-1}[M_1^{\max}\Gamma_{\mbf F} + \sp(f_0)]$ and any initial distribution $\mu\in\cal P(\cal X)$, we have
    \begin{align}
    \label{eqn: Hoeffding ineq IPM time dep}
        \bb{P}_\mu\bigg[\bigg|\Tilde{S}_n - \sum_{i=0}^{n-1} \pi(f_i)\bigg| > n\epsilon\bigg]\leq 2\exp\bigg(-\frac{2\{n\epsilon - [M_1^{\max}\Gamma_{\mbf F} + \sp(f_0)]\}^2}{\sum_{i=0}^{n-1} [M_{i+1}^{\max}\Gamma_{\mbf F} + \sp(f_i)]^2}\bigg),
    \end{align}
    where $M_i^{\max} = \max\{M_i, \cdots, M_{n-1}\}$ for $i\in\{0, \cdots, n-1\}$ and $M_i^{\max} = 0$ for $i\geq n$.
\end{theorem}
\begin{proof}
    See Appendix \ref{app: proof of ipm hoeffding time dep}.
\end{proof}
Markov chains satisfying Equation \ref{eqn: gen con} are called Markov chains converging in $\ipm_{\mbf F}$ with constant $\Gamma_{\mbf F}$. $\Gamma_\mbf F$ defined here is similar to the hypothesis in \citet{sandric2021hoeffding} (also see $\Gamma_\wass$ in Equation \ref{eqn: gamma for wasserstein}). In fact, one can recover their result by picking $\mbf F = \{f\in\bb M(\bb X): \lip(f) \leq 1\}$. We will demonstrate how our framework recovers several existing Hoeffding-type inequalities in the literature in Section \ref{sec: comparison}. Before that, some remarks are in order.

\begin{remark}
\label{rmk: main insight}
    The functions $f_i$ applied to the Markov chain are linked to the generator $\mbf F$ of the IPM by a multiplicative factor $M_i$. This allows the users to adaptively choose the strength of the IPM according to the functions used on the Markov chain. 
\end{remark}

\begin{remark}
    The numerator of the exponent in Equation \ref{eqn: Hoeffding ineq IPM time dep} is different from the i.i.d. case in Equation \ref{eqn: og hoeffding inequality}. This is due to $\tilde S_n$ is centered at $\sum_{i=0}^{n-1} \pi(f_i)$ instead of $\bb E_{\mu}\sum_{i=0}^{n-1} f_i(X_i)$ in the tail probability.
\end{remark}

\begin{remark}
    The supremum in Equation \ref{eqn: gen con} may be considered restrictive by some. However, our proof technique requires that in the most general case of $\mbf F$. We discuss an alternative condition in Section \ref{sec: dobrushin} with additional constraints of $\mbf F$.
\end{remark}

    Our tail bound in Equation \ref{eqn: Hoeffding ineq IPM time dep} does not have obscure constants. Some results such as Theorem \ref{thm: v hoeffding ineq} and the ones presented in \citet[Chapter 23.3]{douc2018markov} contains some constants without explicit forms, which are less convenient for computation purposes.
    The proof technique we employed combined with the use of function span makes our tail bound tighter than certain existing results as explained in Remark \ref{rmk: gamma tv}.

    

As a byproduct of the time-dependent case, we recover the time-independent case in Theorem \ref{thm: main result time indep}. The proof of Theorem \ref{thm: main result time dep} can be reused with $f_i =f$ for $i\in\{0,\cdots,n-1\}$. 

\begin{theorem}[Time-independent case]
\label{thm: main result time indep}
    Let $\{X_n\}_{n\in\bb{N}} \subset\bb{X}$ be a Markov chain with transition kernel $P$. Assume that the Markov kernel $P$ admits a unique invariant probability measure $\pi$ and satisfies Equation \ref{eqn: gen con}. Suppose $f$ has bounded span with minimal stretch $M\in (0, \infty)$ such that $f \in M\cdot \mbf F$. For $n\in\bb{Z}_+$, let $S_{n} \coloneqq \sum_{i=1}^{n} f(X_i)$. 
    
    Then, for $\epsilon > n^{-1}[M\Gamma_{\mbf F} + \sp(f)]$ and any initial distribution $\mu\in\cal P(\cal X)$, we have,
    \begin{align}
    \label{eqn: ipm tail bound time indep}
        \bb{P}_\mu[|S_n - n\pi(f)| > n\epsilon]\leq 2\exp \bigg\{-\frac{2\{n\epsilon - [M\Gamma_{\mbf F} + \sp(f)]\}^2}{n[M\Gamma_{\mbf F} + \sp(f)]^2}\bigg\}.
    \end{align}
\end{theorem}

\subsection{Relation to Dobrushin Coefficient}
\label{sec: dobrushin}

Dobrushin contraction coefficient was first proposed by R. Dobrushin \citep{dobrushin1956central} for discrete state space Markov chains. Later studies by \citet{del2003contraction} extended the contraction coefficient to general state space Markov chains. The traditional Dobrushin coefficient describes the contraction rate of the Markov kernel on the space of probability measures with respect to the TV metric. The formal definition can be found in Chapter 18 of \citet{douc2018markov}. \cite{rudolf2018perturbation} studied the properties of the coefficient where $\ipm_\mbf F$ is the Wasserstein metric and called it the \textit{generalized Dobrushin coefficient} of the Markov kernel.

In this section, we establish the connection between the Dobrushin coefficient and the concentrability constant. It is beneficial to the understanding of the generalized concentrability condition as it provides a way to estimate the concentrability constant under scenarios. As demonstrated later in section \ref{sec: average of rsa}, the concentrability constant of a Markov chain generated by a random operator can be estimated from the contraction coefficient of the operator. 

We start by defining the \textbf{IPM Dobrushin coefficient} as the analog of the usual Dobrushin coefficient. 

\begin{definition}[IPM Dobrushin Coefficient]
\label{def: ipm dobrushin coefficient}
    The IPM Dobrushin contraction coefficient of $P$ with generator $\mbf F$ is defined as 
    \begin{align}
        \Delta_\mbf F (P)= \sup\{\ipm_\mbf F(\xi P, \xi' P) : \xi, \xi'\in\cal P(\cal X), 0<\ipm_\mbf F(\xi, \xi')\leq 1\}.
    \end{align}
\end{definition}

\begin{remark}
    One can recover the usual definition by simply setting $\mbf F = \mbf F_\tv \coloneqq \{f\in\bb M(\bb X): \|f\|_\infty\leq 1\}$.
\end{remark}

The following two definitions are crucial to the connection between the IPM Dobrushin coefficient and the concentrability constant. 
\begin{definition}[Stability of generator]
\label{def: stability of gen}
    For $f\in\mbf F$, suppose there exists $\rho = \inf \{ m > 0: Pf \in m \cdot \mbf F\}$\footnote{ It suffices to consider the positive stretch of $\mbf F$, since $\mbf F$ is assumed to the maximal generator which implies it is a balanced convex hull.}. $\mbf F$ is said to be sable by $P$ if $\rho \leq 1$.
\end{definition}
\begin{definition}[Nonexpansiveness in IPM]
    $P$ is said to be nonexpansive in $\ipm_\mbf F$ if for any $\xi, \xi' \in \cal P(\cal X)$
    \begin{align*}
        \ipm_\mbf F (\xi P, \xi' P) \leq \ipm_\mbf F (\xi, \xi')
    \end{align*}
\end{definition}
For some choices of $\mbf F$, Definition \ref{def: stability of gen} can be satisfied without additional conditions.
\begin{prop}
\label{thm: good choices of gen}
    Suppose $P$ admits invariant distribution $\pi$. The following choices of $\mbf F$ are stable by $P$,
    \begin{enumerate}
        \item $\bb B_\unif \coloneqq \{f\in\bb M(\bb X): \|f\|_\infty \leq 1\}$\footnote{$\bb B_\unif$ and $\bb B_{L^\infty(\pi)}$ are different, especially when $\pi$ is not supported on the whole space.};
        \item $\bb B_\lppi \coloneqq \{f\in\bb M(\bb X): \|f\|_\lppi\leq 1\}$ for $p\in[1, \infty]$.
    \end{enumerate}
\end{prop}
\begin{proof}
    1) It is obvious that $Pf = \int_{\bb X} f(y)P(x, dy) \leq \|f\|_\infty$. Thus, $f\in \bb B_\unif \Rightarrow Pf \in B_\unif$. 2) By \citet[Proposition 1.6.3]{douc2018markov}, the operator norm of $P$ on $\lppi$ is 1, and the proof is completed. 
\end{proof}
It is easy to determine that $P$ is nonexpansive when $\mbf F$ is stable by $P$ which is formally stated as follows.
\begin{prop}
    If $\mbf F$ is stable by $P$, then $P$ is nonexpansive in $\ipm_\mbf F$.
\end{prop}
\begin{proof}
    From the definition of $\ipm_\mbf F$ and nonexpansiveness, we have
    \begin{align*}
        \ipm_\mbf F (\xi P, \xi' P) &= \sup_{f\in\mbf F} \bigg|\int_{\bb X} f d(\xi P) - \int_{\bb X} f d(\xi' P)\bigg| = \sup_{f\in\mbf F} \bigg|\int_{\bb X} Pf d\xi - \int_{\bb X} Pf d\xi'\bigg|\\
        &\leq \sup_{f\in\mbf F} \bigg|\int_{\bb X} f d\xi - \int_{\bb X} f d\xi'\bigg| = \ipm_\mbf F (\xi, \xi'),
    \end{align*}
    where the last inequality follows from the fact that $\mbf F$ is stable by $P$.
\end{proof}
The above development allows us to upper-bound the concentrability constant with IPM Dobrushin coefficient. As we will see in Lemma \ref{thm: theta chain conv in wass} in the example of SGD (Section \ref{sec: average of rsa}), the Markov chain generated by SGD is contractive in Wasserstein-1 metric and the concentrability constant can be easily bounded by the corresponding IPM Dobrushin coefficient. 
\begin{prop}
\label{eqn: ipm dobrushin bounds concentrability constant}
    Suppose $P$ is nonexpansive in $\ipm_\mbf F$, we have for $\mu\in\cal P(\cal X)$,
    \begin{align*}
        \sum_{i=0}^\infty \ipm_\mbf F(\mu P^i, \pi) \leq \ipm_\mbf F(\mu, \pi) \sum_{i=0}^\infty \Delta_\mbf F(P^i).
    \end{align*}
\end{prop}
\begin{proof}
    Repeated application of the nonexpansiveness of $P$ yields the desired result.
\end{proof}

As mentioned in the previous section, it is tempting to remove the supremum in Equation \ref{eqn: gen con}. However, we cannot accomplish that in Theorem \ref{thm: main result time dep} for two reasons: (i) it is at the expanse of restricting the initial condition; (ii) our proof in Lemma \ref{lem: span of g_i} require the supremum to be bounded. Nevertheless, with the nonexpansiveness of $P$, we can at least circumvent (ii) in the following theorem. 

\begin{theorem}
\label{thm: main result dobrushin}
    Let $\{X_n\}_{n\in\bb{N}}$ be a Markov chain with transition kernel $P$. Assume that the Markov kernel $P$ admits a unique invariant probability measure $\pi$ such that
    \begin{align}
    \label{eqn: dobrushin con}
        \tilde\Gamma_\mbf F \coloneqq \sum_{i=0}^\infty\Delta_\mbf F(P^i) < \infty,
    \end{align}
    where $\mbf F$ is one of the choices in Proposition \ref{thm: good choices of gen}. For $i\in\{0,\cdots,n-1\}$ and $n\in\bb{Z}_+$, suppose $f_i$ has bounded span with minimal stretch $M_i\in(0, \infty)$ such that $f_i \in M_i\cdot\mbf F$. For $n\in\bb Z_+$, let $\tilde{S}_{n} \coloneqq \sum_{i=0}^{n-1} f_i(X_i)$. 
    
    Then, for $\epsilon > (2n)^{-1}\ipm_\mbf F(\mu, \pi)\tilde A_0$ and any initial distribution $\mu\in\cal P(\cal X)$ such that $\ipm_\mbf F (\mu, \pi) < \infty$, we have
    \begin{align}
    \label{eqn: Hoeffding ineq IPM time dep bdd}
        \bb{P}_\mu\bigg[\bigg|\Tilde{S}_n - \sum_{i=0}^{n-1} \pi(f_i)\bigg| > n\epsilon\bigg]\leq 2\exp\bigg(-\frac{2[n\epsilon - \ipm_\mbf F(\mu, \pi)\tilde A_0/2]^2}{\sum_{i=0}^{n-1} \tilde A_i^2}\bigg),
    \end{align}
    where
    \begin{align}
        \tilde A_i = 2\sum_{l=i}^{n-1} M_l\Delta_\mbf F(P^{l-i}) + \sp(f_i), \quad i\in\{0, \cdots, n-1\}.
    \end{align}
\end{theorem}
\begin{proof}
    See Appendix \ref{app: proof of main result dobrushin}.
\end{proof}

\section{Comparison to Existing Results}
\label{sec: comparison}
In this section, we demonstrate the generalization power of our result by recovering some notable Hoeffding's inequality for Markov chains in the recent literature. We focus our attention on the general state space setting, thus excluding results obtained under finite state space settings. According to the assumption on the Markov chains, we separate the existing Hoeffding-type inequalities into three categories: (i) uniformly ergodic chains, (ii) geometric ergodic chains, and (iii) non-irreducible chains. We start from the strongest assumption (uniform ergodicity) and move toward more relaxed assumptions. We demonstrate that our theory is a natural extension and generalizes the previous results.

\subsection{Ergodic Chains}
Traditionally, the ergodicity of the Markov chain is defined with respect to the \textit{total variation metric} (TV). The TV metric can be written as IPM with $\bb M_1(\bb{X})$ as its generator. For $\mu, \nu \in\cal P(\cal X)$, 
\begin{align}
\tv(\mu, \nu) \coloneqq \sup_{E\in\cal X} |\mu(E) - \nu(E)| = \frac{1}{2} \sup_{f\in\bb M_1(\bb{X})} |\mu(f) - \nu(f)|.
\end{align}
The definition of the TV metric indicates that it ranges between 0 and 1, whereas IPM is between 0 and 2 by Definition \ref{def: ipm}. We say an irreducible, aperiodic, and positive Markov kernel $P:\bb{X}\to\cal X$ is said to be geometrically ergodic if it converges to its invariant distribution $\pi$ in the following sense. For $n\in\bb{N}$ and $x\in\bb{X}$, there exist $\beta\in[0, 1)$ and measurable function $\alpha: \bb{X}\to[0, \infty)$ with $\pi(\alpha) < \infty$ such that
\begin{align}
    \|P^n(x, \cdot) - \pi\|_\tv \leq \alpha(x)\beta^{n}.
\end{align}
The Markov kernel $P$ is \textit{uniformly geometrically ergodic} if $\sup_{x\in\bb X}\alpha(x)$ is finite. Since uniform ergodicity implies geometric ergodicity \citep[Proposition 15.2.3]{douc2018markov}, we refer to uniform geometric ergodicity as uniform ergodicity in the sequel.

We start our discussion with uniformly ergodic Markov chains. We provide two versions of Hoeffding inequality with this type of Markov chains in Proposition \ref{thm: doebline Hoeffding inequality} and \ref{thm: tv d coeff Hoeffding inequality}. Proposition \ref{thm: doebline Hoeffding inequality} assumes \textit{Doeblin's condition} (Assumption \ref{asp: Doeblin condition}), and  Proposition \ref{thm: tv d coeff Hoeffding inequality} poses requirements on the \textit{TV Dobrushin coefficient}. For aperiodic Markov chains, Doeblin's condition (Assumption \ref{asp: Doeblin condition}) is equivalent to the uniform ergodicity of the Markov kernel $P$ \citep[Theorem 16.0.2]{meyn2012markov}. Thus, the latter is more general as it allows subgeometrically ergodic Markov chains. As an extension to the usual ergodic condition, Hoeffding's inequality under $V$-\textit{uniformly ergodicity} (Definition \ref{def: v uniform ergodicity}) is studied in \citet{douc2011consistency} which we present in Proposition \ref{thm: v hoeffding ineq}. Instead of the TV metric, the \textit{$V$-total variation metric} (Definition \ref{def: v norm for measures}) is used to measure the convergence. Additionally, \citet{miasojedow2014hoeffding} and \citet{fan2021hoeffding} view a Markov kernel as a bounded linear operator on $L^2(\pi)$ and study the Hoeffding-type inequalities under the \textit{$L^2(\pi)$-geometric ergodicity}. We summarize the result from \citet{fan2021hoeffding} in Proposition \ref{thm: l2 hoeffding ineq}. Additionally, certain periodic Markov chains can also satisfy Doeblin's condition. In fact, the extension of Hoeffding's inequality to periodic Markov chains was recently studied in \citet{liu2021hoeffding} which we omit here as it is similar in form to \citet{glynn2002hoeffding}.

We demonstrate in this section that the aforementioned results fit into our framework which allows easy interpolation and extension. 

\subsubsection{MC converging in TV Metric}
\label{sec: mc converging in tv}

In this section, we review the main result of \citet{glynn2002hoeffding} which states a Hoeffding-type inequality for uniformly ergodic Markov chains. We start by introducing Doeblin's condition in Assumption \ref{asp: Doeblin condition} which is the main assumption on Markov chains used by \citet{glynn2002hoeffding}. 

\begin{assumption}
    [Doeblin's Condition]
    \label{asp: Doeblin condition}
        A Markov kernel $P$ is said to satisfy Doeblin's condition if there exists a probability measure $\phi$ on $\bb{X}$, $\lambda > 0$, and a positive integer $m\in\bb{Z}_+$ such that $P^m(x, \cdot) \geq \lambda\phi(\cdot)$ for all $x\in\bb{X}$.
\end{assumption}
Assumption \ref{asp: Doeblin condition} is also known as the minorization condition and $\phi$ is called the minorization measure. It is closely related to the definition of a \textit{small set} \citep[see][Definition 9.1.1]{douc2018markov}. If $P$ satisfies Assumption \ref{asp: Doeblin condition}, then the state space $\bb{X}$ is a ($m$, $\lambda\phi$)-small set and $P$ is $\phi$-irreducible. Theorem 16.2.4 in \citet{meyn2012markov} indicates that Assumption \ref{asp: Doeblin condition} directly implies the geometric convergence of $P$ in TV metric:
\begin{align}
\label{eqn: contraction under doeblin}
    \tv (P^n(x, \cdot) - \pi) \leq 2(1-\lambda\phi(E))^{\lfloor n/m\rfloor},
\end{align}
 for all $x\in E \subset \bb X$. Thus, under Assumption \ref{asp: Doeblin condition} the Markov chain is uniformly ergodic, and the following proposition shows a Hoeffding-type inequality for Markov chains as such.

\begin{prop}[\citealp{glynn2002hoeffding}]
\label{thm: doebline Hoeffding inequality}
    Suppose that the Markov chain $\{X_n\}_{n\in\bb{N}}$ with transition kernel $P$ and invariant probability measure $\pi$ satisfies Assumption \ref{asp: Doeblin condition} and let function $f:\bb{X} \to \bb{R}$, $f\in L^\infty(\bb{X})$. Let $S_n = \sum_{i=0}^{n-1} f(X_i)$, $u = (m+1)\|f\|_\infty/\lambda$, where $m$ and $\lambda$ are described in Assumption \ref{asp: Doeblin condition}. Then, for any $\epsilon > 0$ and $\epsilon > u/n$, we have
    \begin{equation}
    \label{eqn: doebline tail bound}
        \bb{P}_x(|S_n - n\pi(f)| \geq n\epsilon) \leq 2 \exp\left\{-\frac{(n\epsilon - 2 u)^2}{2n u^2}\right\}.
    \end{equation}
\end{prop}
\begin{remark}
\label{rmk: gamma tv}
    Suppose Assumption \ref{asp: Doeblin condition} is satisfied with $m=1$ and Equation \ref{eqn: contraction under doeblin} holds. If we take $E=\bb X$ and denote $\Gamma_{\tv}$ as the convergence constant under the TV metric, then we have, for all $x\in\bb X$, 
    \begin{align*}
        \Gamma_\tv &= \sum_{n=0}^{\infty} \tv(P^n(x, \cdot) - \pi) \leq  \sum_{n=0}^{\infty} 2(1-\lambda)^n = \frac{2}{\lambda}.
    \end{align*}
    We can see that $\Gamma_\tv <\infty$ as soon as $0 < \lambda \leq 1$. Therefore, Theorem \ref{thm: main result time indep} holds with $\mbf F =\bb M_1(\bb X)$ and the resulting tail bound reads:
    \begin{align}
    \label{eqn: recovered doebline tail bound}
        \bb{P}_\mu[|S_n - n\pi(f)| > n\epsilon]\leq 2\exp \bigg\{-\frac{2\{n\epsilon - [2\|f\|_\infty/\lambda + \sp(f)]\}^2}{n[2\|f\|_\infty/\lambda + \sp(f)]^2}\bigg\},
    \end{align}
    where $M$ and $\Gamma_\mbf F$ in Equation \ref{eqn: ipm tail bound time indep} are replaced with $\|f\|_\infty$ and $\Gamma_\tv = 2/\lambda$, respectively. Comparing Equation \ref{eqn: doebline tail bound} to Equation \ref{eqn: recovered doebline tail bound}, we can observe that 
    \begin{align*}
        \sp(f) \leq 2\|f\|_\infty/\lambda,
    \end{align*}
    due to $\lambda\leq 1$. Consequently, for $\epsilon > n^{-1}(4\|f\|_\infty/\lambda)$, we have
    \begin{align}
    \label{eqn: tighter than doeblin tail bound}
        \exp\bigg\{-\frac{2\{n\epsilon - [2\|f\|_\infty/\lambda + \sp(f)]\}^2}{n[2\|f\|_\infty/\lambda + \sp(f)]^2}\bigg\} \leq \exp\bigg\{-\frac{\{n\epsilon - 4\|f\|_\infty/\lambda\}^2}{2n[2\|f\|_\infty/\lambda]^2}\bigg\},
    \end{align}
    and the right-hand side readily recovers the tail bound in Equation \ref{eqn: doebline tail bound}. Equation \ref{eqn: tighter than doeblin tail bound} also shows our tail bound is tighter than \citet{glynn2002hoeffding}, especially when $\lambda$ is close to 0, i.e., when the convergence rate is slow. This discrepancy is a result of different proof techniques used to obtain the tail bounds. We refer readers to Appendix \ref{app: proof of ipm hoeffding time dep} for details.
\end{remark}

The following result found in \citet[Corollary 23.2.4]{douc2018markov} has an alternative hypothesis that resembles ours in Theorem \ref{thm: main result dobrushin}. Indeed, it can be readily recovered by setting $\mbf F = \bb M_1(\bb X)$ in Theorem \ref{thm: main result dobrushin}.

\begin{prop}[\citealp{douc2018markov}]
\label{thm: tv d coeff Hoeffding inequality}
    Suppose the Markov chain $\{X_n\}_{n\in\bb{N}}$ with transition kernel $P$ admits an invariant probability measure $\pi$ is uniformly ergodic. Set 
    \begin{align}
        \tilde\Gamma_\tv = \sum_{n=0}^{\infty} \Delta(P^n) < \infty.
    \end{align}
    Given $f\in\bb M_1(\bb X)$ and $\sp(f) < \infty$. Let $S_n = \sum_{i=0}^{n-1} f(X_i)$. Then, for all $\mu\in\cal P(\cal X)$ and $\epsilon \geq n^{-1}\tv(\mu, \pi)(1+\Gamma_\tv)\sp(f)$,
    \begin{align}
    \label{eqn: tv d tail bound}
        \bb{P}_\mu (|S_n - n\pi(f)| > n\epsilon)\leq 2\exp\bigg\{-\frac{2n[\epsilon - n^{-1}\tv(\mu, \pi)(1+\Gamma_\tv)\sp(f)]^2}{\sp(f)^2(1+\Gamma_\tv)^2}\bigg\}
    \end{align}
\end{prop}
Our proof is largely inspired by the proof of \citet[Theorem 23.2.2]{douc2018markov}, thus the constants in the exponent of our tail bound are identical to theirs after setting $\mbf F = \bb M_1(\bb X)$. By Remark \ref{rmk: gamma tv}, the tail bound in Equation \ref{eqn: tv d tail bound} is tighter than that in Equation \ref{eqn: doebline tail bound}.
 
To conclude this section, we note that uniform ergodicity is known to be rather strict for dynamic systems. We give two simple cases where Doeblin's condition (Assumption \ref{asp: Doeblin condition}) can easily fail. Consider a deterministic stable linear system (or an AR(1) process) on $\bb X$, $P^m(x, \cdot)$ and $P^m(x', \cdot)$ are Dirac measures with disjoint supports for any $x\neq x'\in\bb X$ and $m\in\bb Z_+$, thus the minorization measure $\phi$ does not exist in this case. Suppose $\bb X$ is unbounded and the linear system is now perturbed by an additive noise supported on the entire $\bb X$ (such as a Gaussian random variable). One can readily verify that for any $x\neq x'\in\bb X$ and $m\in\bb Z_+$ the minorization measure for $P^m(x, \cdot)$ and $P^m(x', \cdot)$ converges to the trivial measure as $\msf d(x, x') \to \infty$. The exclusion of simple systems as such leaves more to be desired, which motivates us to search for a more flexible notion of convergence. 

In fact, it is rather tedious to verify the uniform ergodicity for general linear and nonlinear systems. See \citet[Chapter 7]{meyn2012markov} for a complete overview of the characterization of uniformly ergodic state-space models. It is also available as a summary in Section III-C in \citet{chen2022change}.

\subsubsection{MC converging in \texorpdfstring{$V$}{V}-Norm}
\label{sec: mc converging in v-tv}

Hoeffding's inequality for $V$-uniform ergodic Markov chains is presented in \citet{douc2011consistency} and \citet[Chapter 23.3]{douc2018markov}. It generalizes the TV metric to $V$-total variation metric (Definition \ref{def: v norm for measures}), which is an IPM with $\bb M_{V}(\bb X)$ as its generator. The resulting Hoeffding-type inequality can be applied to $V$-geometrically ergodic Markov chains. 

\begin{definition}[$V$-TV Metric]
\label{def: v norm for measures}
    For probability measures $\mu,\nu \in \cal P_V(\cal X)$ and function $V: \bb{X} \to [1, +\infty)$, $V$-total variation metric is defined as
    \begin{align}
        \ipm_V(\mu, \nu) = \sup_{f\in\bb M_V}\bigg|\int f d\mu - \int f d\nu \bigg|,\nn
    \end{align}
    where $|\cdot|$ denotes the absolute value. 
\end{definition}
From the definition of the V-TV metric, we can see that for $\mu, \nu \in \cal P_V(\cal X)$, $\tv(\mu, \nu) \leq \ipm_V(\mu, \nu)$. Thus, convergence in $\ipm_V$ constitutes a stronger assumption of the Markov chain than the TV metric. 
\begin{definition}[$V$-uniform ergodicity]
\label{def: v uniform ergodicity}
A Markov kernel $P$ is $V$-uniform ergodic if it possesses an invariant distribution $\pi$ and there exists constant $R < \infty$ and $\beta\in[0, 1)$ such that
\begin{align}
    \ipm_V(P^n(x, \cdot), \pi) \leq RV(x) \beta^{n},
\end{align}
for $\forall x\in\bb{X},\ \forall n\in\bb{N}$.
\end{definition}

\begin{prop}[\citealp{douc2011consistency}]
\label{thm: v hoeffding ineq}
    Suppose the Markov chain $\{X_n\}_{n\geq 0}$ is $V$-uniform ergodic with unique invariant probability distribution $\pi$ and initial distribution $\eta$. Let $f:\bb{X}\to \bb{R}$ be a  measurable function with $\|f\|_\infty<\infty$ and $S_n = \sum_{i=0}^{n-1} f(X_i)$. Then, for any $\epsilon > 0$ and initial distribution $\eta$, there exists a constant $K > 0$ such that 
    \begin{align}
        \bb{P}_{\eta}(|S_n - n\pi (f)|>n\epsilon) \leq K\bb{E}_{\eta}[V] \exp\left[-\frac{1}{K}\left(\frac{n\epsilon^2}{\|f\|^2_{\infty}}\right) \wedge \frac{n\epsilon}{\|f\|_\infty}\right],
    \end{align}
    where $V$ is defined in Definition \ref{def: v norm for measures}.
\end{prop}
\begin{remark}
\label{rmk: v is too strong}
    This tail bound is a hybrid between a Bernstein-type inequality and a Hoeffding-type inequality. In the regime where $n\to\infty$, the exponential term would be $\exp\left[-\frac{1}{K}\left(\frac{n\epsilon^2}{\|f\|^2_{\infty}}\right)\right]$ which behaves more like a Hoeffding-type inequality. Although in this case, the generator of $\ipm_V$ is $\bb M_V(\bb X)$, the functions applied to the Markov chain still belong to $\bb M_1(\bb X)$ which is a subset of $\bb M_V(\bb X)$ for general choices of $V$ other than $V\equiv1$. $\ipm_V$ chosen here is unnecessarily strong for the task. In fact, if one only considers the bounded functions of Markov chains, it suffices to use an IPM as strong as the TV metric.  
\end{remark}

\subsubsection{MC converging in \texorpdfstring{$L^2(\pi)$}{L2(pi)}}
\label{sec: mc converging in l2}

In this section, we review a Hoeffding-type inequality obtained with spectral methods for Markov chains. Markov transition kernels can be thought of as a bounded linear operator on a functional space, such as $L^2(\pi) = \{f\in\bb M(\bb X): \|f\|_{L^2(\pi)} < \infty\}$. One can study the convergence behavior of Markov chains by looking at the spectrum of the corresponding operator. For a comprehensive overview of spectral methods for Markov chains, we advise interested readers to look at \citet[Chapter 22]{douc2018markov}. We only introduce the necessary facts to facilitate our discussion in the sequel.

We start by introducing the definition of $L^2(\pi)$-absolute spectral gap. Denote $\{f\in L^2(\pi) : \pi(f) = 0\}$ as $L^2_0(\pi)$.
\begin{definition}[\citealp{douc2018markov}]
    Let $P$ be a Markov transition kernel with invariant probability $\pi$. The transition kernel $P$ has an \textbf{$L^2(\pi)$-absolute spectral gap} if 
    \begin{align}
        \absgap_{L^2(\pi)}(P) \coloneqq 1 - \sup\{|\lambda|: \lambda \in \spec(P| L^2_0(\pi)) \} >0,
    \end{align}
    where $\sup\{|\lambda|: \lambda \in \spec(P| L^2_0(\pi)) \} = \lim_{m\to \infty} \sup\{\|P^m g\|_{L^2(\pi)}: g\in L^2_0(\pi), \|g\|_{L^2(\pi)} \leq 1\}^{1/m}$.
\end{definition}
Under the assumption that $P$ has a positive absolute spectral gap, \citet{fan2021hoeffding} obtains a Hoeffding-type for bounded functions. We repeat their main result in Proposition \ref{thm: l2 hoeffding ineq} for completeness. 

\begin{prop}[\citealp{fan2021hoeffding}]
\label{thm: l2 hoeffding ineq}
    Let $\{X_n\}_{n\geq 0}$ be a Markov chain with invariant measure $\pi$ and absolute spectral gap $1-\lambda> 0$. For any $r \in\bb{R}$, uniformly for all bounded functions $f_i: \bb{X}\to [a_i, b_i]$, we have
    \begin{align}
        \bb{E}_\pi\bigg[e^{r[\sum_{i=1}^n f_i(X_i)-\sum_{i=1}^n\pi (f_i)]}\bigg] \leq \exp\bigg(\frac{1+\lambda}{1-\lambda}\sum_{i=1}^n\frac{(b_i-a_i)^2 r^2}{8}\bigg).
    \end{align}
    It follows that for any $\epsilon > 0$,
    \begin{align}
        \bb{P}_\pi \bigg(\sum_{i=1}^n f_i(X_i) - \sum_{i=1}^n \pi (f_i) > n\epsilon\bigg)\leq \exp\bigg(-\frac{1-\lambda}{1+\lambda}\cdot\frac{2(n\epsilon)^2}{\sum_{i=1}^n(b_i - a_i)^2}\bigg)
    \end{align}
\end{prop}

\begin{remark}
    Notice that the initial distribution is required to be $\pi$. The authors discussed it could be extended to all probability measures with $\ltpi$ Radon-Nikodym derivative against $\pi$. In contrast, our framework and results mentioned earlier do not have this restriction on the initial distribution.
\end{remark}

To draw the connection with our framework, we need to further understand the implications of the absolute spectral gap. By Proposition 22.2.4 in \citet{douc2018markov}, having an $L^2(\pi)$-absolute spectral gap implies being $L^2 (\pi)$ geometrically ergodic which is defined formally as follows. Denote the stationary transition kernel as $\Lambda_\pi(x, \cdot) \coloneqq \pi(\cdot)$ and the unit ball in $L^2(\pi)$ as $\mbf F_{L^2} \coloneqq \{f\in L^2(\pi): \|f\|_{L^2(\pi)} \leq 1\}$.  With an abuse of notation, we write the $L^2(\pi)$-norm for a probability measure as $\|\mu\|_{L^2(\pi)} \coloneqq \|d\mu / d\pi\|_{L^2(\pi)}$ whenever $\mu$ in dominated by $\pi$ and $\|\mu\|_{L^2(\pi)} = \infty$  otherwise. Denote the space of probability measures with finite $\ltpi$-norm as $\cal P_\ltpi(\cal X)\coloneqq \{\nu \in \cal P (\cal X) : \|\nu\|_\ltpi < \infty\}$.
\begin{definition}[\citealp{douc2018markov}]
    $P$ is said to be $L^2(\pi)$-\textbf{geometrically ergodic} if there exist $\beta\in[0, 1)$, a constant $\tilde \alpha(\mu) < \infty$, and for all $\mu \in\cal P_\ltpi(\cal X)$ satisfying
    \begin{align}
        \|\mu P^n - \pi\|_{L^2(\pi)} \leq \tilde\alpha(\mu)\beta^n,
    \end{align}
    for $n\in\bb N$.
\end{definition}

Let us consider the $L^2(\pi)$-IPM denoted as $\ipm_{L^2(\pi)}$ with generator $\mbf F_{L^2}$. A geometric convergence in $\ipm_{L^2(\pi)}$ can be analogously defined as follows.
\begin{definition}
$P$ is said to be \textbf{geometrically convergent in $\ipm_{L^2(\pi)}$} if there exist $\beta \in [0, 1)$, a constant $\alpha(\mu) < \infty$, and for all $\mu \in\cal P_\ltpi (\cal X)$ satisfying
    \begin{align}
    \label{eqn: l2 ipm geometric convergence}
        \ipm_{L^2(\pi)} (\mu P^n, \pi) = \sup_{f\in\mbf F_\ltpi} |\mu P^n f - \pi(f) | \leq \alpha(\mu)\beta^n,
\end{align}
for $n\in\bb N$.
\end{definition}
The following lemma establishes the equivalence between a geometric convergence in $\ipm_{L^2(\pi)}$ and $L^2(\pi)$-geometric ergodicity.
\begin{lemma}
    Suppose $P$ has invariant probability $\pi$. The following statements are equivalent:
    \begin{enumerate}
        \item $P$ is $L^2(\pi)$-geometrically ergodic;
        \item $P$ is geometrically convergent in $\ipm_{L^2(\pi)}$.
    \end{enumerate}
\end{lemma}
\begin{proof}
    Let $\mu\in \cal P_\ltpi(\cal X)$ be the initial distribution. Let $g_\mu = d\mu / d\pi$ be the Radon–Nikodym derivative and $g_\mu \in L^2(\pi)$. We write the inner product between $\mu$ and a real-valued function $f$ in $L^2(\pi)$ as $\langle \mu, f\rangle = \int f d\mu = \int fg_\mu d\pi = \langle g_\mu, f \rangle_{\ltpi}$.
    
    It is easy to show that $\|\mu P\|_{L^2(\pi)} \leq \|\mu\|_{L^2(\pi)}$ \citep[Lemma 22.1.3]{douc2018markov}. Thus, it follows $\mu P^n$ is dominated by $\pi$ and  $\|\mu P^n\|_\ltpi\leq \|\mu\|_\ltpi$ for $n\in\bb N$. We can write $\ipm_{L^2(\pi)}$ using the inner product form,
    \begin{align}
    \label{eqn: l2 ipm equals l2 norm}
        \ipm_\ltpi(\mu P^n, \pi) &= \sup_{f\in\mbf F_\ltpi} |\langle \mu P^n, f\rangle - \langle \pi, f \rangle| =  \sup_{f\in\mbf F_\ltpi} |\langle \mu P^n - \pi, f\rangle | \nn\\
        &= \| \mu P^n - \pi\|_\ltpi,
    \end{align}
    where the last equality is due to Cauchy–Schwarz inequality and the equality can be obtained with an appropriate choice of $f$ in $\mbf F_\ltpi$.
\end{proof}

It is obvious that the geometric convergence in $\ipm_{L^2(\pi)}$ implies the following
\begin{align}
\label{eqn: nonuniform l2 gamma}
    \Gamma_\ltpi(\mu) \coloneqq \sum_{i=1}^{\infty} \ipm_\ltpi(\mu P^i, \pi) < \infty,
\end{align}
for $\mu\in\cal P_\ltpi(\cal X)$. Equation \ref{eqn: nonuniform l2 gamma} depends on the initial distribution $\mu$, thus being less demanding than Equation \ref{eqn: gen con} used in Theorem \ref{thm: main result time dep} which requires a uniform property across the state space $\bb X$. Naturally, one would be skeptical if geometric convergence in $\ipm_{L^2(\pi)}$ leads to the same conclusion as in Theorem \ref{thm: main result time dep} where the initial distribution is arbitrary. In the rest of this section, however, we demonstrate a similar consequence (with restricted initial distribution) to that of Proposition \ref{thm: l2 hoeffding ineq} using Equation \ref{eqn: nonuniform l2 gamma}. From previous discussions, we can deduce $\Gamma_\ltpi(\mu) < \infty$ is not stronger than $\absgap_\ltpi(P) > 0$, which implies our framework generalizes that of \citet{fan2021hoeffding}. 

\begin{figure}[t]
    \centering
    \begin{tikzcd}[arrows=Rightarrow]
        \text{$\absgap_\ltpi(P) > 0$} \arrow[r]\arrow[d]
        &\text{$L^2(\pi)$-geometric ergodicity}\arrow[d, Leftrightarrow] &\ \\
        \exists m>1,\Delta_\ltpi (P^m) < 1 \arrow[ru] &\text{$\ipm_{L^2(\pi)}$ geometric convergence} \arrow[r]
         &\text{$\Gamma_\ltpi(\mu) < \infty$}
    \end{tikzcd}
    \caption{Relation between the convergence concepts in Section \ref{sec: mc converging in l2}.}
    \label{fig: Implication diagram for l2}
\end{figure}
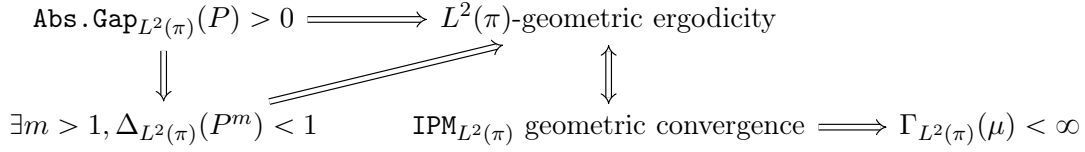

It is also possible to define the IPM Dobrushin coefficient with $\ipm_\ltpi$ which we call the \textbf{$\ltpi$-Dobrushin coefficient} and is denoted as 
\begin{align}
    \Delta_\ltpi (P) \coloneqq \sup \bigg\{\ipm_\ltpi(\xi P, \xi' P) : \xi\neq\xi' \in \cal P_\ltpi(\cal X), \ipm_\ltpi(\xi, \xi') \leq 1\bigg\}.
\end{align}
We draw the connection between $\Delta_\ltpi (P)$ and the operator norm of $P$ on $L^2_0(\pi)$ in the following lemma. From there, one can readily establish the connection with other convergence-related concepts introduced in this section. To provide a holistic view, we summarize their relations in Figure \ref{fig: Implication diagram for l2}.
\begin{lemma}
    Let $\vvvert P \vvvert_{L^2_0(\pi)} \coloneqq \sup_{f\in L^2_0(\pi)} \frac{\|Pf\|_\ltpi}{\|f\|_\ltpi}$ denote the operator norm on $L^2_0(\pi)$. It holds that $\Delta_\ltpi(P) = \vvvert P \vvvert_{L^2_0(\pi)}$.
\end{lemma}
\begin{proof}
    Using Equation \ref{eqn: l2 ipm equals l2 norm}, we can deduce that 
    \begin{align*}
       \Delta_\ltpi (P) = \sup \bigg\{\|(\xi -\xi') P\|_\ltpi : \xi\neq\xi' \in \cal P_\ltpi(\cal X), \|\xi - \xi'\|_\ltpi \leq 1\bigg\}.
    \end{align*}
    Since $\xi$ and $\xi'$ are arbitrary positive measures in $\cal P_\ltpi(\cal X)$, $\xi - \xi'$ can be any signed measure in $\cal P_\ltpi(\cal X)$ and $\int \frac{d(\xi-\xi')}{d\pi} d\pi = 0$. Let $P^*$ be the adjoint operator of $P$, i.e., $\langle Pf, g\rangle_\ltpi = \langle f, P^*g\rangle_\ltpi$ for $f, g\in \ltpi$. By Proposition 22.1.5 from \citet{douc2018markov}, we have the $\xi P = P^*(d\xi / d\pi)$. Thus, we can deduce that
    \begin{align*}
         &\sup \bigg\{\|(\xi -\xi') P\|_\ltpi : \xi\neq\xi' \in \cal P_\ltpi(\cal X), \|\xi - \xi'\|_\ltpi \leq 1\bigg\}\\
         =& \sup \bigg\{\bigg\|P^*\frac{d\xi -d\xi'}{d\pi}\bigg\|_\ltpi : \xi\neq\xi' \in \cal P_\ltpi(\cal X), \bigg\|\frac{d(\xi -\xi')}{d\pi}\bigg\|_\ltpi \leq 1\bigg\}\\
         =& \vvvert P^*\vvvert_{L^2_0(\pi)} = \vvvert P\vvvert_{L^2_0(\pi)},
    \end{align*}
    where the last equality is a result of combining Theorem 22.1.8 and Lemma 22.2.1 from \citet{douc2018markov}.
\end{proof}
\begin{remark}
    Combining \citet[Proposition 22.2.4]{douc2018markov} and the above lemma, it follows that $\absgap_\ltpi(P) > 0$ implies there exists $m>1$ such that $\Delta_\ltpi^2(P^m) < 1$ which further implies $\ltpi$-geometric ergodicity. 
\end{remark}

Similar to Equation \ref{eqn: dobrushin con}, we can write  $\tilde\Gamma_\ltpi \coloneqq \sum_{i=0}^\infty \Delta_\ltpi(P^i)$. It is worth noticing that $\absgap_\ltpi(P) > 0$ also leads to $\tilde\Gamma_\ltpi < \infty$ after realizing $\vvvert P^n\vvvert_\ltzpi \leq \|P^m\|_\ltzpi^{\lfloor n/m\rfloor}$, for $n\geq m$. By conditioning on $\tilde\Gamma_\ltpi < \infty$, it is possible to obtain a Hoeffding-type inequality following the steps in the proof of Theorem \ref{thm: main result dobrushin}. However, to avoid repetitiveness, we conclude this section with the variant that uses Equation \ref{eqn: nonuniform l2 gamma} in the following corollary. The proof is omitted as it is a special case of Theorem \ref{thm: main result dobrushin}.  

\begin{corollary}
    Let $\{X_n\}_{n\in\bb{N}} \subset\bb{X}$ be a Markov chain with transition kernel $P: \bb{X}\times\cal X\to[0, 1]$. For initial distribution $\mu \in \cal P_\ltpi(\cal X)$, assume that the Markov kernel $P$ admits a unique invariant probability measure $\pi$ such that Equation \ref{eqn: nonuniform l2 gamma} holds. For $i\in\{0, \cdots, n-1\}$ and $n\in\bb{Z}_+$, given $f_i$ with bounded span and $f_i\in M_i \cdot \mbf F$ with minimal stretch $M_i=\|f_i\|_\ltpi$. Let $\tilde{S}_{n} \coloneqq \sum_{i=0}^{n-1} f_i(X_i)$. Then, for $\epsilon >  n^{-1}[M_1^{\max}\Gamma_\ltpi(\mu) + \sp(f_0)]$, we have
    \begin{align}
    \label{eqn: Hoeffding ineq l2 ipm}
        \bb{P}_\mu\bigg[\bigg|\Tilde{S}_n - \sum_{i=0}^{n-1} \pi(f_i)\bigg| > n\epsilon\bigg]\leq 2\exp\bigg(-\frac{2\{n\epsilon - [M_1^{\max}\Gamma_\ltpi(\mu) + \sp(f_0)]\}^2}{\sum_{i=0}^{n-1} [M_{i+1}^{\max}\Gamma_\ltpi(\mu) + \sp(f_i)]^2}\bigg),
    \end{align}
    where $M_i^{\max} = \max\{M_i, \cdots, M_{n-1}\}$ for $i\in\{0, \cdots, n-1\}$ and $M_i^{\max} = 0$ for $i\geq n$.
\end{corollary}

\subsection{Non-Ergodic MC}

In this section, we discuss the Markov chains that are not covered by the aforementioned ergodicity conditions, namely, non-irreducible chains and periodic chains. 

\subsubsection{MC converging in Wasserstein metric}
\label{sec: non-irreducible chains}
Hoeffding's inequality for Markov chains converging in the Wasserstein metric is studied in \citet{sandric2021hoeffding} which we review in this section. Their hypothesis includes certain non-irreducible chains. 
From the definition of the small set and Doeblin's condition, we can see that they enforce the probability measure $P(x, \cdot)$ to be non-singular for $x$ in at least some subset of the state space $\bb{X}$. Since small sets and Doeblin's condition play a central role in the ergodicity of Markov chains, one can make the educated guess that the ergodicity may fail when $P(x, \cdot)$ is singular. In fact, the real problem here is caused by the TV metric. For example, if we consider two Dirac measures $\delta_0$ and $\delta_{1/n}$ for $n\in\bb{Z}_+$, it is easy to see that $\delta_{1/n}$ does not converge to $\delta_0$ in TV metric as $\tv(\delta_{1/n}, \delta_0) = 1$ for all $n\in\bb{Z}_+$. Nevertheless, $\delta_{1/n} \to \delta_0$ in some weaker sense which is to be specified later. This ``incompatibility" with singular measures raises problems when one tries to apply Markov chain theory to deterministic systems or systems with discrete noise. The solution to the issue is to consider weaker metrics that induce coarser topologies in the space of probability measures. One of the notable alternatives researchers have considered is the \textit{Wasserstein metric}. For Polish space $(\bb X, \cal X, \msf d)$, define the Wasserstein-($p,\msf d$) metric on $\cal P(\cal X)$ as
\begin{align}
    \wass_{p, \msf d}(\mu, \nu) \coloneqq \bigg[\inf_{\Pi\in\ALP C(\mu, \nu)} \int_{\bb{X}\times\bb{X}}
    \mathsf d(x, y)^p \Pi(dx, dy)\bigg]^{1/p},
\end{align}
where $\ALP C(\mu, \nu)$ denotes the family of couplings of $\mu$ and $\nu$. For Wasserstein-1 metric in particular ($p = 1$), it holds, by Kantorovich-Rubinstein duality theorem \cite{dudley2018real}, that
\begin{align}
    \wass_{1, \msf d}(\mu, \nu) = \sup_{\{f:\lip(f) \leq 1\}} \bigg|\int_{\bb{X}} f d\mu - \int_{\bb{X}} f d\nu\bigg|,
\end{align}
whenever the integrals are well-defined. Note that the presence of metric $\msf d$ is contained in the Lipschitz condition of the witness function $f$, i.e., $|f(x)-f(y)| \leq \lip(f)\msf d(x, y)$, for $x,y\in\bb X$. Sufficient conditions for geometric and subgeometric convergence for Markov chains in Wasserstein-1 metric have been established in \citet{hairer2011asymptotic} and \citet{butkovsky2014subgeometric}, respectively. We borrow the following example from \citet{butkovsky2014subgeometric} to illustrate the necessity of using the Wasserstein metric. 

\begin{example}[\citealp{butkovsky2014subgeometric}]
\label{exa: non-irreducible mc}
    Let $\boldsymbol X = \{X_n\}_{n\geq 0}$ be an autoregressive process satisfying the following equation:
    \begin{align}
        X_{n+1} = \frac{1}{10} X_{n} + \epsilon_{n+1},
    \end{align}
    where $\epsilon_1, \epsilon_2, \cdots$ are i.i.d. random variables uniformly distributed on the set $\{0, \frac{1}{10}, \frac{2}{10}, \cdots, \frac{9}{10}\}$ and independent of $X_0$. $X_0$ is uniformly distributed on $[0, 1)$. Then, $\boldsymbol{X}$ is a Markov chain with state space $(\bb{X}, \cal X) = ([0, 1),\ALP B([0, 1)))$ with a unique invariant measure $\pi$ which is the uniform distribution on $[0, 1)$. Let $d(x, y) = |x-y|$ be the metric on $\bb{X}$. One can show \citep{butkovsky2014subgeometric} the following facts:
    \begin{enumerate}
        \item $\boldsymbol{X}$ converge to $\pi$ in Wasserstien-(1, $\msf d$) metric and $\wass_{1, \msf d}(P^n(x, \cdot), \pi(\cdot)) \leq 10^{-n}$.
        \item $\tv(P^n(x, \cdot), \pi(\cdot)) = 1$ for $\forall n\in\bb{N}$.
        \item $\boldsymbol{X}$ is not $\psi$-irreducible. Consider $\bb{Q}\cap[0, 1)$ and $\bb{Q}^c\cap[0, 1)$ which are disjoint and absorbing. 
    \end{enumerate}
    Since $\boldsymbol{X}$ is not ergodic in the usual sense, Proposition \ref{thm: doebline Hoeffding inequality}, \ref{thm: tv d coeff Hoeffding inequality}, \ref{thm: v hoeffding ineq}, and \ref{thm: l2 hoeffding ineq} cannot be used. 
\end{example}
Along the same line, one can deduce that systems with stable dynamics and additive discrete noise are excluded by the usual ergodicity condition as well. To bridge this gap, \citet{sandric2021hoeffding} has made a step in this direction by giving a version of Hoeffding's inequality for non-irreducible Markov chains. We copy their theorem below for completeness.
\begin{prop}[\citealp{sandric2021hoeffding}]
\label{thm: wasserstein hoeffding ineq}
    Let $f: \bb{X} \to \bb{R}$ be bounded and Lipschitz continuous, and let $S_n \coloneqq \sum_{i=0}^{n-1} f(X_i)$ for $n\in\bb{Z}_+$. Assume that $\{X_n\}_{n\geq0}$ admits an invariant probability measure $\pi(dx)$ such that
    \begin{align}
    \label{eqn: gamma for wasserstein}
        \Gamma_{\wass}\coloneqq \sup_{x\in\bb{X}}\sum_{n=1}^\infty \wass(P^n(x, \cdot), \pi(\cdot)) < \infty
    \end{align}
    Then for any $\epsilon > 2n^{-1}\lip(f)\Gamma_{\wass}$,
    \begin{align}
    \label{eqn: tail bound wasserstein}
        \bb{P}_x[|S_{n} - n\pi (f)| > n\epsilon] \leq 2\exp\bigg\{\frac{-(n\epsilon - 2\lip(f)\Gamma_{\wass})^2}{8n(\lip(f)\Gamma_{\wass} + \|f\|_\infty)^2}\bigg\}
    \end{align}
\end{prop}

\begin{remark}
    The term $\lip(f)\Gamma_{\wass} + \|f\|_\infty$ in the denominator is not squared in the original statement of \citet[Theorem 1]{sandric2021hoeffding}. This term should be squared when \citet[Lemma 8.1]{devroye2013probabilistic} is applied but it is missing in their original proof.
\end{remark}

With our framework, we present the following corollary which achieves the same concentration result as in Proposition \ref{thm: wasserstein hoeffding ineq} but under a slightly more general condition. We notice that Proposition \ref{thm: wasserstein hoeffding ineq} considers bounded Lipschitz functions applied to the Markov chain, which indicates that the convergence in the Wasserstein-1 metric is stronger than needed. As pointed out in Remark \ref{rmk: main insight} and \ref{rmk: v is too strong}, the generator of the IPM should be large enough to cover the function of interests and anything more is unnecessary. In this case, the appropriate IPM generator is bounded Lipschitz functions which is smaller than that of Wasserstein-1 metric, resulting in a weaker convergence requirement. Let $\|\cdot\|_{\bl}$ denote the bounded Lipschitz norm and is given by
\begin{align*}
    \|f\|_{\bl} = \max \bigg\{ \|f\|_{\infty}, \lip(f)\bigg\}.
\end{align*}
Consider the generalized concentrability condition defined with the generator $\mbf F_\bl = \{f\in\bb M(\bb X): \|f\|_{\bl}\leq 1\}$. Let the IPM generated by $\mbf F_\bl$ be denoted as $\bl$. We define $\Gamma_\bl$ using $\bl$ in a similar fashion as Equation \ref{eqn: gen con}. The following corollary is given under the condition $\Gamma_\bl < \infty$. 


\begin{corollary}
    \label{thm: Hoeffding ineq wass bdd space}
    Let $\boldsymbol X = \{X_n\}_{n\in\bb{N}} \subset\bb{X}$ be a Markov chain with transition kernel $P: \bb{X}\times\cal X$. Assume $P$ admits a unique invariant probability measure $\pi$ such that
    \begin{align}
    \label{eqn: hoeff ineq wass cond bdd}
        \Gamma_{\bl} \coloneqq \sup_{x\in\bb X}\sum_{i=0}^\infty \bl(P^i(x, \cdot), \pi) < \infty.
    \end{align} 
    Suppose $\boldsymbol X$ initializes with $\mu\in\cal P(\cal X)$. Consider a bounded Lipschitz continuous function $f:\bb X\to \bb R$ and $S_n = \sum_{i=0}^{n-1}f(X_i)$ for $n\in\bb Z_+$. Then, for $\epsilon > (2n)^{-1}\bl(\mu, \pi) [2\|f\|_{\bl}\Gamma_\bl+\sp(f)]$, we have
    \begin{align}
    \label{eqn: Hoeffding ineq wass bdd}
        \bb{P}_\mu\bigg[\bigg|S_n - n\pi(f)\bigg| > n\epsilon\bigg]\leq 2\exp\bigg(-\frac{2[n\epsilon - \bl(\mu, \pi)[2\|f\|_{\bl}\Gamma_{\bl}+\sp(f)]/2]^2}{n[2\|f\|_{\bl}\Gamma_{\bl}+\sp(f)]^2}\bigg).
    \end{align}
\end{corollary}

\begin{proof}
    This is a restatement of the main result using $\mbf F_\bl$.
\end{proof}

\begin{remark}
    For any bounded Lipschitz function $f$, let $M$ be the minimal stretch of $\mbf F_\bl$ that includes $f$. Then, $M = \|f\|_{\bl}$ and $f\in \|f\|_{\bl} \cdot\Gamma_{\bl}$.
\end{remark}

We apply Corollary \ref{thm: Hoeffding ineq wass bdd space} to obtain the sample complexity of Ployak-Ruppert averaging of SGD in Section \ref{sec: average of rsa}. The function of interest there is the identity function on a compact space. Thus, it can be treated as a bounded Lipschitz function.

\subsubsection{Periodic Chains}
\label{sec: periodic chains}

A version of Hoefdding's inequality for periodic Markov chains is given in \citet{liu2021hoeffding}. The work is motivated by the observation that Markov chains satisfying Doeblin's condition in Assumption \ref{asp: Doeblin condition} are not necessarily aperiodic, which indicates the uniform ergodicity condition used in \citet{glynn2002hoeffding} can be relaxed. Since the key step in \citet{glynn2002hoeffding} is to upper-bound the solution of the Poisson equation which is obtained via Doeblin's condition. The authors in \citet{liu2021hoeffding} refine this argument by directly bounding the solution of the Poisson equation in terms of the drift condition and the ergodicity coefficient. We skip their results as it is similar in form as \citet{glynn2002hoeffding}. 

\section{Applications}
\label{sec: application}

In this section, we provide three application examples in machine learning to demonstrate the utility of our theory.

\subsection{Generalization Error Bound for Non-IID Samples}
\label{sec: generalization error bound}
The assumption of i.i.d. samples is prevalent in machine learning literature. In practice, the samples may not be independent. To relax the i.i.d. assumption, samples drawn from mixing processes have been considered from the perspective of empirical process theory \citep{yu1994rates}, time series prediction \citep{alquier2012model}, online learning \citep{agarwal2012generalization}, empirical risk minimization (ERM) \citep{zou2009generalization}, etc. Stationary processes are considered in the aforementioned papers except for \citet{agarwal2012generalization} which considers asymptotically mixing processes, i.e. converging to stationary distribution in a suitable sense. 

In this section, we study the generalization error for online learning algorithms with samples drawn from a Markov chain converging in general IPM. Consider the following stochastic optimization problem
\begin{align}
\label{eqn: stochastic optimization problem}
    \minimize_{\theta\in\Theta}F(\theta) \coloneqq \bb{E} [f(u; \theta)],
\end{align}
where $f:\ALP U\times \Theta \to \bb{R}$ is the loss function depending on the random sample $u\in\ALP U$ and parameter $\theta\in\Theta$, and the expectation is taken with respect to the distribution of $u$. Before moving forward, we state the following assumptions.

\begin{assumption}
\label{asp: stochastic optimization problem}
\hfill

\begin{enumerate}
    \item $f$ is uniformly bounded on $\ALP U\times \Theta$.
    \item $f$ is strongly convex in $\theta$.
\end{enumerate}
\end{assumption}
In practice, the objective in Equation \ref{eqn: stochastic optimization problem} is approximated with the empirical loss estimated from data. Suppose we have $N$ samples $\{u_i\}_{i=1}^N \subset \ALP U$. The empirical loss in Equation \ref{eqn: stochastic optimization problem} is written as 
\begin{align}
\label{eqn: approximated stochastic optimization problem}
    \minimize_{\theta\in\Theta} \hat{F}(\theta) \coloneqq \frac{1}{N}\sum_{i=1}^N f(u_i; \theta).
\end{align}
Under appropriate assumption on $f$, both the original problem in Equation \ref{eqn: stochastic optimization problem} and the approximated problem in Equation \ref{eqn: approximated stochastic optimization problem} have unique solutions, which are denoted as $\theta^*$ and $\hat{\theta}$ respectively. In general, the optimal solution $\theta^*$ and approximated solution $\hat{\theta}$ differ, and one would like to bound the performance degradation caused by the approximated solution. That is, finding an upper bound for $|F(\hat{\theta}) - F(\theta^*)|$ which is referred to as the generalization error bound. When the samples are i.i.d., the problem can be solved readily using Hoeffding's inequality for independent random variables \citep{hoeffding1994probability}. Similarly, if we assume the samples $\{u_i\}_{i=1}^N$ follow a Markov chain converging in some sense to be specified, then the generalization error bound can be obtained by applying our version of Hoeffding's inequality for Markov chains.

\begin{theorem}
Suppose Assumption \ref{asp: stochastic optimization problem} holds and $f\in M\cdot\mbf F$ for $M\in(0, \infty)$ and $\mbf F\subset \bb M_1$. Let $\{u_i\}_{i=1}^\infty$ be a Markov chain converging in $\ipm_\mbf F$ with constant $\Gamma_\mbf F$. For $\delta\in(0, 1)$, the following holds with probability at least $1-\delta$
\begin{align}
    F(\hat{\theta}) \leq F(\theta^*) + \frac{4M\Gamma_{\mbf F} + 2\sp(f)}{N}\bigg(1+ \sqrt{\frac{N}{2}\log(\frac{1}{\delta})}\bigg).
\end{align}
\end{theorem}

\begin{proof}
Since $\mbf F \subset \bb M_1$, we can apply Theorem \ref{thm: main result dobrushin} to $F(\hat{\theta}) - \hat{F}(\hat{\theta})$,
    \begin{align}
        \bb{P}\{F(\hat{\theta}) - \hat{F}(\hat{\theta}) < \epsilon\} =& \bb{P}\bigg\{ \bb{E}_\pi f(u; \hat{\theta}) - \frac{1}{N}\sum_{i=1}^N f(u_i;\hat{\theta}) < \epsilon \bigg\} \\
        \geq& 1 - \exp \bigg\{-\frac{2\{N\epsilon - [2M\Gamma_{\mbf F} + \sp(f)]\}^2}{N[2M\Gamma_{\mbf F} + \sp(f)]^2}\bigg\}.
    \end{align}
    With probability at least $1-\delta$, we have
    \begin{align}
        F(\hat{\theta}) &\leq \hat{F}(\hat{\theta}) + \frac{2M\Gamma_{\mbf F} + \sp(f)}{N}\bigg(1+ \sqrt{\frac{N}{2}\log(\frac{1}{\delta})}\bigg)\nn\\
        &\leq \hat{F}(\theta^*) + \frac{2M\Gamma_{\mbf F} + \sp(f)}{N}\bigg(1+ \sqrt{\frac{N}{2}\log(\frac{1}{\delta})}\bigg), \label{eqn: generalization bound part 1}
    \end{align}
    where the last inequality is due to $\theta^*$ being the minimizer of $\hat F$. Applying Theorem \ref{thm: main result dobrushin} again to $\hat{F}(\theta^*) - F(\theta^*)$ yields
    \begin{align}
        \bb{P}\{\hat{F}(\theta^*) - F(\theta^*) <\epsilon \} =& \bb{P}\bigg\{\frac{1}{N}\sum_{i=1}^N f(u_i;\theta^*) - \bb{E}_\pi f(u; \theta^*) < \epsilon \bigg\} \\
        \geq& 1 - \exp \bigg\{-\frac{2\{N\epsilon - [2M\Gamma_{\mbf F} + \sp(f)]\}^2}{N[2M\Gamma_{\mbf F} + \sp(f)]^2}\bigg\}. 
    \end{align}
    With probability at least $1-\delta$, we have
    \begin{align}
        \hat{F}(\theta^*)\leq F(\theta^*) + \frac{2M\Gamma_{\mbf F} + \sp(f)}{N}\bigg(1+ \sqrt{\frac{N}{2}\log(\frac{1}{\delta})}\bigg).\label{eqn: generalization bound part 2}
    \end{align}
    Combine Equation \ref{eqn: generalization bound part 1} and Equation \ref{eqn: generalization bound part 2}, we have, with probability at least $1-\delta$, 
    \begin{align}
    \label{eqn: generalization bound}
         F(\hat{\theta}) \leq  F(\theta^*) + \frac{4M\Gamma_{\mbf F} + 2\sp(f)}{N}\bigg(1+ \sqrt{\frac{N}{2}\log(\frac{1}{\delta})}\bigg).
    \end{align}
\end{proof}

\subsection{Ployak-Ruppert averaging for SGD}
\label{sec: average of rsa}
Stochastic recursive algorithms (RSA), including SGD, empirical value iteration (EVI), etc., are widely used in optimization, stochastic approximation, and machine learning. In \citet{gupta2020some}, the authors model the general form of RSAs with constant step size as a Markov chain and study its convergence properties. In particular, the time average of the RSA iterates is shown to have the variance reduction property given the Markov chain induced by the RSA converges weakly to its invariant distribution \citep{gupta2020some}. This result aligns well with the common acceleration technique used in SGD and stochastic approximation (SA), namely Ployak-Ruppert averaging \citep[see][]{polyak1992acceleration, ruppert1988efficient, nemirovski2009robust}. In the following example, we complement the existing study by demonstrating a \textit{finite sample} analysis of the Ployak-Ruppert averaging of the SGD algorithm. For \textit{finite time} convergence guarantees are derived for SGD-based algorithms with constant step size, please refer to \citet{dieuleveut2017harder, bach2013non, bach2014adaptivity}.

We first introduce the necessary theoretic background on stochastic contractive operators, then dive into a case study of the empirical risk minimization (ERM) via SGD. Let us start with the contractive property for deterministic operators on complete separable metric spaces. Let the Polish space $\bb{X}$ be equipped with metric $\msf d: \bb{X}\times\bb{X}\to [0, \infty)$. Let $T: \bb{X}\to\bb{X}$ be a contraction operator with coefficient $\alpha\in(0,1)$ and the unique fixed point $x^*\in\bb{X}$, i.e., 
\begin{align}
    \msf d(T(x_1), T(x_2)) \leq \alpha \msf d(x_1, x_2),
\end{align}
for all $x_1, x_2\in\bb{X}$, and $T(x^*) = x^*$. Given any starting point $x_0\in\bb{X}$, recursively applying the operator $T$ generates a sequence $(x_k)_{k\in\bb{N}}$ where $x_{k+1} = T(x_k)$ for $k\in\bb{N}$. By Banach contraction mapping theorem \citet[Theorem 1.1]{granas2003fixed}, $\msf d(x_k, x^*) \leq \alpha^k \msf d(x_0, x^*)$, thus $x_k \to x^*$ geometrically fast as $k\to\infty$.

For stochastic operators, we need a few more notations. Let $(\Omega, \ALP F, \bb{P})$ be a probability space where $\Omega$ is the set of uncertainties, $\ALP F$ is the Borel $\sigma$-algebra over $\Omega$, and $\bb{P}$ is the base probability measure. Let $\hat {T}^n_k: \bb{X}\times\Omega\to \bb{X}$ be a random operator used at $k$-th iteration and indexed by $n\in\bb{N}$. This random operator is used to model the approximation of the true operator $T$ carried out by RSAs. In the case of SGD, $n$ is the number of samples (batch size) used to estimate the gradient. In EVI, $n$ is the number of samples used to estimate the expected cost-to-go. For simplicity of the presentation, we suppress the dependence of $\hat{T}^n_k$ on $\omega \in\Omega$ and we write $\hat{T}^n_k (x) \coloneqq \hat{T}^n_k(x, \omega)$. For each $x\in\bb{X}$, $\hat{T}^n_k(x)$ is a random variable, and we make the following independence assumption for random operators at a different time step.
\begin{assumption}[\citealp{gupta2020some}]
\label{asp: operator independence}
    For each $x\in\bb{X}$, $\hat{T}^n_k(x)$ and $\hat{T}^n_{k'}(x)$ are independent and identically distributed (i.i.d.) for $k\neq k'$.
\end{assumption}

\begin{assumption}[\citealp{diaconis1999iterated}]
\label{asp: assumption 5.3}
The following conditions are satisfied:
\begin{enumerate}
    \item $\bb{X}$ is a Polish space.
    \item There exists $a, b > 0$ such that the operator $\hat{T}^n_k$ satisfy
    \begin{align}
    \label{eqn: stochastic operator close to true in prob}
        \bb{P}\bigg\{\msf d\bigg(\hat{T}^n_k(z^*), z^*\bigg) > \epsilon \bigg\} \leq \frac{a}{\epsilon^b},
    \end{align}
    for all $\epsilon >0$ and $k\in\bb{N}$.
    \item Let $\hat{\alpha}^n_k$ denote the Lipschitz coefficient of $\hat{T}_k^n$. Then, for all $k\in\bb{N}$, 
    \begin{align}
        \bb{E}[\hat{\alpha}^n_k] < \infty, \quad \bb{E}[\log(\hat{\alpha}^n_k)] < 0.
    \end{align}
\end{enumerate}
\end{assumption}

\begin{lemma}[\citealp{diaconis1999iterated}]
\label{thm: theorem 5.5}
Suppose that Assumption \ref{asp: operator independence} and \ref{asp: assumption 5.3} hold, then there exists a unique invariant probability measure $\pi_n$ such that the distribution of $\hat{z}^n_k$ converges to $\pi_n$ in weak* topology. 
\end{lemma}

Now, we look at a concrete machine learning example. Consider an empirical risk minimization problem with the loss function $l:\ALP U\times\Theta\to\bb{R}$, where $\ALP U \subset\bb{R}^{d_s}$ is the sample space , $\Theta \subseteq \bb{R}^{d_p}$ is the parameter space. Suppose $\Theta$ is convex, equipped with the Euclidean metric $\|\cdot\|_2$, and has a finite diameter. Let $\ALP D = \{u_i\}_{i=1}^N \subset \ALP U$ be the training data set with i.i.d. samples. Let $l(u_i, \theta)$ be the loss associated with sample $u_i$ when the parameter is $\theta$. We formulate the problem as follows:
\begin{align}
\label{eqn: empirical risk minization}
    \minimize_{\theta\in\Theta} F_{\ALP D}(\theta) \coloneqq \frac{1}{N}\sum_{i=1}^N l(u_i, \theta).
\end{align} 
We make the following assumptions before moving forward.
\begin{assumption}
\label{asp: strong convex and smooth}
Suppose the following conditions hold:
\begin{enumerate}
    \item $l(u, \cdot)$ is twice-differentiable almost everywhere, for all $u\in\ALP U$.
    \item $\nabla_\theta l(u, \cdot)$ is Lipschitz continuous with constant $L > 0$ for all $u \in\ALP U$, i.e., 
    \begin{align}
        \|\nabla_\theta l(u, \theta) - \nabla_\theta l(u, \theta')\|_2 \leq L \|\theta - \theta'\|_2, \quad \text{for }\forall \theta,\theta'\in\Theta. \nn
    \end{align}
    \item $l(u, \cdot)$ is strongly convex with modulus $\eta \in (0, L)$ for all $u\in\cal U$, i.e.,
    \begin{align}
        \langle \nabla_\theta l(u, \theta) - \nabla_\theta l(u, \theta'), \theta - \theta' \rangle \geq \eta \|\theta - \theta'\|_2^2, \quad \text{for }\forall \theta, \theta'\in\Theta.\nn
    \end{align}
    \item $\nabla_\theta l(u, \theta)$ is uniformly bounded function of $u\in\ALP U$ for all $\theta \in \Theta$.
    \item $\Theta \subset \bb R^{d_p}$ is compact.
\end{enumerate}
\end{assumption}

As an immediate result, the problem in Equation \ref{eqn: empirical risk minization} has a unique minimizer $\theta^* \in\Theta$ and can be readily solved by the projected gradient descent algorithm. The gradient of $F_{\ALP D}$ is given by
\begin{align}
    \nabla_\theta F_{\ALP D}(\theta) = \frac{1}{N}\sum_{i=1}^N \nabla_\theta l(u_i, \theta).\nn
\end{align}
The usual projected gradient descent algorithm approximates $\theta^*$ by performing the following iterative operation $T:\Theta\to\Theta$. Given an initial point $\theta_0\in\Theta$,
\begin{align}
\label{eqn: exact gradient descent}
    \theta_{k+1} = T(\theta_k) \coloneqq \proj_{\Theta}\{\theta_k -\beta \nabla_\theta F_{\ALP D}(\theta_k)\} = \proj_{\Theta}\bigg\{\theta_k - \frac{\beta}{N}\sum_{i=1}^N \nabla_\theta l(u_i, \theta_k)\bigg\},
\end{align}
where $\beta\in(0, 1/L]$ is the learning rate and $\proj_\Theta\{\cdot\}$ is the projection operator on to $\Theta$. The following Lemma gives the fact that $T$ is a contraction with a unique fixed point $\theta^*$. 
\begin{lemma}
Suppose Assumption \ref{asp: strong convex and smooth} is satisfied. Then, $F_{\ALP D}$ has unique minimizer $\theta^*\in\bb R^{d_p}$. Given $\theta_0\in\Theta$ and $\beta \in (0, 2/L)$, $T$ is a contraction with coefficient $\alpha_T = \max\{|1-\beta\eta|, |1-\beta L|\}$ and unique fixed point $\theta^*$, and the sequence $(\theta_k)_{k\in\bb{N}}$ generated by Equation \ref{eqn: exact gradient descent} converges geometrically,
\begin{align}
    \|\theta_{k} - \theta^*\|_2 \leq \alpha_T^k\|\theta_0 - \theta^*\|_2.
\end{align}
\end{lemma}
\begin{proof}
    $F_\cal D$ has a unique minimizer comes from the fact $F_\cal D$ is continuous and strongly convex. To see $T$ is a contraction, we first consider $h(\theta) = F_\cal D(\theta) - \frac{\eta}{2}\|\theta\|_2^2$ which is convex and with $(L-\eta)$-Lipschitz gradient. Due to the co-coercivity of $\nabla_\theta h$, we have, for $\theta, \theta' \in \Theta$,
    \begin{align}
    \label{eqn: co-coercivity }
        \langle\nabla_\theta F_\cal D (\theta) - \nabla_{\theta} F_\cal D (\theta'), \theta - \theta'\rangle \geq \frac{L\eta}{L+\eta}\|\theta - \theta'\|^2_2 + \frac{1}{\eta + L} \|\nabla_\theta F_\cal D(\theta) - \nabla_\theta F_\cal D(\theta')\|_2^2.
    \end{align}
    Now, we bound the different between $T(\theta)$ and $T(\theta')$ using Equation \ref{eqn: co-coercivity } and the nonexpansiveness of the projection operator,
    \begin{align}
        &\|T(\theta) - T(\theta')\|_2^2 \leq \|(\theta - \theta') - \beta(\nabla_\theta F_\cal D (\theta) - \nabla_\theta F_\cal D (\theta))\|^2_2\nn\\
        =& \|\theta - \theta'\|_2^2 + \beta^2\|\nabla_\theta F_\cal D (\theta) - \nabla_\theta F_\cal D (\theta)\|^2_2 - 2\beta\langle\nabla_\theta F_\cal D (\theta) - \nabla_{\theta} F_\cal D (\theta'), \theta - \theta'\rangle\nn\\
        \leq & \bigg[1- 2\beta\bigg(\frac{L\eta}{L+ \eta}\bigg)\bigg]\|\theta - \theta'\|_2^2 + \bigg(\beta^2 - \frac{2\beta}{L+\eta}\bigg)\|\nabla_\theta F_\cal D (\theta) - \nabla_\theta F_\cal D (\theta)\|^2_2.\nn
    \end{align}
    Suppose $\beta \in (0, 2/L)$, we consider the following two cases:

    \textbf{Case 1} $\beta \leq \frac{2}{L+\eta}$: in this case, $\beta^2 - \frac{2\beta}{L+\eta} \leq 0$, we replace $\|\nabla_\theta F_\cal D (\theta) - \nabla_\theta F_\cal D (\theta)\|_2$ with its lower bound $\eta\|\theta - \theta'\|_2$ which is a consequence of the $\eta$-strong convexity. Then, we arrive at 
    \begin{align}
        \|T(\theta) - T(\theta')\|_2 \leq (1 -\beta\eta)\|\theta - \theta'\|_2,\nn
    \end{align}
    which indicates the contraction rate of $T$ is $1 - \beta\eta$. One can check $1 - \beta\eta \geq |1 - \beta L|$ when $\beta \in (0, \frac{2}{L+\eta})$ .

    \textbf{Case 2} $\beta \geq \frac{2}{L+\eta}$: in this case, $\beta^2 - \frac{2\beta}{L+\eta} \geq 0$, we replace $\|\nabla_\theta F_\cal D (\theta) - \nabla_\theta F_\cal D (\theta)\|_2$ with its upper bound $L\|\theta - \theta'\|_2$ which is a consequence of having a $L$-Lipschitz gradient. Then, we arrive at 
    \begin{align}
        \|T(\theta) - T(\theta')\|_2 \leq (\beta L -1)\|\theta - \theta'\|_2,\nn
    \end{align}
    which indicates the contraction rate of $T$ is $(\beta L - 1)$. One can check $1 - \beta\eta \geq |1 - \beta L|$ when $\beta \in (\frac{2}{L+\eta}, \frac{2}{L})$.

    Combining the above two cases, we have that $T$ is a contraction with rate $\alpha_T = \max\{|1-\beta\eta|, |1-\beta L|\}$ and $\beta\in (0, 2/L)$ ensures $\alpha_T < 1$. By Banach contraction mapping theorem \cite[Theorem 1.1]{granas2003fixed}, we arrive at the desired result after verifying $T(\theta^*) = \theta^*$. 
\end{proof}

In practice, the exact gradient descent in Equation \ref{eqn: exact gradient descent} is replaced by the SGD to reduce the computational complexity, especially when $N$ is large. At the time $k$, the algorithm chooses a subset of $n$ samples $\ALP N_k\subset\{1, \cdots, N\}$ randomly and independently from previous steps. The gradient at time $k$ is approximated with samples in $\ALP N_k$. Let $\hat{T}_k^n$ be the SGD operator, and we write the update equation as
\begin{align}
    \hat{\theta}^n_{k+1} = \hat{T}_k^n(\hat{\theta}^n_{k}) \coloneqq \hat{\theta}_k^n -\frac{\beta}{n}\sum_{j\in\ALP N_k} \nabla_\theta l(u_j, \hat \theta^n_k),\nn
\end{align}
and $\boldsymbol{\hat\theta}^n \coloneqq \{\hat{\theta}^n_k\}_{k\in\bb{N}}$ is in fact a Markov chain. Under some mild conditions, we show, in Lemma \ref{thm: theta chain conv in wass}, that $\boldsymbol{\hat\theta}^n$ admits a unique invariant distribution $\pi_n$ and converges to $\pi_n$ weakly and in Wasserstein metric as $k\to\infty$. Lemma \ref{thm: theorem 5.5} enables the weak convergence of $\boldsymbol{\hat{\theta}}$, however, it remains to check that the random operators $\{\hat T^n_k\}_{k\in\bb{N}}$ satisfy Assumption \ref{asp: assumption 5.3}. 
\begin{lemma}
\label{thm: sgd satisfies weak convergence assumption}
    Under Assumption \ref{asp: strong convex and smooth}, $\{\hat T^n_k\}_{k\in\bb{N}}$ satisfies Assumption \ref{asp: assumption 5.3}
\end{lemma}
\begin{proof}
    Assumption \ref{asp: assumption 5.3} (1) is satisfied by Assumption \ref{asp: strong convex and smooth} (5). Assumption \ref{asp: assumption 5.3} (2) can be shown by applying Hoeffding's inequality for bounded i.i.d. random variables. Without loss of generality, we can pick the metric $\msf d$ in Equation \ref{eqn: stochastic operator close to true in prob} to be $\|\cdot\|_1$ as the $\Theta$ has finite dimensions. Given $\epsilon > 0$ and $k \in \bb N$,
    \begin{align}
        \bb P\bigg\{ \bigg\|\hat T^n_k(\theta^*) - \theta^*\bigg\|_1 > \beta\epsilon\bigg\} &= \bb P\bigg\{ \bigg\|\frac{1}{n}\sum_{j\in\ALP N_k}\nabla_\theta l(u_j, \theta^*)\bigg\|_1 > \epsilon\bigg\} \stackrel{(a)}{\leq} 2d_p\exp\bigg\{\frac{-n\epsilon^2}{\sigma}\bigg\},\nn
    \end{align}
    where $(a)$ is due to \citet{hoeffding1994probability} and $\sigma \in (0, \infty)$ is a constant related to the uniform norm of $\nabla_\theta l(\cdot, \theta)$.
    
    Assumption \ref{asp: assumption 5.3} (3) can be established by Assumption \ref{asp: strong convex and smooth} (1)(2)(3) which, together, imply the fact that eigenvalues of $\nabla^2_{\theta\theta}l(u, \theta)$ fall in the interval of $[\eta, L]$. By picking the appropriate learning rate $\beta$, $\hat T^n_k$ is a contraction for every realization with contraction rate $\hat a_k^n = \max\{|1-\beta L|, |1 - \beta \eta|\} < 1$.
\end{proof}

 Combining the results in Lemma \ref{thm: sgd satisfies weak convergence assumption} and \ref{thm: theorem 5.5}, we can conclude $\boldsymbol{\hat\theta}^n$ converges to $\pi_n$ weakly. Conveniently, for metric spaces with bounded diameter, which is the case for $\Theta$, Wasserstein-(1, $\|\cdot\|_2$) metrizes the weak topology \citep{gibbs2002choosing}, i.e. weak convergence is equivalent to convergence in Wasserstein-(1, $\|\cdot\|_2$). Using the fact that $\tilde{\mbf F}_\bl \subseteq \tilde{\mbf F}_\wass = \{f\in\bb M(\Theta) : \lip(f)\leq 1\}$, we can deduce that the bounded Lipschitz metric $\bl$ is bounded by Wasserstein-(1, $\|\cdot\|_2$), and $\boldsymbol{\hat\theta}^n$ also converges in $\bl$. Therefore, it suffices to verify that $\Gamma_\wass$ is finite.

\begin{lemma}
\label{thm: theta chain conv in wass}
    $\boldsymbol{\hat\theta}^n$ converges to its unique invariant distribution $\pi_n$ in Wasserstein-(1, $\|\cdot\|_2$) with constant $\Gamma_\wass$ defined in Equation \ref{eqn: hoeff ineq wass cond bdd},
    \begin{align}
        \Gamma_{\wass} = \sup_{\theta\in\Theta}\wass_{1, \|\cdot\|_2}(\delta_{\theta}, \pi_n) \sum_{k=1}^\infty \prod_{i=0}^k\hat\alpha_i^n < \infty,\nn
    \end{align}
    where $\hat\alpha_i^n$ is the contraction coefficient of $T_i^n$ for $i=0, 1, 2, \cdots$.
\end{lemma}
\begin{proof}
    Let $Q_n$ denote the transition kernel of $\boldsymbol{\hat\theta}^n$. Consider two independent chains starting from $\mu$ and $\nu$, respectively. The contraction in the Wasserstein metric from time $k$ to $k+1$ is written as follows,
    \begin{align*}
        \wass_{1, \|\cdot\|_2} (\mu Q_n^{k+1}, \nu Q_n^{k+1}) =& \inf_{\Pi\in\Pi(\mu P^k, \nu P^k)}\int_{\Theta\times\Theta} \|\hat T^n_k(\theta) - \hat T^n_k(\theta')\|_2 d\Pi(\mu Q_n^k, \nu Q_n^k)\nn\\
        \leq& \hat\alpha_k^n \inf_{\Pi\in\cal C(\mu Q_n^k, \nu Q_n^k)} \int_{\Theta\times\Theta} \|\theta - \theta'\|_2d\Pi(\mu Q_n^k, \nu Q_n^k)\nn\\
        = & \hat\alpha_k^n \wass_{1, \|\cdot\|_2} (\mu Q_n^k, \nu Q_n^k),
    \end{align*}
    for $k = 0, 1, 2,\cdots$. Thus, $\{\alpha_k^n\}$ is also the contraction coefficient of the $\boldsymbol{\hat\theta}$ in Wasserstein-(1, $\|\cdot\|_2$). Consequently, we can deduce that
    \begin{align*}
        \Gamma_\wass &= \sup_{\theta\in\Theta} \sum_{k=0}^\infty \wass_{1, \|\cdot\|_2}(Q_n^k(\theta, \cdot), \pi_n) =  \sup_{\theta\in\Theta} \sum_{k=0}^\infty \wass_{1, \|\cdot\|_2}(\delta_{\theta} Q_n^k, \pi_n Q_n^k) \\
        &\leq  \sup_{\theta\in\Theta} \sum_{i=0}^\infty \wass_{1, \|\cdot\|_2} (\delta_{\theta}, \pi_n) \prod_{i=0}^k\hat\alpha_i^n.
    \end{align*}
    The proof is complete by noting $\hat\alpha_i^n < 1$ for every realization.
\end{proof}
Combining Lemma \ref{thm: theta chain conv in wass} and the fact that $\bl \leq \wass_{1, \|\cdot\|_2}$ yields the following corollary which states that $\Gamma_\bl$ defined in Equation \ref{eqn: hoeff ineq wass cond bdd} is also bounded. This completes our preparation for the main conclusion in this example.
\begin{corollary}
\label{cor: bl gamma is bounded}
    It holds that $\Gamma_\bl \leq \Gamma_\wass < \infty$.
\end{corollary}

With the above development, we are ready to present the main result of this example. Corollary \ref{cor: bl gamma is bounded} allows us to apply Hoeffding's inequality in Corollary \ref{thm: Hoeffding ineq wass bdd space} to determine the concentration of Ployak-Ruppert averaging. As a result, we demonstrate the sample complexity bound for the averaged iterates in the following theorem. 
\begin{theorem}
\label{thm: RSA sample complexity}
     Let $\bar{\theta}^n_m = \frac{1}{m}\sum^{m-1}_{k=0}\hat{\theta}^n_k$ for $m\geq 1$. Suppose the Markov chain $\boldsymbol{\hat\theta}^n$ converges in Wasserstein-(1, $\|\cdot\|_2$) metric with constant $\Gamma_\bl$. Given $\epsilon > 0$ and $\delta \in (0, 1)$, the $\|\bar{\theta}^n_m - \bb E_{\pi_n}(\theta)\|_2 < \epsilon$ holds with probability at least $1-\delta$ whenever
    \begin{align}
    \label{eqn: RSA sample complexity}
        m \geq \frac{d_p^2}{4\epsilon^2}\log\frac{2d_p}{\delta} \diam(\Theta)^2(2\Gamma_\bl +1)^2,
    \end{align}
    where $\diam(\Theta) <\infty$ is the diameter of the parameter space $\Theta$ in $\|\cdot\|_2$.
\end{theorem}
\begin{proof}
    Let $f$ be the identity function on $\Theta$ and $\|\cdot\|_1$ denote the 1-norm on $\bb{R}^{d_p}$. We have $\|\theta\|_2 \leq \|\theta\|_1$ for all $\theta \in \bb{R}^{d_p}$, which implies
    \begin{align}
    \label{eqn: 2 norm to 1 norm ineq}
        &\|\bar{\theta}^n_m - \pi_n(f)\|_2 = \bigg\|\frac{1}{m}\sum^{m-1}_{k=0}\hat{\theta}^n_k - \pi_n(f)\bigg\|_2 \nn\\ 
        \leq &\bigg\|\frac{1}{m}\sum^{m-1}_{k=0}\hat{\theta}^n_k - \pi_n(f)\bigg\|_1 = \sum_{i=1}^{d_p}\bigg|\frac{1}{m} \sum_{k=0}^{m-1}f_i(\hat{\theta}^n_k) - \pi_n(f_i)\bigg|,
    \end{align}
    where $f_i:\bb{R}^{d_p} \to \bb{R}$ maps a vector in $\bb{R}^{d_p}$ to its $i$-th component for $i\in\{1, \cdots, d_p\}$. For identity function $f$, it holds $\|f_i\|_\bl = \max\{\lip(f_i), \|f_i\|_\infty\} \leq \max\{1, \diam(\Theta)/2\}$. Applying Equation \ref{eqn: Hoeffding ineq wass bdd} to each summand of the last term in Equation \ref{eqn: 2 norm to 1 norm ineq} yields
    \begin{align}
        \bb{P}_{\theta_0}\bigg[\bigg|\frac{1}{m} \sum_{k=0}^{m-1}f_i(\hat{\theta}^n_k) - \pi_n(f_i)\bigg| > \epsilon \bigg] \leq 2\exp\bigg\{\frac{-2(m\epsilon - 2\|f_i\|_\bl\Gamma_\bl - \sp(f_i))^2}{m(2\|f_i\|_\bl\Gamma_\bl+ \sp(f_i))^2}\bigg\},
    \end{align}
    for $i\in\{1,\cdots, d_p\}$. After applying union bound, we have
    \begin{align}
         \bb{P}_{\theta_0} &\bigg[\sum_{i=1}^{d_p}\bigg|\frac{1}{m} \sum_{k=0}^{m-1}f_i(\hat{\theta}^n_k) - \pi_n(f_i)\bigg| > \epsilon \bigg] \leq \bb{P}_{\theta_0}\bigg[\bigcup_{i=1}^{d_p}\bigg\{\bigg|\frac{1}{m} \sum_{k=0}^{m-1}f_i(\hat{\theta}^n_k) - \pi_n(f_i)\bigg| > \frac{\epsilon}{d_p}\bigg\} \bigg]\nn\\
         &\leq \sum_{i=1}^{m_{\theta}}\bb{P}_{\theta_0}\bigg[\bigg|\frac{1}{m} \sum_{k=0}^{m-1}f_i(\hat{\theta}^n_k) - \pi_n(f_i)\bigg| > \frac{\epsilon}{d_p}\bigg]\nn\\
         &\leq 2d_p\exp\bigg\{\frac{-2(m\epsilon/d_p - 2\diam(\Theta)\Gamma_\bl - \max_{i\in\{1, \cdots, d_p\}}\|f_i\|_\infty)^2}{m(2\diam(\Theta)\Gamma_\bl + \max_{i\in\{1, \cdots, d_p\}}\|f_i\|_\infty)^2}\bigg\}.
    \end{align}
    Note that $f$ is the identity map, thus $\max_{i\in\{1, \cdots, d_p\}}\|f_i\|_\infty = \sup_{\theta\in\Theta}\|\theta\|_\infty \leq \diam(\Theta)$. Using the inequality in Equation \ref{eqn: 2 norm to 1 norm ineq}, we conclude that
    \begin{align}
        \bb{P}_{\theta_0}[\|\bar{\theta}^n_m - \pi_n(f)\|_2 > \epsilon] &\leq  \bb{P}_{\theta_0}\bigg[\sum_{i=1}^{d_p}\bigg|\frac{1}{m} \sum_{k=0}^{m-1}f_i(\hat{\theta}^n_k) - \pi_n(f_i)\bigg| > \epsilon \bigg] \nn\\
        &\leq 2d_p\exp\bigg\{\frac{-2(m\epsilon/d_p - 2\diam(\Theta)\Gamma_\bl-\diam(\Theta))^2}{m(2\diam(\Theta)\Gamma_\bl + \diam(\Theta))^2}\bigg\}.\nn
    \end{align}
    By setting $\delta \geq2d_p\exp\bigg\{\frac{-2(m\epsilon/d_p - 2\diam(\Theta)\Gamma_\bl-\diam(\Theta))^2}{m(2\diam(\Theta)\Gamma_\bl + \diam(\Theta))^2}\bigg\}$ and solving for the lower bound of $m$, we have
    \begin{align*}
        m &\geq \frac{-b + \sqrt{b^2 - 4ac}}{2a}, \text{ where}\\
        a = \frac{2\epsilon^2}{d_p^2},\ 
        b &= -\log\frac{2d_p}{\delta} (2\diam(\Theta)\Gamma_\bl + \diam(\Theta))^2 - 4(2\diam(\Theta)\Gamma_\bl + \diam(\Theta))\frac{\epsilon}{d_p},\\
        c &= 2(2\diam(\Theta)\Gamma_\bl + \diam(\Theta)).
    \end{align*}
    For small $\epsilon$, $\delta$, and large $d_p$, $\Gamma_\cal W$, $\diam(\Theta)$, we have 
    \begin{align*}
        m \geq \frac{d_p^2}{4\epsilon^2}\log\frac{2d_p}{\delta} \diam(\Theta)^2(2\Gamma_\bl +1)^2,
    \end{align*}
    which yields the desired result in Equation \ref{eqn: RSA sample complexity}.
\end{proof}

\begin{remark}
    So far we have discussed the concentration of averaged iterates around the mean of $\pi_n$, where $n$ is the batch size. It is expected that as $n\to\infty$ the mean of $\pi_n\to\theta^*$. Indeed, one can directly apply \citet[Theorem 2.2]{gupta2018probabilistic} to the current example to show $\pi_n$ converges weakly to $\mathbbm{1}_{\theta^*}$ without additional assumption other than Assumption \ref{asp: strong convex and smooth}. 
\end{remark}

\subsection{Rested Bandit with UCB}
\label{sec: rested bandit}
In this section, we consider the finite sample analysis of the UCB-type policy for the multi-arm bandit (MAB) with Markovian reward. Suppose $\bb{X}$ is a Polish space. We assume the state of the $i$-th arm $\{X_{i, n}\}_{n\in\bb{N}} \subset \bb{X}$ is an Markov chain with transition kernel $P_i$, for $i\in\{1, \cdots, K\}$. The states of arms evolve independently. The reward of the $i$-th arm is a non-negative function of the state $r_i: \bb{X}\to [0, \infty)$. This is known as the rested or restless bandit problem in the literature. The distinction between rested and restless bandits is that the Markov chain associated with an arm in the former evolves only when the arm is pulled, whereas the Markov chains in the latter evolve regardless of which arm is pulled at each time step. Since the purpose of this example is to demonstrate the usage of Hoeffding's inequality in the analysis of the regret bound, we opt for the classical UCB-type policy for the sake of simplicity. For finite state space and uniform ergodic Markov chains, the UCB-type policy for both rested and restless bandits is studied in \citet{tekin2012online}. We restate their UCB-M algorithm for rested bandit in Algorithm \ref{alg: ucb-m for rested bandit} for completeness.

With our theory, the regret analysis can be readily extended to general state space Markov chains with weaker convergence requirements. The logarithmic regret bound is presented in Theorem \ref{thm: rested bandit regret bound}. Before presenting the regret analysis, we make the following further assumptions for the state process and reward function of each arm.

\begin{assumption}
\label{asp: rested bandit markov chain}
    For each arm $i$, suppose there exist a measurable function $V_i: \bb{X} \to [1, \infty)$
    such that the following condition holds:
    \begin{enumerate}
        \item There exists a differentiable concave increasing function $\phi_i: \bb{R}_+ \to \bb{R}_+$,  $\phi_i(0) = 0$ and a constant $b_i \geq 0$ such that the following drift condition holds with Lyapunov function $V_i$,
        \begin{align}
            P_iV_i \leq  V_i - \phi_i \circ V_i + b_i.\nn
        \end{align}
        \item The Markov chain $\{X_i(n)\}_{n\in\bb{N}}$ induced by transition kernel $P_i$ converges to its unique stationary distribution $\pi_i$ in IPM with generator $\mbf F_i \subset \bb{M}_1(\bb{X})$ and constant $\Gamma_{\mbf F_i}$.
    \end{enumerate}
\end{assumption}
\begin{assumption}
\label{asp: rested bandit reward function}
For $i\in\{1,\cdots, K\}$, the reward function $r_{i}\in R_i\cdot\mbf F_i$ for $R_i<\infty$. Denote $R_{\max} = \max\{R_1, \cdots, R_K\}$ and $R_{\min} = \min\{R_1, \cdots, R_K\}$
\end{assumption}
\begin{remark}
Assumption \ref{asp: rested bandit markov chain} (1) comes from \citet[Theorem 2.1]{butkovsky2014subgeometric} which shows a set of sufficient conditions for convergence of Markov chains in Wasserstein-(1, $\msf d$) metric with subgeometric rate. Assumption \ref{asp: rested bandit markov chain} (1) is known as the drift condition for subgeometric convergence. It becomes the drift condition for geometric convergence when $\phi_i$ is a linear function. Assumption \ref{asp: rested bandit markov chain} (1) turns out to be a crucial component in the proof of Lemma \ref{thm: Tekin lemma 2 mod}. 
\end{remark}

For Markov chains converging in TV metric, necessary and sufficient conditions exist for both geometric and subgeometric rate \citep{douc2018markov}.
For Markov chains converging in the Wasserstein metric, sufficient conditions have been proposed for both geometric rate \citep{hairer2011asymptotic} and subgeometric rate \citep{butkovsky2014subgeometric}. To the best of our knowledge, sufficient conditions for Markov chains converging in general IPM is still an open problem. The drift condition seems to be the common requirement for convergence in TV metric and Wasserstein metric, and they only differ in the definition of the small set (or $d$-small set in the case of Wasserstein-(1, $d$) metric). We only speculate here that the convergence in general IPM requires a new definition of ``small set". 

We introduce the following notations which are used in the sequel.
\begin{itemize}
    \item $X_i(n)$ is the state of arm $i$ at time $n$.
    \item $\mu_i$ is the initial distribution of the state process of arm $i$.
    \item $A(n)$ is the set of selected arms at time $n$.
    \item $T_i(n)$ is the number of time arm $i$ pulled up to time $n$.
    \item $\bar r_i(n) = (1/n)\sum_{t=1}^n r_i(X_i(t))$ is the (sample) average reward for arm $i$ after $n$ pulls.
    \item $h_i(n)$ is the UCB index of arm $i$ at time $n$.
    \item $\eta_i = \pi_i(r_i)$ is the expected reward of arm $i$ under stationary distribution $\pi_i$.
    \item $L$ is the exploration parameter for the UCB policy.
    \item $M$ number of arms pulled per time step, $M\in\{1,\cdots,K-1\}$.
\end{itemize}
Suppose the arms are ordered by their expected rewards in decreasing order, that is, $\eta_1 \geq \eta_2 \geq \cdots \geq \eta_M \geq \eta_{M+1} \geq \cdots \geq \eta_K$. The equality between $\eta_M$ and $\eta_{M+1}$ is allowed under the rested-arm assumption, whereas a strict inequality is required for restless bandits \citep{tekin2012online}. The regret $\ALP R(n)$ is defined as the difference in the expected total reward up to time $n$ between playing the first $M$ arms only and playing according to some policy that determines the selected arms $A(n)$. Formally, we have 
\begin{align}
    \ALP R(n) = n\sum_{j=1}^M\eta_j - \bb{E} \bigg[\sum_{t=1}^n\sum_{i\in A(t)} r_i(X_i(t))\bigg].\nn
\end{align}
The goal of this application example is to show that regret is order-optimal under the UCB-M policy. Before presenting the regret bound in Theorem \ref{thm: rested bandit regret bound}, we introduce the supporting results in Lemma \ref{thm: Tekin lemma 2 mod}, \ref{thm: Tekin lemma 7 mod}, and \ref{thm: Tekin lemma 5 mod}. 
\begin{algorithm}[h]
\caption{UCB-M for rested bandit}\label{alg: ucb-m for rested bandit}
\SetAlgoLined
\SetKwInput{Input}{input}
\Input{exploration constant $L$\;}
Play each arm $M$ time in the first $K$ time steps\;
\While{$n \geq K$}{
\ForEach{$i \in \{1, \cdots, K\}$}{
   Calculate index: $h_{i}(n) = \bar r_i(T_i(n)) + \sqrt{\frac{L\log n}{T_i(n)}}$\;
}
Play $M$ arms with highest indices\;
Observe $R_{i,n+1}$ and calculate $T_i(n+1)$ for $\forall i\in\{1, \cdots, K\}$\;
$n \gets n + 1$\;
}
\end{algorithm}

Lemma \ref{thm: Tekin lemma 2 mod} is a modification of \citet[Theorem 2.1]{moustakides1999extension}. The original assumption is that the Markov chain is geometrically convergent in TV metric which implies the drift condition is satisfied \citep[Theorem 1.1]{moustakides1999extension}. We replace that with Assumption \ref{asp: rested bandit markov chain} (1) and obtain the following Lemma. Lemma \ref{thm: Tekin lemma 7 mod} gives an upper bound on the average of times the optimal arm is missed by the UCB-M policy. It is an adaption of \citet[Lemma 7]{tekin2012online} using the Hoeffding-type inequality from Theorem \ref{thm: main result time indep}. Building on top of Lemma \ref{thm: Tekin lemma 7 mod},  Lemma \ref{thm: Tekin lemma 5 mod} further upper-bounds the expected number of times each suboptimal arm is pulled. Combining Lemma \ref{thm: Tekin lemma 2 mod} and \ref{thm: Tekin lemma 5 mod}, we are able to deduce the logarithmic regret bound in Theorem \ref{thm: rested bandit regret bound}.

\begin{lemma}
\label{thm: Tekin lemma 2 mod}
    Suppose Assumption \ref{asp: rested bandit reward function} and \ref{asp: rested bandit markov chain} are satisfied and $\bb{E}V_i(X_i(0))<\infty$ where the expectation is taken with respect to initial distribution $\mu_i$, then for any stopping time $\tau$ such that $\bb{E}\tau < \infty$, we have
    \begin{align}
        \sum_{n=1}^\tau r_i(X_i(n)) - \eta_i \bb{E}\tau =  \bb{E}g_i(X_i(0)) - \bb{E}g_i(X_i(N)),\nn
    \end{align}
    where
    \begin{align}
        g_i(x) = \sum_{n=1}^\infty \bigg[\bb{E}_x [r_i(X_n)] - \pi_i(r_i)\bigg]\nn
    \end{align}
    is the solution to $g_i(x) - P_ig_i(x) = r_i(x) - \pi_i(f_i)$ which is the Poisson's equation associated with transition kernel $P_i$.
\end{lemma}

\begin{proof}
    From Assumption \ref{asp: rested bandit markov chain} (1), we have
    \begin{align}
    \label{eqn: subgeo drift}
        P_iV_i \leq  V_i - \phi_i \circ V_i + b_i.
    \end{align}
    Since $\phi_i$ is differentiable and concave on $\bb{R}_+$, its derivative $\phi_i'$ is non-increasing and has no simple discontinuity in any closed interval in $\bb{R}_+$. However, monotonic functions have no discontinuities of the second kind. Thus, it follows that $\phi_i'$ is a continuous function on $\bb{R}_+$. Suppose $\phi_i'(x) \neq 0$ for $x\in\bb{R}_+$, $\phi$ is invertible on $\bb{R}_+$ by inverse function theorem, and $\phi_i^{-1}$ is a convex function. Applying $\phi_i^{-1}$ to both sides of Equation \ref{eqn: subgeo drift} and Jensen's inequality to the right-hand side, we have
    \begin{align}
        P_i(\phi_i^{-1}\circ V_i) \leq  \phi_i^{-1}\circ V_i - V_i + \phi_i^{-1}(b_i).
    \end{align}
    For any function $f\in\bb{M}_{V_i}(\bb{X})$, i.e. $\|f\|_{V_i} = \sup_{x\in\bb{X}}|f(x)/V_i(x)|<\infty$, and any stopping time $\tau$, we can apply Proposition 11.3.2 from \citet{meyn2012markov} with $Z_k =\|f\|_{V_i} \phi_i^{-1} \circ V_i(X_i(k))$, $f_k(x) = |f(x)|$ and $s_k(x) = \|f\|_{V_i}\phi_i^{-1}(K_i)$ and get 
    \begin{align}
    \label{eqn: stopping time sum for mc}
        \bb{E}\bigg[\sum_{n=0}^{\tau-1}|f(X_i(n))|\bigg] \leq \|f\|_{V_i}\bigg(\bb{E}[\phi_i^{-1}\circ V_i(X_i(0))] + \phi_i^{-1}(b_i) \bb{E}\tau\bigg).
    \end{align}
    Since $r_i$ is bounded by Assumption \ref{asp: rested bandit reward function}, $\|r_i\|_{V_i}\leq\|r_i\|_\infty < \infty$ and we replace $f$ with $r_i$ in Equation \ref{eqn: stopping time sum for mc}. Then, the desired result follows the proof of Theorem 2.1 from \citet{moustakides1999extension} and the fact that we assume $r_i$ is bounded and $\bb{E}V_i(X_i(0)) < \infty$.  
\end{proof}

\begin{lemma}
\label{thm: Tekin lemma 7 mod}
Suppose Assumption \ref{asp: rested bandit markov chain} holds and all arms are rested. Under UCB-M algorithm with constant $L\geq(R_i\Gamma_{\mbf F_i} + \sp(r_i))^2\vee (R_j\Gamma_{\mbf F_j} + \sp(r_j))^2(1/\sqrt{\log K} + 2\sqrt{2})^2$, for any suboptimal arm $i\in\{M+1, \cdots, K\}$ and optimal arm $j\in\{1,\cdots,M\}$, we have
\begin{align}
    \bb{E}\bigg[\sum_{t=K+1}^n\sum_{s=1}^{t-1}\sum_{s_i=l}^{t-1} \bbm{1}\{h_{j}(s)\leq h_{i}(s_i)\}\bigg] \leq 2\sum_{t=1}^n t^{-2},
\end{align}
where $l = \left\lceil\frac{4L\log n}{(\eta_M - \eta_i)^2}\right\rceil$.
\end{lemma}
\begin{proof}
     Let $c_{t, s} = \sqrt{(L\log t)/s}$ and observe that the event $\{h_{j, s} \leq h_{i, s_i}\}$ implies that at least one of the following must hold
    \begin{align}
        \bar{r}_j(s) &\leq \eta_j - c_{t, s}\label{eqn: k arm bandit event decomp 1}\\
        \bar{r}_{i}(s_i) &\geq \eta_i + c_{t, s_i}\label{eqn: k arm bandit event decomp 2}\\
        \eta_j &< \eta_i + 2c_{t, s_i}.\label{eqn: k arm bandit event decomp 3}
    \end{align}
    We bound the probability of events Equation \ref{eqn: k arm bandit event decomp 1} and Equation \ref{eqn: k arm bandit event decomp 2} using Theorem \ref{thm: main result time indep},
    \begin{align}
        \bb{P}\{\bar{r}_j(s) \leq \eta_j - c_{t, s}\} &= \bb{P}\bigg\{\frac{1}{s}\sum_{t=1}^s r_j(t)\leq  \bb{E}_\pi r_j - c_{t, s} \bigg\} \leq  \exp\bigg(\frac{-2(s c_{t, s}- (R_j\Gamma_{\mbf F_j} + \sp(r_j)))^2}{s(R_j\Gamma_{\mbf F_j} + \sp(r_j))^2 }\bigg)\nn\\
        &\leq \exp\bigg(\frac{-2[\sqrt{L\log t}- (R_j\Gamma_{\mbf F_j} + \sp(r_i))/\sqrt{s}]^2}{(R_j\Gamma_{\mbf F_j} + \sp(r_j))^2}\bigg),\label{eqn: k arm bandit event bound 1}\\
        \bb{P}\{\bar{r}_{i}(s_i) \geq \eta_i + c_{t, s_i}\} &= \bb{P}\bigg\{\frac{1}{s_i}\sum_{t=1}^{s_i} r_i(t)\geq \bb{E}_\pi r_i + c_{t, s_i} \bigg\} \leq \exp\bigg(\frac{-2(s_i c_{t, s_i}- (R_i\Gamma_{\mbf F_i} + \sp(r_i)))^2}{s_i(R_i\Gamma_{\mbf F_i} + \sp(r_i))^2}\bigg)\nn\\
         &\leq \exp\bigg(\frac{-2[\sqrt{L\log t}- (R_i\Gamma_{\mbf F_i} + \sp(r_i))/\sqrt{s_i}]^2}{(R_i\Gamma_{\mbf F_i} + \sp(r_i))^2}\bigg). \label{eqn: k arm bandit event bound 2}
    \end{align}
    If we choose $s_i \geq (4L\log n) / (\eta_M - \eta_i)^2$, then
    \begin{align}
        2c_{t, s_i} = 2\sqrt{\frac{L\log t}{s_i}} \leq 2\sqrt{\frac{L\log t(\eta_M - \eta_i)^2}{4L\log n}} \leq \eta_j - \eta_i \quad \text{for } t\leq n,
    \end{align}
    which indicates event Equation \ref{eqn: k arm bandit event decomp 3} never occurs. Combining Equation \ref{eqn: k arm bandit event bound 1} and Equation \ref{eqn: k arm bandit event bound 2}, we obtain the desired result by choosing an appropriate $L \geq (R_i\Gamma_{\mbf F_i} + \sp(r_i))^2\vee (R_j\Gamma_{\mbf F_j} + \sp(r_j))^2 (1/\sqrt{\log K } + 2\sqrt{2})^2$,
    \begin{align}
        \bb{E}&\bigg[\sum_{t=K+1}^n\sum_{s=1}^{t-1}\sum_{s_i=l}^{t-1} \bbm{1}\{h_{j}(s)\leq h_{i}(s_i)\}\bigg] \\
        &\leq 2\sum_{t=K+1}^n t^2\exp\bigg(\frac{-2[\sqrt{L\log t}- [(R_i\Gamma_{\mbf F_i} + \sp(r_i))/\sqrt{s_i}\vee (R_j\Gamma_{\mbf F_j} + \sp(r_j))/\sqrt{s}]]^2}{(R_i\Gamma_{\mbf F_i} + \sp(r_i))^2\vee (R_j\Gamma_{\mbf F_j} + \sp(r_j))^2}\bigg)\nn\\
        &\leq 2\sum_{t=K+1}^n t^2\exp\bigg(\frac{-2\log t [\sqrt{L}- [(R_i\Gamma_{\mbf F_i} + \sp(r_i))/\sqrt{s_i\log t }\vee (R_j\Gamma_{\mbf F_j} + \sp(r_j))/\sqrt{s\log t }]]^2}{(R_i\Gamma_{\mbf F_i} + \sp(r_i))^2\vee (R_j\Gamma_{\mbf F_j} + \sp(r_j))^2}\bigg)\nn\\
        &\leq 2\sum_{t=K+1}^n t^2\exp(\log t^{-4})\nn\leq 2\sum_{t=K+1}^n t^{-2}.
    \end{align}
\end{proof}

\begin{lemma}
\label{thm: Tekin lemma 5 mod}
Suppose Assumption \ref{asp: rested bandit markov chain} holds and all arms are rested. Under UCB-M algorithm (Algorithm \ref{alg: ucb-m for rested bandit}) with exploration constant $L\geq \max_{i\in\{1,\cdots,K\}} \{(R_i\Gamma_{\mbf F_i} + \sp(r_i))^2\} (1/\sqrt{\log K} + 2\sqrt{2})^2$, for any suboptimal arm $i$, we have
\begin{align}
    \bb{E}[T^i(n)] \leq (1+2\beta)M + \frac{4L\log n}{(\eta_M - \eta_i)^2},
\end{align}
where $\beta = \sum_{t=1}^\infty t^{-2}$.
\end{lemma}
\begin{proof}
    For any arm $i$, we upper bound $T_i(n)$ on any sequence of plays. Let $l$ be an arbitrary positive integer. 
    \begin{align}
    \label{eqn: rested bandit ineq 1}
        T_i(n) &= M + \sum_{t=K+1}^n\bbm{1}\{i\in A(t)\} \leq M -1 + l + \sum_{t=K+1}^n \bbm{1}\{i\in A(t),\ T_i(t-1) \geq l\}
    \end{align}
    Consider the event 
    \begin{align}
        E_i &= \bigcup_{j=1}^M \bigg\{h_{j}(n) \leq h_{i}(n)\bigg\},\nn\\
        E_i^C &= \bigcap_{j=1}^M \bigg\{h_{j}(n) > h_{i}(n)\bigg\}.\nn
    \end{align}
    If $E^C_i$ occurs then $i\notin A(t)$. Thus, $\{i\in A(t)\} \subset E_i$ and
    \begin{align}
    \label{eqn: rested bandit ineq 2}
        \bbm{1}\{i\in A(t),\ T_i(t-1) \geq l\} &\leq \bbm{1}\{E_i, \ T_i(t-1) \geq l\}\nn\\
        &\leq \sum_{j=1}^M \bbm{1}\{h_{j}(n) \leq h_{i}(n),\ T_i(t-1)\geq l\}
    \end{align}
    Plug Equation \ref{eqn: rested bandit ineq 2} into Equation \ref{eqn: rested bandit ineq 1}, we have
    \begin{align}
        T_i(n) &\leq M -1 + l + \sum_{j=1}^M\sum_{t=K+1}^n \bbm{1}\{h_{j}(n) \leq h_{i}(n),\ T_i(t-1)\geq l\}\nn\\
        &\leq M -1 + l + \sum_{j=1}^M\sum_{t=K+1}^n \bbm{1}\bigg\{\min_{0<s<t}h_{j}(s) \leq \max_{l\leq s_i \leq t} h_{i}(s_i)\bigg\}\nn\\
        &\leq M -1 + l + \sum_{j=1}^M\sum_{t=K+1}^n\sum_{s=1}^{t-1}\sum_{s_i=l}^{t-1}\bbm{1}\{h_{j}(s)\leq h_{j}(s_i)\}.
    \end{align}
   Using Lemma \ref{thm: Tekin lemma 7 mod} with $l=\left\lceil\frac{4L\log n}{(\eta_M - \eta_i)^2}\right\rceil$, we have for any suboptimal arm $i$,
   \begin{align}
       \bb{E}[T_i(n)] \leq M + \frac{4L \log n}{(\eta_M - \eta_i)^2} + 2M\sum_{t=1}^\infty t^{-2}.
   \end{align}
\end{proof}

We conclude this example with the following regret bound for rested Markovian bandits under the generalized concentrability condition and general state space.
\begin{theorem}
\label{thm: rested bandit regret bound}
    Suppose the Markov chain $\{X_i(n)\}_{n\in\bb{N}}$ associated with each arm $i \in\{1, \cdots, K\}$ satisfy Assumption \ref{asp: rested bandit markov chain}, and the reward function $r_i$ satisfies Assumption \ref{asp: rested bandit reward function}. Then, the regret $\ALP R(n)$ of Algorithm \ref{alg: ucb-m for rested bandit} up to time $n$ is upper bounded by
    \begin{align}
        \ALP R(n) \leq 4L\log n \sum_{i > M} \frac{(\eta_1 - \eta_i)}{(\eta_M - \eta_i)^2} + \sum_{i > M} (\eta_1 - \eta_i)(1+2\beta)M + 2\sum_{i=1}^K R_i\Gamma_{\mbf F_i},
    \end{align}
    where $\beta > 0$ is a constant.
\end{theorem}
\begin{proof}
    \begin{align}
    \label{eqn: rested bandit regret step 1}
        n\sum_{j=1}^M\eta_i - \sum_{i=1}^{K}\eta_i\bb{E}[T_i(n)] &= \sum_{j=1}^M\sum_{i=1}^K \eta_j \bb{E}[T_{i,j}(n)] - \sum_{j=1}^M\sum_{i=1}^K \eta_i \bb{E}[T_{i,j}(n)] \nn\\
        &= \sum_{j=1}^M \sum_{i > M} (\eta_j - \eta_i)\bb{E}[T_{i,j}(n)] \leq \sum_{i > M} (\eta_1 -\eta_i)\bb{E}[T_i(n)]
    \end{align}
    Using Lemma \ref{thm: Tekin lemma 2 mod}, we can get the following equality in $(a)$ and $g_i$ is the solution to Poisson's equation associated with the transition kernel $P_i$ and function $r_i$.
    \begin{align}
    \label{eqn: rested bandit regret step 2}
        &\bigg|\ALP R(n) - \bigg(n\sum_{j=1}^M\eta_j - \sum_{i=1}^K \eta_i\bb{E}[T_i(n)]\bigg)\bigg| =\bigg|\bb{E}\bigg[\sum_{i=1}^K\sum_{t=1}^{T_{i}(n)}r_i(X_i(t))\bigg] - \sum_{i=1}^K \eta_i\bb{E}[T_i(n)]\bigg|\nn\\
        \stackrel{(a)}{=}& \bigg|\sum_{i=1}^K\bb{E}[g_i(X_i(0)) - g_i(X_i(N))]\bigg| \stackrel{(b)}{\leq} 2\sum_{i=1}^K R_i\Gamma_{\mbf F_i},
    \end{align}
    where $(b)$ comes from Lemma \ref{lem: ipm ineq}. 
    To bound the regret, we combine Equation \ref{eqn: rested bandit regret step 1} and Equation \ref{eqn: rested bandit regret step 2} and obtain the inequality $(c)$ and $(d)$ 
    \begin{align}
        \ALP R(n) &\stackrel{(c)}{\leq} n\sum_{j=1}^M\eta_j - \sum_{i=1}^K\eta_i\bb{E}[T_i(n)] + 2\sum_{i=1}^K R_i\Gamma_{\mbf F_i}
        \stackrel{(d)}{\leq} \sum_{i > M} (\eta_1 - \eta_i) \bb{E}[T_i(n)] + 2\sum_{i=1}^K R_i\Gamma_{\mbf F_i}\nn\\
        &\stackrel{(e)}{\leq}\sum_{i > M} (\eta_1 - \eta_i)\bigg[(1+2\beta)M + \frac{4L\log n}{(\eta_M - \eta_i)^2}\bigg] + 2\sum_{i=1}^K R_i\Gamma_{\mbf F_i}\nn\\
        &= 4L\log n \sum_{i > M} \frac{(\eta_1 - \eta_i)}{(\eta_M - \eta_i)^2} + \sum_{i > M} (\eta_1 - \eta_i)(1+2\beta)M + 2\sum_{i=1}^K R_i\Gamma_{\mbf F_i},
    \end{align}
    where $(e)$ and $\beta$ are from Lemma \ref{thm: Tekin lemma 5 mod}.
\end{proof}

\section{Conclusion}
\label{sec: conclusion}

In this paper, we study the Hoeffding-type inequality for Markov chains under the generalized concentrability condition and the IPM Dobrushin coefficient. The generalized concentrability condition provides a flexible notion of convergence for Markov chains, which allows easy application of the Hoeffding-type inequality in practice. The IPM Dobrushin coefficient provides an approach to estimate the concentrability constant whenever the problem structure allows it. Under our framework, several existing Hoeffding-type inequalities for Markov chains in the literature emerge as special cases. We provide three application examples from machine learning and reinforcement learning to demonstrate the utility of our theory. 

We hope our work can shed some light on the concentration of measure and convergence of Markov chains from the IPM perspective. Nevertheless, there are many problems left unanswered which we list a few here for future studies:
\begin{itemize}
    \item What are the necessary and sufficient conditions for the Markov chains to satisfy the generalized concentrability condition?

    \item Can we establish the connection between generalized concentrability condition and usual conditions studied in Markov chain theory, such as the drift condition, Harris recurrence, and small set?

    \item How to extend the Hoeffding inequality to the unbounded functions with sub-Gaussian distribution assumption?

    \item Are there more ways to estimate the concentrability constant in Equation \ref{eqn: gen con} other than the IPM Dobrushin coefficient?
\end{itemize}

\acks{The authors gratefully acknowledge the insightful discussions with our collaborators from Cisco Systems on potential applications of this work, including Ashish Kundu, Jayanth Srinivasa, and Hugo Latapie.

Hao Chen and Abhishek Gupta's work was supported by Cisco Systems grant GR127553. Yin Sun's work was supported in part by the NSF under grant CNS-2239677, and by the ARO under grant W911NF-21-1-02444. Ness Shroff's work has been supported in part by NSF grants NSF AI Institute (AI-EDGE) CNS-2112471, CNS-2106933, CNS-2312836, CNS-1955535, and CNS-1901057, by Army Research Office under Grant W911NF-21-1-0244, and Army Research Laboratory under Cooperative Agreement Number W911NF-23-2-0225.
}

\newpage
\appendix

\section{Proofs of Section \ref{sec: main results}}

The following two lemmas encapsulate simple inequalities repeatedly used in the proofs. Let $\mbf F \subseteq \bb M(\bb X)$ be a set of real-valued measurable functions on $\bb X$.

\begin{lemma}
\label{lem: ipm ineq}
Consider a bounded real-valued function $f\in M\cdot \mbf F$ which $M \in (0, \infty)$ is the minimal stretch of $\mbf F$ to include $f$. It holds
\begin{align}
    |\mu (f) - \nu (f)|\leq M\cdot\ipm_{\mbf F}(\mu, \nu)
\end{align}
\end{lemma}
\begin{proof}
    \begin{align*}
        |\mu(f) - \nu(f)| &= \bigg|\int_{\bb X} f d\mu - \int_{\bb X}f d\nu \bigg| = M \bigg|\int_{\bb X} \frac{1}{M} f d\mu - \int_{\bb X} \frac{1}{M} f d\nu \bigg|\nn\\
        &\leq M \sup_{g\in\mbf F} \bigg|\int_{\bb X} g d\mu - \int_{\bb X} g d\nu \bigg| = M\cdot\ipm_{\mbf F}(\mu, \nu)
    \end{align*}
\end{proof}

\begin{lemma}
\label{lem: ipm ineq for time dependent}
For $i\in\{0,\cdots,n-1\}$, consider a bounded real-valued function $f_i\in M_i\cdot \mbf F$ with minimal stretch $M_i\in(0, \infty)$. Suppose the transition kernel $P: \bb{X}\times\ALP X\to[0, 1]$ admits a unique invariant probability measure $\pi$, then we have
\begin{align*}
     \bigg|\sum_{i=0}^{n-1}\bb{E}_\mu[f_i(X_i)] - \pi(f_i) \bigg| \leq \sum_{i=0}^{n-1} M_i \ipm_{\mbf F}(\mu P^i, \pi),
\end{align*}
for any initial distribution $\mu\in\ALP P(\cal X)$.
\end{lemma}
\begin{proof}
    This is a direct consequence of Lemma \ref{lem: ipm ineq} after applying triangular inequality. 
\end{proof}

The following lemma demonstrates the martingale decomposition of cumulative sums of a Markov chain. 

\begin{lemma}
\label{lem: martingale decomp of mc}
    Let $\mbf X = \{X_i\}_{i\in\bb N}$ be a Markov chain with transition kernel $P$ and initial distribution $\mu\in\cal P(\cal X)$. Assume that $P$ admits a unique invariant probability measure $\pi$. For $i\in\{0, \cdots, n-1\}$ and $n\in\bb Z_+$, let $f_i: \bb X\to \bb R$ be different functions on the Markov chain. Let $\tilde S_n \coloneqq  \sum_{i=1}^{n-1} f_i(X_i)$ be cumulative sum. Let $\{\ALP F_i\}_{i=0}^{n-1}$ be the filtration generated by $\mbf X$, i.e., $\cal F_i = \sigma(X_0, \cdots, X_i)$. Then, there exists a Martingale difference sequences $\{D_i\}_{i=1}^{n-1}$ adapted to $\{\ALP F_i\}_{i=0}^{n-1}$ such that 
    \begin{align}
    \label{eqn: martingale decomp of mc}
        \tilde S_n - \bb E_\mu[\tilde S_n] = \sum_{i=1}^{n-1} D_i + f_0(X_0) - \bb E_{\mu}[f_0].
    \end{align}
\end{lemma}

\begin{proof}
    We employ the martingale decomposition technique introduced in Chapter 23 of \citet{douc2018markov}. For the following development, we use these short notations for tuples: $X_{m}^{n} \coloneqq (X_m, \cdots, X_{n})$ and $x_{m}^{n} \coloneqq (x_m, \cdots, x_{n})$, for $0\leq m \leq n$. For $i\in\{0, \cdots, n-1\}$, we define
    \begin{align}
    \label{eqn: definition of g_i}
        g_i (x_0^i) \coloneqq \int_{\bb{X}^n}\sum_{l=0}^{n-1} f_l(x_l)\prod_{j=i+1}^{n-1}P(x_{j-1}, dx_j) = \sum_{l=0}^{i} f_l(x_l) + \sum_{l=i+1}^{n-1} \bb{E}_{x_i}[f_l(X_l)],
    \end{align}
    where $g_{n-1}(x_{0}^{n-1}) = \sum_{l=0}^{n-1}f_l(x_l)$ and $f_l: \bb X\to [a_i, b_i]$ belongs to $M_{l}\cdot\mbf F$ with bounded span. With this definition, we observe that
    \begin{align}
        g_{n-1}(x_0^{n-1}) = \sum_{i=1}^{n-1}\bigg[g_i(x_0^i) - g_{i-1}(x_0^{i-1})\bigg] + g_0(x_0),\nn
    \end{align}
    and for $i\in\{1, \cdots, n-1\}$ and $x_0^{i-1}\in\bb{X}^{i}$,
    \begin{align}
        g_{i-1}(x_0^{i-1}) = \int_{\bb{X}^i}g_i(x_0^{i-1}, x_i) P(x_{i-1}, dx_i).\nn
    \end{align}
    Let $\ALP F_i = \sigma(X_0,\cdots, X_i)$, the above equation shows that $g_{i-1}(X_{0}^{i-1}) = \bb{E}[g_i(X_{0}^i) | \ALP F_{i-1}]$ $\bb{P}_{\mu}$-a.s. for $i\geq 1$. Thus, $\{g_i(X_0^i)\}_{i=0}^{n-1}$ is a $\bb{P}_{\mu}$-martingale adapted to filtration $\{\ALP F_i\}_{i=0}^{n-1}$ for initial distribution $\mu\in\ALP P(\ALP X)$. It follows the martingale decomposition of $\Tilde{S}_n$ centered at $\sum_{i=0}^{n-1}\bb{E}_\mu[f_i(X_i)]$,
    \begin{align}
        \Tilde{S}_n - \sum_{i=0}^{n-1}\bb{E}_\mu[f_i(X_i)] = g_{n-1}(X_0^{n-1}) = \sum_{i=1}^{n-1} D_i + g_0(X_0) - \sum_{i=0}^{n-1}\bb{E}_\mu[f_i(X_i)],\nn
    \end{align}
    where $\{D_i\}_{i=1}^{n-1} = \{g_i(X_0^i) - g_{i-1}(X_0^{i-1})\}_{i=1}^{n-1}$ is a martingale difference sequence. We arrive at the Equation \ref{eqn: martingale decomp of mc} after canceling out the common terms.
\end{proof}

The following lemma shows the span of the function $g_i$ defined in Equation \ref{eqn: definition of g_i}.

\begin{lemma}
\label{lem: span of g_i}
    Suppose the transition kernel $P$ with invariant probability $\pi$ satisfies $\Gamma_\mbf F < \infty$ (defined in Equation \ref{eqn: gen con}). For each $i\in\{0, \cdots, n-1\}$, if we define 
    \begin{align}
    \label{eqn: dummy variable}
        A_i \coloneqq 2M_{i+1}^{\max}\Gamma_{\mbf F} + \sp(f_i),
    \end{align}
    where $M_{i}^{\max} = \max \{M_i, \cdots, M_{n-1}\}$ for $i\in\{0, \cdots, n-1\}$ and $M_i^{\max} = 0$ for $i \geq n$, then it holds that for $x_0^i \in \bb X^{i+1}$ and $x_0 \in\bb X$,
    \begin{align}
        \label{eqn: g span bound}
        \inf_{x\in\bb X} g_i(x_0^{i-1}, x) &\leq g_i(x_0^i) \leq \inf_{x\in\bb X} g_i(x_0^{i-1}, x) + A_i,\\
        \label{eqn: g spand bound 0}
        \inf_{x\in\bb X} g_0(x) &\leq g_0(x_0) \leq \inf_{x\in\bb X} g_0(x) + A_0.
    \end{align}
\end{lemma}

\begin{proof}
     It suffices to show Equation \ref{eqn: g span bound} holds. The first inequality is obvious. To show the second inequality, we pick arbitrary $x^*\in\bb X$ and we have
    \begin{align*}
        g_i(x_0^i) &= \sum_{l=0}^{i}f_l(x_l) + \sum_{l=i+1}^{n-1} \bb E_{x_i}f_l(x_l)\\
        &\leq \sum_{l=0}^{i-1}f_l(x_l) +f_i(x^*) +\sum_{l=i+1}^{n-1} \bb E_{x_i}f_l(x_l) + (b_i - a_i).
    \end{align*}
    Using Lemma \ref{lem: ipm ineq}, we can see that
    \begin{align*}
        \sum_{l=i+1}^{n-1} \bb E_{x_i}f_l(x_l) - \sum_{l=i+1}^{n-1} \bb E_{x^*}f_l(x_l) \leq \sum_{l=i+1}^{n-1} M_l \ipm_{\mbf F}(\delta_{x_i}P^{l-i}, \delta_{x^*}P^{l-i}),
    \end{align*}
    where the right-hand side is bounded by $2M_{i+1}^{\max}\Gamma_{\mbf F} < \infty$. Thus, we have
    \begin{align*}
        g_i(x_0^i) \leq & \sum_{l=0}^{i-1}f_l(x_l) +f_i(x^*) +\sum_{l=i+1}^{n-1} \bb E_{x^*}f_l(x_l) + \sum_{l=i+1}^{n-1} M_l \ipm_{\mbf F}(\delta_{x_i}P^{l-i}, \delta_{x^*}P^{l-i}) + (b_i - a_i)\\
        \leq & g_i(x_0^{i-1}, x^*) + 2M_{i+1}^{\max}\Gamma_{\mbf F} + (b_i - a_i)\\
        = & g_i(x_0^{i-1}, x^*) + A_i.
    \end{align*}
    Since $x^*$ is arbitrary, we obtain $g_i(x_0^i) \leq \inf_{x^*\in\bb X} g_i(x_0^{i-1}, x^*) + A_i$, which completes the proof of Equation \ref{eqn: g span bound} and Equation \ref{eqn: g spand bound 0} can be obtained in a similar fashion.
\end{proof}

Alternatively, we can bound the span of $g_i$ using IPM Dobrushin coefficient for certain choices of generator listed in Proposition \ref{thm: good choices of gen}. This is useful in the proof of Theorem \ref{thm: main result dobrushin}.

\begin{lemma}
\label{lem: span of g_i alt}
    Suppose the transition kernel $P$ with invariant probability $\pi$ satisfies $\tilde \Gamma_\mbf F < \infty$ (defined in Equation \ref{eqn: dobrushin con}). Let $\mbf F$ be one of the choices in Proposition \ref{thm: good choices of gen}. For each $i\in\{0, \cdots, n-1\}$, if we define 
    \begin{align}
        \tilde A_i \coloneqq 2\sum_{l=i}^{n-1} M_l\Delta_\mbf F(P^{l-i}) + \sp(f_i),
    \end{align}
    for $i\in\{0, \cdots, n-1\}$, then it holds that for $x_0^i \in \bb X^{i+1}$ and $x_0 \in\bb X$,
    \begin{align}
    \label{eqn: g span bound alt}
        \inf_{x\in\bb X} g_i(x_0^{i-1}, x) &\leq g_i(x_0^i) \leq \inf_{x\in\bb X} g_i(x_0^{i-1}, x) + \tilde A_i,\\
    \label{eqn: g span bound 0 alt}
        \inf_{x\in\bb X} g_0(x) &\leq g_0(x_0) \leq \inf_{x\in\bb X} g_0(x) + \tilde A_0.
    \end{align}
\end{lemma}
\begin{proof}
    It suffices to show Equation \ref{eqn: g span bound alt} holds. The first inequality is obvious. To show the second inequality, we pick arbitrary $x^*\in\bb X$ and we have
    \begin{align*}
        g_i(x_0^i) &= \sum_{l=0}^{i}f_l(x_l) + \sum_{l=i+1}^{n-1} \bb E_{x_i}f_l(x_l)\\
        &\leq \sum_{l=0}^{i-1}f_l(x_l) +f_i(x^*) +\sum_{l=i+1}^{n-1} \bb E_{x_i}f_l(x_l) + (b_i - a_i).
    \end{align*}
    Using Lemma \ref{lem: ipm ineq} and the nonexpansiveness of $P$ in $\ipm_\mbf F$, we can see that
    \begin{align}
    \label{eqn: ipm between deltas}
        \sum_{l=i+1}^{n-1} \bb E_{x_i}f_l(x_l) - \sum_{l=i+1}^{n-1} \bb E_{x^*}f_l(x_l) &\leq \sum_{l=i+1}^{n-1} M_l \ipm_{\mbf F}(\delta_{x_i}P^{l-i}, \delta_{x^*}P^{l-i})\nn\\
        &\leq \ipm_\mbf F(\delta_{x_i}, \delta_{x^*}) \sum_{l=i+1}^{n-1} M_l \Delta_{\mbf F}(P^{l-i}).
    \end{align}
    Suppose $\mbf F = \bb B_\lppi$ for $p\in[1,\infty]$, then we have for $f\in\bb B_\lppi$ and $x, y\in\bb X$,
    \begin{align*}
         \bigg|\int f d\delta_x - \int f d\delta_y \bigg| = |\int f (\bbm 1_x - \bbm 1_y) d\pi| \leq \|\bbm 1_x - \bbm 1_y\|_{L^q(\pi)} = 2,
    \end{align*}
    where the last inequality comes from H\"older's inequality and $q\in[1, \infty]$ is the conjugate of $p$. Thus, $\ipm_\mbf F(\delta_{x}, \delta_{y}) \leq 2$ for any $x, y\in\bb X$. As for $\mbf F = \bb B_\unif$, it is easy to deduce that $|\int f d\delta_x - \int f d\delta_y | \leq 2\|f\|_\infty \leq 2$ for any $x, y\in \bb X$. Therefore, the right-hand side of Equation \ref{eqn: ipm between deltas} is bounded by $2\sum_{l=i+1}^{n-1} M_l \Delta_{\mbf F}(P^{l-i})\leq 2 M^{\max}_{i+1} \tilde \Gamma_\mbf F(P) < \infty$. We can now go back to bound $g_i(x_0^i)$ as
    \begin{align*}
        g_i(x_0^i) \leq & \sum_{l=0}^{i-1}f_l(x_l) +f_i(x^*) +\sum_{l=i+1}^{n-1} \bb E_{x^*}f_l(x_l) + 2\sum_{l=i+1}^{n-1} M_l \Delta_{\mbf F}(P^{l-i}) + (b_i - a_i)\\
        \leq & g_i(x_0^{i-1}, x^*) + 2\sum_{l=i+1}^{n-1} M_l \Delta_{\mbf F}(P^{l-i})  + \sp(f_i)\\
        = & g_i(x_0^{i-1}, x^*) + \tilde A_i.
    \end{align*}
    Since $x^*$ is arbitrary, we obtain $g_i(x_0^i) \leq \inf_{x^*\in\bb X} g_i(x_0^{i-1}, x^*) + \tilde A_i$, which completes the proof of Equation \ref{eqn: g span bound alt} and Equation \ref{eqn: g span bound 0 alt} can be obtained in a similar fashion.
\end{proof}

\subsection{Proof of Theorem \ref{thm: main result time dep}}
\label{app: proof of ipm hoeffding time dep}
\begin{proof}
    By Lemma \ref{lem: martingale decomp of mc}, we have the martingale decomposition of $\mbf X$ as in Equation \ref{eqn: martingale decomp of mc},
    \begin{align*}
        \tilde S_n - \bb E_\mu[\tilde S_n] = \sum_{i=0}^{n-1} D_i,
    \end{align*}
    where we abuse the notation and denote $D_0 = f_0(X_0) - \bb E_{\mu}[f_0]$.
    One can use Chernoff's bounding method \citep{chernoff1952measure} to obtain an exponential bound on the desired quantity. Taking the moment generating function on both sides of Equation \ref{eqn: martingale decomp of mc} and applying the chain rule for conditional expectation recursively yield,
    \begin{align}
    \label{eqn: moment gen of tilde s decomp}
        &\bb{E}_\mu\bigg[\exp\bigg(\theta\bigg(\Tilde{S}_n - \sum_{i=0}^{n-1}\bb{E}_\mu[f_i(X_i)]\bigg)\bigg)\bigg] \nn\\
        = &\bb{E}_\mu\bigg[\exp\bigg(\theta\sum_{i=0}^{n-1} D_i\bigg)\bigg]  = \bb{E}_\mu[\exp(\theta D_0)] \prod_{i=1}^{n-1}\bb{E}_\mu[\exp(\theta D_i)| \ALP F_{i-1}],
    \end{align}
    for $\theta\geq 0$. 
   
    From lemma \ref{lem: span of g_i}, we know that $D_i = g_i(X_0^i) - g_{i-1}(X_0^{i-1})$ lies in an interval of length $A_i$ for all $i\in\{0, \cdots, n-1\}$. By Hoeffding's Lemma \citep[Lemma 23.1.4]{douc2018markov} for bounded martingale difference sequences, we have for $\theta \geq 0$,
    \begin{align}
        \bb{E}_\mu[\exp(\theta D_0)] \leq \exp(\theta^2A_0^2/8),&\quad \bb{E}_\mu[\exp(\theta D_i)| \ALP F_{i-1}] \leq \exp(\theta^2A_l^2/8),\nn
    \end{align}
    and plugging above into Equation \ref{eqn: moment gen of tilde s decomp} yields
    \begin{align}
        \bb{E}_\mu\bigg[\exp\bigg(\theta\bigg(\Tilde{S}_n - \sum_{i=0}^{n-1}\bb{E}_\mu[f_i(X_i)]\bigg)\bigg)\bigg] &\leq \exp\bigg(\frac{\theta^2}{8}\sum_{i=0}^{n-1} A_i^2\bigg).\nn
    \end{align}
    Applying Markov's inequality to the left-hand side, we have
    \begin{align}
        \bb{P}_\mu\bigg[\bigg(\Tilde{S}_n - \sum_{i=0}^{n-1}\bb{E}_\mu[f_i(X_i)]\bigg) > n\epsilon\bigg] 
        &\leq \exp(-n\epsilon\theta)\bb{E}_\mu\bigg[\exp\bigg(\theta\bigg(\Tilde{S}_n - \sum_{i=0}^{n-1}\bb{E}_\mu[f_i(X_i)]\bigg)\bigg)\bigg] \nn\\
        &\leq \exp\bigg(-n\epsilon\theta + \frac{\theta^2}{8}\sum_{i=0}^{n-1} A_i^2\bigg).\nn
    \end{align}
    Picking $\theta = 4n\epsilon / \sum_{i=0}^{n-1}A_i^2$ minimizes the right-hand side and yields
    \begin{align}
        \bb{P}_\mu \bigg[\bigg(\Tilde{S}_n - \sum_{i=0}^{n-1}\bb{E}_\mu[f_i(X_i)]\bigg) > n\epsilon\bigg] \leq \exp\bigg(-\frac{2n^2\epsilon^2}{\sum_{i=0}^{n-1} A_i^2}\bigg).\nn
    \end{align}
    The tail probability of the other side can be bounded in an analogous fashion. Therefore, we have
    \begin{align}
        \label{eqn: tail bound of alternative center}
        \bb{P}_\mu \bigg[\bigg|\Tilde{S}_n - \sum_{i=0}^{n-1}\bb{E}_\mu[f_i(X_i)]\bigg| > n\epsilon\bigg] \leq 2\exp\bigg(-\frac{2n^2\epsilon^2}{\sum_{i=0}^{n-1} A_i^2}\bigg).
    \end{align}
    To consider the following tail probability of $\Tilde{S}_n$ centered around $\sum_{i=0}^{n-1}\pi(f_i)$, we apply Lemma \ref{lem: ipm ineq for time dependent} to $(a)$ and Equation \ref{eqn: tail bound of alternative center} to $(b)$ and obtain
    \begin{align}
        \bb{P}_\mu \bigg[\bigg|\tilde{S}_n - \sum_{i=0}^{n-1}\pi(f_i) \bigg| > n\epsilon\bigg] &\stackrel{(a)}{\leq} \bb{P}_\mu \bigg[\bigg|\tilde{S}_n - \sum_{i=0}^{n-1}\bb{E}_\mu[f_i(X_i)]\bigg| + \sum_{i=0}^{n-1} M_i\ipm_{\mbf F}(\mu P^i, \pi) > n\epsilon\bigg]\nn\\
        &\stackrel{(b)}{\leq} 2\exp\bigg(-\frac{2[n\epsilon - A_0]^2}{\sum_{i=0}^{n-1} A_i^2}\bigg),\nn
    \end{align}
    for $\epsilon \geq n^{-1}\ipm_\mbf F(\mu, \pi)A_0$, which yields the desired result in Equation \ref{eqn: Hoeffding ineq IPM time dep} after expanding $A_i$ with Equation \ref{eqn: dummy variable}. 
\end{proof}

\subsection{Proof of Theorem \ref{thm: main result dobrushin}}
\label{app: proof of main result dobrushin}

\begin{proof}
    The overall procedure follows the proof in Appendix \ref{app: proof of main result dobrushin} with the exception where Lemma \ref{lem: span of g_i} is replaced by Lemma \ref{lem: span of g_i alt}. Thus, $D_i$ lies in a interval of length $\tilde A_i$ for all $i\in\{1,..., n-1\}$. The rest is the same after replacing $A_i$ with $\tilde A_i$.
\end{proof}

\newpage
\bibliography{main.bib}

\begin{thebibliography}{86}
\providecommand{\natexlab}[1]{#1}
\providecommand{\url}[1]{\texttt{#1}}
\expandafter\ifx\csname urlstyle\endcsname\relax
  \providecommand{\doi}[1]{doi: #1}\else
  \providecommand{\doi}{doi: \begingroup \urlstyle{rm}\Url}\fi

\bibitem[Adler and Lunz(2018)]{adler2018banach}
Jonas Adler and Sebastian Lunz.
\newblock Banach {W}asserstein {GAN}.
\newblock \emph{Advances in neural information processing systems}, 31, 2018.

\bibitem[Agarwal and Duchi(2012)]{agarwal2012generalization}
Alekh Agarwal and John~C Duchi.
\newblock The generalization ability of online algorithms for dependent data.
\newblock \emph{IEEE Transactions on Information Theory}, 59\penalty0
  (1):\penalty0 573--587, 2012.

\bibitem[Alquier and Wintenberger(2012)]{alquier2012model}
Pierre Alquier and Olivier Wintenberger.
\newblock Model selection for weakly dependent time series forecasting.
\newblock \emph{Bernoulli}, 18\penalty0 (3):\penalty0 883--913, 2012.

\bibitem[Anantharam et~al.(1987)Anantharam, Varaiya, and
  Walrand]{anantharam1987asymptotically}
Venkatachalam Anantharam, Pravin Varaiya, and Jean Walrand.
\newblock Asymptotically efficient allocation rules for the multiarmed bandit
  problem with multiple plays-part ii: Markovian rewards.
\newblock \emph{IEEE Transactions on Automatic Control}, 32\penalty0
  (11):\penalty0 977--982, 1987.

\bibitem[Arbel et~al.(2018)Arbel, Sutherland, Bi{\'n}kowski, and
  Gretton]{arbel2018gradient}
Michael Arbel, Danica~J Sutherland, Miko{\l}aj Bi{\'n}kowski, and Arthur
  Gretton.
\newblock On gradient regularizers for {MMD} {GAN}s.
\newblock \emph{Advances in neural information processing systems}, 31, 2018.

\bibitem[Arjovsky et~al.(2017)Arjovsky, Chintala, and
  Bottou]{arjovsky2017wasserstein}
Martin Arjovsky, Soumith Chintala, and L{\'e}on Bottou.
\newblock {W}asserstein generative adversarial networks.
\newblock In \emph{International conference on machine learning}, pages
  214--223. PMLR, 2017.

\bibitem[Arlot et~al.(2019)Arlot, Celisse, and Harchaoui]{arlot2019kernel}
Sylvain Arlot, Alain Celisse, and Zaid Harchaoui.
\newblock A kernel multiple change-point algorithm via model selection.
\newblock \emph{Journal of machine learning research}, 20\penalty0 (162), 2019.

\bibitem[Aronszajn(1950)]{aronszajn1950theory}
Nachman Aronszajn.
\newblock Theory of reproducing kernels.
\newblock \emph{Transactions of the American mathematical society}, 68\penalty0
  (3):\penalty0 337--404, 1950.

\bibitem[Azuma(1967)]{azuma1967weighted}
Kazuoki Azuma.
\newblock Weighted sums of certain dependent random variables.
\newblock \emph{Tohoku Mathematical Journal, Second Series}, 19\penalty0
  (3):\penalty0 357--367, 1967.

\bibitem[Bach(2014)]{bach2014adaptivity}
Francis Bach.
\newblock Adaptivity of averaged stochastic gradient descent to local strong
  convexity for logistic regression.
\newblock \emph{The Journal of Machine Learning Research}, 15\penalty0
  (1):\penalty0 595--627, 2014.

\bibitem[Bach and Moulines(2013)]{bach2013non}
Francis Bach and Eric Moulines.
\newblock Non-strongly-convex smooth stochastic approximation with convergence
  rate o (1/n).
\newblock \emph{Advances in neural information processing systems}, 26, 2013.

\bibitem[Bellemare et~al.(2017)Bellemare, Danihelka, Dabney, Mohamed,
  Lakshminarayanan, Hoyer, and Munos]{bellemare2017cramer}
Marc~G Bellemare, Ivo Danihelka, Will Dabney, Shakir Mohamed, Balaji
  Lakshminarayanan, Stephan Hoyer, and R{\'e}mi Munos.
\newblock The {C}ramer distance as a solution to biased {W}asserstein
  gradients.
\newblock \emph{arXiv preprint arXiv:1705.10743}, 2017.

\bibitem[Billingsley(2013)]{billingsley2013convergence}
Patrick Billingsley.
\newblock \emph{Convergence of probability measures}.
\newblock John Wiley \& Sons, 2013.

\bibitem[Bi{\'n}kowski et~al.(2018)Bi{\'n}kowski, Sutherland, Arbel, and
  Gretton]{binkowski2018demystifying}
Miko{\l}aj Bi{\'n}kowski, Danica~J Sutherland, Michael Arbel, and Arthur
  Gretton.
\newblock Demystifying {MMD} {GAN}s.
\newblock In \emph{International Conference on Learning Representations}, 2018.

\bibitem[Blanchet and Murthy(2019)]{blanchet2019quantifying}
Jose Blanchet and Karthyek Murthy.
\newblock Quantifying distributional model risk via optimal transport.
\newblock \emph{Mathematics of Operations Research}, 44\penalty0 (2):\penalty0
  565--600, 2019.

\bibitem[Boucheron et~al.(2013)Boucheron, Lugosi, and
  Massart]{boucheron2013concentration}
St{\'e}phane Boucheron, G{\'a}bor Lugosi, and Pascal Massart.
\newblock \emph{Concentration inequalities: A nonasymptotic theory of
  independence}.
\newblock Oxford university press, 2013.

\bibitem[Breiman(1960)]{breiman1960strong}
Leo Breiman.
\newblock The strong law of large numbers for a class of {M}arkov chains.
\newblock \emph{The Annals of Mathematical Statistics}, 31\penalty0
  (3):\penalty0 801--803, 1960.

\bibitem[Brockman et~al.(2016)Brockman, Cheung, Pettersson, Schneider,
  Schulman, Tang, and Zaremba]{brockman2016openai}
Greg Brockman, Vicki Cheung, Ludwig Pettersson, Jonas Schneider, John Schulman,
  Jie Tang, and Wojciech Zaremba.
\newblock Openai gym.
\newblock \emph{arXiv preprint arXiv:1606.01540}, 2016.

\bibitem[Butkovsky(2014)]{butkovsky2014subgeometric}
Oleg Butkovsky.
\newblock Subgeometric rates of convergence of {M}arkov processes in the
  {W}asserstein metric.
\newblock \emph{The Annals of Applied Probability}, 24\penalty0 (2):\penalty0
  526--552, 2014.

\bibitem[Chen et~al.(2022)Chen, Tang, and Gupta]{chen2022change}
Hao Chen, Jiacheng Tang, and Abhishek Gupta.
\newblock Change detection of {M}arkov kernels with unknown pre and post change
  kernel.
\newblock In \emph{2022 IEEE 61st Conference on Decision and Control (CDC)},
  pages 4814--4820, 2022.
\newblock \doi{10.1109/CDC51059.2022.9992982}.

\bibitem[Chernoff(1952)]{chernoff1952measure}
Herman Chernoff.
\newblock A measure of asymptotic efficiency for tests of a hypothesis based on
  the sum of observations.
\newblock \emph{The Annals of Mathematical Statistics}, pages 493--507, 1952.

\bibitem[Czapla et~al.(2024)Czapla, Horbacz, and
  Wojew{\'o}dka-{\'S}ci{\k{a}}{\.z}ko]{czapla2024central}
Dawid Czapla, Katarzyna Horbacz, and Hanna Wojew{\'o}dka-{\'S}ci{\k{a}}{\.z}ko.
\newblock The central limit theorem for {M}arkov processes that are
  exponentially ergodic in the bounded-lipschitz norm.
\newblock \emph{Qualitative Theory of Dynamical Systems}, 23\penalty0
  (1):\penalty0 1--46, 2024.

\bibitem[Davisson et~al.(1981)Davisson, Longo, and Sgarro]{davisson1981error}
L~Davisson, Giuseppe Longo, and Andrea Sgarro.
\newblock The error exponent for the noiseless encoding of finite ergodic
  {M}arkov sources.
\newblock \emph{IEEE Transactions on Information Theory}, 27\penalty0
  (4):\penalty0 431--438, 1981.

\bibitem[Del~Moral et~al.(2003)Del~Moral, Ledoux, and
  Miclo]{del2003contraction}
Pierre Del~Moral, Michel Ledoux, and Laurent Miclo.
\newblock On contraction properties of {M}arkov kernels.
\newblock \emph{Probability theory and related fields}, 126\penalty0
  (3):\penalty0 395--420, 2003.

\bibitem[Devroye et~al.(2013)Devroye, Gy{\"o}rfi, and
  Lugosi]{devroye2013probabilistic}
Luc Devroye, L{\'a}szl{\'o} Gy{\"o}rfi, and G{\'a}bor Lugosi.
\newblock \emph{A probabilistic theory of pattern recognition}, volume~31.
\newblock Springer Science \& Business Media, 2013.

\bibitem[Diaconis and Freedman(1999)]{diaconis1999iterated}
Persi Diaconis and David Freedman.
\newblock Iterated random functions.
\newblock \emph{SIAM review}, 41\penalty0 (1):\penalty0 45--76, 1999.

\bibitem[Dieuleveut et~al.(2017)Dieuleveut, Flammarion, and
  Bach]{dieuleveut2017harder}
Aymeric Dieuleveut, Nicolas Flammarion, and Francis Bach.
\newblock Harder, better, faster, stronger convergence rates for least-squares
  regression.
\newblock \emph{The Journal of Machine Learning Research}, 18\penalty0
  (1):\penalty0 3520--3570, 2017.

\bibitem[Dobrushin(1956)]{dobrushin1956central}
Roland~L Dobrushin.
\newblock Central limit theorem for nonstationary {M}arkov chains. ii.
\newblock \emph{Theory of Probability \& Its Applications}, 1\penalty0
  (4):\penalty0 329--383, 1956.

\bibitem[Douc et~al.(2011)Douc, Moulines, Olsson, and
  Van~Handel]{douc2011consistency}
Randal Douc, Eric Moulines, Jimmy Olsson, and Ramon Van~Handel.
\newblock Consistency of the maximum likelihood estimator for general hidden
  {M}arkov models.
\newblock \emph{the Annals of Statistics}, 39\penalty0 (1):\penalty0 474--513,
  2011.

\bibitem[Douc et~al.(2018)Douc, Moulines, Priouret, and
  Soulier]{douc2018markov}
Randal Douc, Eric Moulines, Pierre Priouret, and Philippe Soulier.
\newblock \emph{{M}arkov chains}.
\newblock Springer, 2018.

\bibitem[Dudley(2018)]{dudley2018real}
Richard~M Dudley.
\newblock \emph{Real analysis and probability}.
\newblock CRC Press, 2018.

\bibitem[Dziugaite et~al.(2015)Dziugaite, Roy, and
  Ghahramani]{dziugaite2015training}
Gintare~Karolina Dziugaite, Daniel~M Roy, and Zoubin Ghahramani.
\newblock Training generative neural networks via maximum mean discrepancy
  optimization.
\newblock In \emph{Proceedings of the Thirty-First Conference on Uncertainty in
  Artificial Intelligence}, pages 258--267, 2015.

\bibitem[Edwards(2011)]{edwards2011kantorovich}
David~A Edwards.
\newblock On the {K}antorovich--{R}ubinstein theorem.
\newblock \emph{Expositiones Mathematicae}, 29\penalty0 (4):\penalty0 387--398,
  2011.

\bibitem[Fan et~al.(2021)Fan, Jiang, and Sun]{fan2021hoeffding}
Jianqing Fan, Bai Jiang, and Qiang Sun.
\newblock {H}oeffding's inequality for general {M}arkov chains and its
  applications to statistical learning.
\newblock \emph{J. Mach. Learn. Res.}, 22:\penalty0 139--1, 2021.

\bibitem[Flynn and Yoo(2019)]{flynn2019change}
Thomas Flynn and Shinjae Yoo.
\newblock Change detection with the kernel cumulative sum algorithm.
\newblock In \emph{2019 IEEE 58th Conference on Decision and Control (CDC)},
  pages 6092--6099. IEEE, 2019.

\bibitem[Gao and Kleywegt(2023)]{gao2023distributionally}
Rui Gao and Anton Kleywegt.
\newblock Distributionally robust stochastic optimization with {W}asserstein
  distance.
\newblock \emph{Mathematics of Operations Research}, 48\penalty0 (2):\penalty0
  603--655, 2023.

\bibitem[Gao et~al.(2022)Gao, Chen, and Kleywegt]{gao2022wasserstein}
Rui Gao, Xi~Chen, and Anton~J Kleywegt.
\newblock {W}asserstein distributionally robust optimization and variation
  regularization.
\newblock \emph{Operations Research}, 2022.

\bibitem[Gibbs and Su(2002)]{gibbs2002choosing}
Alison~L Gibbs and Francis~Edward Su.
\newblock On choosing and bounding probability metrics.
\newblock \emph{International statistical review}, 70\penalty0 (3):\penalty0
  419--435, 2002.

\bibitem[Glynn and Ormoneit(2002)]{glynn2002hoeffding}
Peter~W Glynn and Dirk Ormoneit.
\newblock {H}oeffding's inequality for uniformly ergodic {M}arkov chains.
\newblock \emph{Statistics \& probability letters}, 56\penalty0 (2):\penalty0
  143--146, 2002.

\bibitem[Goodfellow et~al.(2020)Goodfellow, Pouget-Abadie, Mirza, Xu,
  Warde-Farley, Ozair, Courville, and Bengio]{goodfellow2020generative}
Ian Goodfellow, Jean Pouget-Abadie, Mehdi Mirza, Bing Xu, David Warde-Farley,
  Sherjil Ozair, Aaron Courville, and Yoshua Bengio.
\newblock Generative adversarial networks.
\newblock \emph{Communications of the ACM}, 63\penalty0 (11):\penalty0
  139--144, 2020.

\bibitem[Granas and Dugundji(2003)]{granas2003fixed}
Andrzej Granas and James Dugundji.
\newblock \emph{Fixed point theory}, volume~14.
\newblock Springer, 2003.

\bibitem[Gretton et~al.(2006)Gretton, Borgwardt, Rasch, Sch{\"o}lkopf, and
  Smola]{gretton2006kernel}
Arthur Gretton, Karsten Borgwardt, Malte Rasch, Bernhard Sch{\"o}lkopf, and
  Alex Smola.
\newblock A kernel method for the two-sample-problem.
\newblock \emph{Advances in neural information processing systems}, 19, 2006.

\bibitem[Gretton et~al.(2012)Gretton, Borgwardt, Rasch, Sch{\"o}lkopf, and
  Smola]{gretton2012kernel}
Arthur Gretton, Karsten~M Borgwardt, Malte~J Rasch, Bernhard Sch{\"o}lkopf, and
  Alexander Smola.
\newblock A kernel two-sample test.
\newblock \emph{The Journal of Machine Learning Research}, 13\penalty0
  (1):\penalty0 723--773, 2012.

\bibitem[Gulrajani et~al.(2017)Gulrajani, Ahmed, Arjovsky, Dumoulin, and
  Courville]{gulrajani2017improved}
Ishaan Gulrajani, Faruk Ahmed, Martin Arjovsky, Vincent Dumoulin, and Aaron~C
  Courville.
\newblock Improved training of {W}asserstein {GAN}s.
\newblock \emph{Advances in neural information processing systems}, 30, 2017.

\bibitem[Gupta et~al.(2018)Gupta, Jain, and Glynn]{gupta2018probabilistic}
Abhishek Gupta, Rahul Jain, and Peter Glynn.
\newblock Probabilistic contraction analysis of iterated random operators.
\newblock \emph{arXiv preprint arXiv:1804.01195}, 2018.

\bibitem[Gupta et~al.(2020)Gupta, Chen, Pi, and Tendolkar]{gupta2020some}
Abhishek Gupta, Hao Chen, Jianzong Pi, and Gaurav Tendolkar.
\newblock Some limit properties of {M}arkov chains induced by recursive
  stochastic algorithms.
\newblock \emph{SIAM Journal on Mathematics of Data Science}, 2\penalty0
  (4):\penalty0 967--1003, 2020.

\bibitem[Hairer et~al.(2011)Hairer, Mattingly, and
  Scheutzow]{hairer2011asymptotic}
Martin Hairer, Jonathan~C Mattingly, and Michael Scheutzow.
\newblock Asymptotic coupling and a general form of {H}arris’ theorem with
  applications to stochastic delay equations.
\newblock \emph{Probability theory and related fields}, 149:\penalty0 223--259,
  2011.

\bibitem[Han et~al.(2023)Han, Hu, and Long]{han2021class}
Jiequn Han, Ruimeng Hu, and Jihao Long.
\newblock A class of dimension-free metrics for the convergence of empirical
  measures.
\newblock \emph{Stochastic Processes and their Applications}, 164:\penalty0
  242--287, 2023.

\bibitem[Hille and Theewis(2022)]{hille2022explicit}
Sander~C Hille and Esmee~S Theewis.
\newblock Explicit expressions and computational methods for the
  {F}ortet-{M}ourier distance to finite weighted sums of {D}irac measures.
\newblock \emph{arXiv preprint arXiv:2206.12234}, 2022.

\bibitem[{H}oeffding(1994)]{hoeffding1994probability}
Wassily {H}oeffding.
\newblock Probability inequalities for sums of bounded random variables.
\newblock In \emph{The collected works of Wassily {H}oeffding}, pages 409--426.
  Springer, 1994.

\bibitem[Hu and Hong(2013)]{hu2013kullback}
Zhaolin Hu and L~Jeff Hong.
\newblock Kullback-{L}eibler divergence constrained distributionally robust
  optimization.
\newblock \emph{Available at Optimization Online}, 1\penalty0 (2):\penalty0 9,
  2013.

\bibitem[Husain(2020)]{husain2020distributional}
Hisham Husain.
\newblock Distributional robustness with {IPM}s and links to regularization and
  {GAN}s.
\newblock \emph{Advances in Neural Information Processing Systems},
  33:\penalty0 11816--11827, 2020.

\bibitem[Kim et~al.(2022)Kim, Kim, Kong, Ohn, and Kim]{kim2022learning}
Dongha Kim, Kunwoong Kim, Insung Kong, Ilsang Ohn, and Yongdai Kim.
\newblock Learning fair representation with a parametric integral probability
  metric.
\newblock In \emph{International Conference on Machine Learning}, pages
  11074--11101. PMLR, 2022.

\bibitem[Li et~al.(2017)Li, Chang, Cheng, Yang, and P{\'o}czos]{li2017mmd}
Chun-Liang Li, Wei-Cheng Chang, Yu~Cheng, Yiming Yang, and Barnab{\'a}s
  P{\'o}czos.
\newblock {MMD} {GAN}: Towards deeper understanding of moment matching network.
\newblock \emph{Advances in neural information processing systems}, 30, 2017.

\bibitem[Li et~al.(2015)Li, Xie, Dai, and Song]{li2015m}
Shuang Li, Yao Xie, Hanjun Dai, and Le~Song.
\newblock M-statistic for kernel change-point detection.
\newblock \emph{Advances in Neural Information Processing Systems}, 28, 2015.

\bibitem[Lin et~al.(2022)Lin, Fang, and Gao]{lin2022distributionally}
Fengming Lin, Xiaolei Fang, and Zheming Gao.
\newblock Distributionally robust optimization: A review on theory and
  applications.
\newblock \emph{Numerical Algebra, Control and Optimization}, 12\penalty0
  (1):\penalty0 159--212, 2022.

\bibitem[Liu and Wang(2018)]{PhysRevA.98.062324}
Jin-Guo Liu and Lei Wang.
\newblock Differentiable learning of quantum circuit {B}orn machines.
\newblock \emph{Phys. Rev. A}, 98:\penalty0 062324, Dec 2018.
\newblock \doi{10.1103/PhysRevA.98.062324}.
\newblock URL \url{https://link.aps.org/doi/10.1103/PhysRevA.98.062324}.

\bibitem[Liu and Liu(2021)]{liu2021hoeffding}
Yuanyuan Liu and Jinpeng Liu.
\newblock {H}oeffding’s inequality for {M}arkov processes via solution of
  {P}oisson’s equation.
\newblock \emph{Frontiers of Mathematics in China}, 16\penalty0 (2):\penalty0
  543--558, 2021.

\bibitem[Meyn and Tweedie(2012)]{meyn2012markov}
Sean~P Meyn and Richard~L Tweedie.
\newblock \emph{{M}arkov chains and stochastic stability}.
\newblock Springer Science \& Business Media, 2012.

\bibitem[Miasojedow(2014)]{miasojedow2014hoeffding}
B{\l}a{\.z}ej Miasojedow.
\newblock {H}oeffding’s inequalities for geometrically ergodic {M}arkov
  chains on general state space.
\newblock \emph{Statistics \& Probability Letters}, 87:\penalty0 115--120,
  2014.

\bibitem[Mohajerin~Esfahani and Kuhn(2018)]{mohajerin2018data}
Peyman Mohajerin~Esfahani and Daniel Kuhn.
\newblock Data-driven distributionally robust optimization using the
  {W}asserstein metric: Performance guarantees and tractable reformulations.
\newblock \emph{Mathematical Programming}, 171\penalty0 (1-2):\penalty0
  115--166, 2018.

\bibitem[Moulos(2020)]{moulos2020hoeffding}
Vrettos Moulos.
\newblock A {H}oeffding inequality for finite state {M}arkov chains and its
  applications to markovian bandits.
\newblock In \emph{2020 IEEE International Symposium on Information Theory
  (ISIT)}, pages 2777--2782. IEEE, 2020.

\bibitem[Moulos and Anantharam(2019)]{moulos2019optimal}
Vrettos Moulos and Venkat Anantharam.
\newblock Optimal {C}hernoff and {H}oeffding bounds for finite state {M}arkov
  chains.
\newblock \emph{arXiv preprint arXiv:1907.04467}, 2019.

\bibitem[Moustakides(1999)]{moustakides1999extension}
George~V Moustakides.
\newblock Extension of {W}ald's first lemma to {M}arkov processes.
\newblock \emph{Journal of applied probability}, 36\penalty0 (1):\penalty0
  48--59, 1999.

\bibitem[Mroueh and Sercu(2017)]{mroueh2017fisher}
Youssef Mroueh and Tom Sercu.
\newblock Fisher {GAN}.
\newblock \emph{Advances in neural information processing systems}, 30, 2017.

\bibitem[Mroueh et~al.(2018)Mroueh, Li, Sercu, Raj, and
  Cheng]{mroueh2017sobolev}
Youssef Mroueh, Chun~Liang Li, Tom Sercu, Anant Raj, and Yu~Cheng.
\newblock Sobolev {GAN}.
\newblock In \emph{International Conference on Learning Representations}.
  International Conference on Learning Representations, ICLR, 2018.

\bibitem[M{\"u}ller(1997)]{muller1997integral}
Alfred M{\"u}ller.
\newblock Integral probability metrics and their generating classes of
  functions.
\newblock \emph{Advances in Applied Probability}, 29\penalty0 (2):\penalty0
  429--443, 1997.

\bibitem[Nemirovski et~al.(2009)Nemirovski, Juditsky, Lan, and
  Shapiro]{nemirovski2009robust}
Arkadi Nemirovski, Anatoli Juditsky, Guanghui Lan, and Alexander Shapiro.
\newblock Robust stochastic approximation approach to stochastic programming.
\newblock \emph{SIAM Journal on optimization}, 19\penalty0 (4):\penalty0
  1574--1609, 2009.

\bibitem[Nowozin et~al.(2016)Nowozin, Cseke, and Tomioka]{nowozin2016f}
Sebastian Nowozin, Botond Cseke, and Ryota Tomioka.
\newblock f-{GAN}: Training generative neural samplers using variational
  divergence minimization.
\newblock \emph{Advances in neural information processing systems}, 29, 2016.

\bibitem[Polyak and Juditsky(1992)]{polyak1992acceleration}
Boris~T Polyak and Anatoli~B Juditsky.
\newblock Acceleration of stochastic approximation by averaging.
\newblock \emph{SIAM journal on control and optimization}, 30\penalty0
  (4):\penalty0 838--855, 1992.

\bibitem[Rahimian and Mehrotra(2019)]{rahimian2019distributionally}
Hamed Rahimian and Sanjay Mehrotra.
\newblock Distributionally robust optimization: A review.
\newblock \emph{arXiv preprint arXiv:1908.05659}, 2019.

\bibitem[Rudolf and Schweizer(2018)]{rudolf2018perturbation}
Daniel Rudolf and Nikolaus Schweizer.
\newblock Perturbation theory for {M}arkov chains via {W}asserstein distance.
\newblock \emph{Bernoulli}, 24\penalty0 (4A):\penalty0 2610--2639, 2018.

\bibitem[Ruppert(1988)]{ruppert1988efficient}
David Ruppert.
\newblock Efficient estimations from a slowly convergent {R}obbins-{M}onro
  process.
\newblock Technical report, Cornell University Operations Research and
  Industrial Engineering, 1988.

\bibitem[Sandri{\'c} and {\v{S}}ebek(2023)]{sandric2021hoeffding}
Nikola Sandri{\'c} and Stjepan {\v{S}}ebek.
\newblock {H}oeffding’s inequality for non-irreducible {M}arkov models.
\newblock \emph{Statistics \& Probability Letters}, 200:\penalty0 109870, 2023.
\newblock ISSN 0167-7152.
\newblock \doi{https://doi.org/10.1016/j.spl.2023.109870}.
\newblock URL
  \url{https://www.sciencedirect.com/science/article/pii/S0167715223000949}.

\bibitem[Shorack and Shorack(2000)]{shorack2000probability}
Galen~R Shorack and GR~Shorack.
\newblock \emph{Probability for statisticians}, volume 951.
\newblock Springer, 2000.

\bibitem[Sriperumbudur et~al.(2010{\natexlab{a}})Sriperumbudur, Fukumizu,
  Gretton, Sch{\"o}lkopf, and Lanckriet]{sriperumbudur2010non}
Bharath~K Sriperumbudur, Kenji Fukumizu, Arthur Gretton, Bernhard
  Sch{\"o}lkopf, and Gert~RG Lanckriet.
\newblock Non-parametric estimation of integral probability metrics.
\newblock In \emph{2010 IEEE International Symposium on Information Theory},
  pages 1428--1432. IEEE, 2010{\natexlab{a}}.

\bibitem[Sriperumbudur et~al.(2010{\natexlab{b}})Sriperumbudur, Gretton,
  Fukumizu, Sch{\"o}lkopf, and Lanckriet]{sriperumbudur2010hilbert}
Bharath~K Sriperumbudur, Arthur Gretton, Kenji Fukumizu, Bernhard
  Sch{\"o}lkopf, and Gert~RG Lanckriet.
\newblock Hilbert space embeddings and metrics on probability measures.
\newblock \emph{The Journal of Machine Learning Research}, 11:\penalty0
  1517--1561, 2010{\natexlab{b}}.

\bibitem[Sriperumbudur et~al.(2012)Sriperumbudur, Fukumizu, Gretton,
  Sch{\"o}lkopf, and Lanckriet]{sriperumbudur2012empirical}
Bharath~K Sriperumbudur, Kenji Fukumizu, Arthur Gretton, Bernhard
  Sch{\"o}lkopf, and Gert~RG Lanckriet.
\newblock On the empirical estimation of integral probability metrics.
\newblock \emph{Electronic Journal of Statistics}, 6:\penalty0 1550--1599,
  2012.

\bibitem[Staib and Jegelka(2019)]{staib2019distributionally}
Matthew Staib and Stefanie Jegelka.
\newblock Distributionally robust optimization and generalization in kernel
  methods.
\newblock \emph{Advances in Neural Information Processing Systems}, 32, 2019.

\bibitem[Tekin and Liu(2012)]{tekin2012online}
Cem Tekin and Mingyan Liu.
\newblock Online learning of rested and restless bandits.
\newblock \emph{IEEE Transactions on Information Theory}, 58\penalty0
  (8):\penalty0 5588--5611, 2012.

\bibitem[Watanabe and Hayashi(2017)]{watanabe2017finite}
Shun Watanabe and Masahito Hayashi.
\newblock {Finite-length analysis on tail probability for {M}arkov chain and
  application to simple hypothesis testing}.
\newblock \emph{The Annals of Applied Probability}, 27\penalty0 (2):\penalty0
  811 -- 845, 2017.
\newblock \doi{10.1214/16-AAP1216}.
\newblock URL \url{https://doi.org/10.1214/16-AAP1216}.

\bibitem[Xu(2021)]{xu2021towards}
Minkai Xu.
\newblock Towards generalized implementation of {W}asserstein distance in
  {GAN}s.
\newblock \emph{Proceedings of the AAAI Conference on Artificial Intelligence},
  35\penalty0 (12):\penalty0 10514--10522, 2021.

\bibitem[Yu(1994)]{yu1994rates}
Bin Yu.
\newblock Rates of convergence for empirical processes of stationary mixing
  sequences.
\newblock \emph{The Annals of Probability}, pages 94--116, 1994.

\bibitem[Zhu et~al.(2021)Zhu, Jitkrittum, Diehl, and
  Sch{\"o}lkopf]{zhu2021kernel}
Jia-Jie Zhu, Wittawat Jitkrittum, Moritz Diehl, and Bernhard Sch{\"o}lkopf.
\newblock Kernel distributionally robust optimization: Generalized duality
  theorem and stochastic approximation.
\newblock In \emph{International Conference on Artificial Intelligence and
  Statistics}, pages 280--288. PMLR, 2021.

\bibitem[Zolotarev(1984)]{zolotarev1984probability}
Vladimir~Mikhailovich Zolotarev.
\newblock Probability metrics.
\newblock \emph{Theory of Probability \& Its Applications}, 28\penalty0
  (2):\penalty0 278--302, 1984.

\bibitem[Zou et~al.(2009)Zou, Li, and Xu]{zou2009generalization}
Bin Zou, Luoqing Li, and Zongben Xu.
\newblock The generalization performance of {ERM} algorithm with strongly
  mixing observations.
\newblock \emph{Machine learning}, 75\penalty0 (3):\penalty0 275--295, 2009.

\end{thebibliography}

\end{document}